\documentclass[11pt]{article} 

\pdfoutput=1

\usepackage{etoolbox}
\newtoggle{arxiv}
\toggletrue{arxiv}

\usepackage[utf8]{inputenc}
\usepackage{mathrsfs}
\usepackage{microtype}
\usepackage{subfigure}
\usepackage{booktabs} 
\usepackage{longtable,makecell}
\usepackage{amsmath, amsfonts, amsthm, amssymb, graphicx}
\usepackage{mathtools,verbatim}
\usepackage{enumitem}
\usepackage{bm}
\usepackage{thm-restate}
\usepackage{xspace}
\usepackage[dvipsnames]{xcolor}

\usepackage{color-edits}

\addauthor{df}{ForestGreen}
\addauthor{ms}{orange}

\usepackage[noend]{algpseudocode}
\usepackage{algorithm}

\iftoggle{arxiv}
{
	\usepackage[scaled=.90]{helvet}
	
	\usepackage[vmargin=1.00in,hmargin=1.00in,centering,letterpaper]{geometry}
	\usepackage{fancyhdr, lastpage}

	\setlength{\headsep}{.10in}
	\setlength{\headheight}{15pt}
	\cfoot{}
	\lfoot{}
	\rfoot{\sc Page\ \thepage\ of\ \protect\pageref{LastPage}}

}

\usepackage{hyperref}
\PassOptionsToPackage{hypertexnames=false}{hyperref}
\usepackage{cleveref}
\hypersetup{colorlinks,citecolor=blue}
\newtoggle{upperbound}
\togglefalse{upperbound}

\usepackage{upgreek}
\usepackage{natbib}
\usepackage{wrapfig}
\newcommand{\poly}{\mathrm{poly}}

\numberwithin{equation}{section}

\newcommand{\fro}{\mathrm{F}}
\newcommand{\op}{\mathrm{op}}

{
 \theoremstyle{plain}
      \newtheorem{asm}{Assumption}
}
\Crefname{asm}{Assumption}{Assumptions}

\theoremstyle{plain}
\newtheorem{thm}{Theorem}

\newtheorem{prop}[thm]{Proposition}
\theoremstyle{definition}

\theoremstyle{plain}
\newtheorem{theorem}{Theorem}
\newtheorem{lemma}{Lemma}[section]
\newtheorem{corollary}{Corollary}
\newtheorem{proposition}[thm]{Proposition}
\theoremstyle{definition}

\newtheorem{remark}{Remark}[section]

\renewcommand{\Pr}{\mathbb{P}}
\newcommand{\Exp}{\mathbb{E}}

\newcommand{\rmd}{\mathrm{d}}

\newcommand{\tr}{\mathrm{tr}}

\newcommand{\R}{\mathbb{R}}

\DeclareMathOperator*{\argmin}{arg\,min}
\DeclareMathOperator*{\logdet}{log\,det}

\newcommand{\simplex}{\triangle}
\newcommand{\cOtil}{\widetilde{\cO}}

\def\ddefloop#1{\ifx\ddefloop#1\else\ddef{#1}\expandafter\ddefloop\fi}
\def\ddef#1{\expandafter\def\csname bb#1\endcsname{\ensuremath{\mathbb{#1}}}}
\ddefloop ABCDEFGHIJKLMNOPQRSTUVWXYZ\ddefloop

\def\ddefloop#1{\ifx\ddefloop#1\else\ddef{#1}\expandafter\ddefloop\fi}
\def\ddef#1{\expandafter\def\csname fr#1\endcsname{\ensuremath{\mathfrak{#1}}}}
\ddefloop ABCDEFGHIJKLMNOPQRSTUVWXYZ\ddefloop
\def\ddefloop#1{\ifx\ddefloop#1\else\ddef{#1}\expandafter\ddefloop\fi}
\def\ddef#1{\expandafter\def\csname scr#1\endcsname{\ensuremath{\mathscr{#1}}}}
\ddefloop ABCDEFGHIJKLMNOPQRSTUVWXYZ\ddefloop
\def\ddefloop#1{\ifx\ddefloop#1\else\ddef{#1}\expandafter\ddefloop\fi}
\def\ddef#1{\expandafter\def\csname b#1\endcsname{\ensuremath{\mathbf{#1}}}}
\ddefloop ABCDEFGHIJKLMNOPQRSTUVWXYZ\ddefloop
\def\ddef#1{\expandafter\def\csname c#1\endcsname{\ensuremath{\mathcal{#1}}}}
\ddefloop ABCDEFGHIJKLMNOPQRSTUVWXYZ\ddefloop
\def\ddef#1{\expandafter\def\csname h#1\endcsname{\ensuremath{\widehat{#1}}}}
\ddefloop ABCDEFGHIJKLMNOPQRSTUVWXYZ\ddefloop
\def\ddef#1{\expandafter\def\csname t#1\endcsname{\ensuremath{\widetilde{#1}}}}
\ddefloop ABCDEFGHIJKLMNOPQRSTUVWXYZ\ddefloop

\def\ddefloop#1{\ifx\ddefloop#1\else\ddef{#1}\expandafter\ddefloop\fi}
\def\ddef#1{\expandafter\def\csname mat#1\endcsname{\ensuremath{\mathbf{#1}}}}
\ddefloop abcdefgijklmnopqrstuvwxyzABCDEFGHIJKLMNOPQRSTUVWXYZ\ddefloop

\newcommand{\LC}{\textsc{LC}$^3$\xspace}
\newcommand{\mineigalg}{\textsc{MinEig}\xspace}
\newcommand{\tsum}{{\textstyle \sum}}

\newcommand{\fwregret}{\textsc{FWRegret}\xspace}

\newcommand{\tople}{\textsc{Tople}\xspace}
\newcommand{\expdesign}{\textsc{DynamicOED}\xspace}
\newcommand{\learnpolicies}{\textsc{LearnExp}\text{$\Pi$}\xspace}

\newcommand{\traj}{\bm{\uptau}}

\newcommand{\pihat}{\widehat{\pi}}

\newcommand{\bu}{\bm{u}}
\newcommand{\bphi}{\bm{\phi}}

\newcommand{\inner}[2]{\langle #1, #2 \rangle}
\newcommand{\piexp}{\pi_{\mathrm{exp}}}
\newcommand{\bLambda}{\bm{\Lambda}}

\newcommand{\btheta}{\bm{\theta}}
\newcommand{\pist}{\pi_\star}

\newcommand{\Ast}{A_\star}
\newcommand{\cost}{\mathrm{cost}}
\newcommand{\Ahat}{\widehat{A}}
\newcommand{\bGamma}{\bm{\Gamma}}
\newcommand{\bOmega}{\bm{\Omega}}
\newcommand{\bthetast}{\btheta_\star}
\newcommand{\bthetahat}{\widehat{\btheta}}

\newcommand{\Kst}{\pi_\star}
\newcommand{\CR}{C_{\cR}}
\newcommand{\frakD}{\mathfrak{D}}
\newcommand{\regalg}{\bbA_{\cR}}

\newcommand{\bx}{\bm{x}}
\newcommand{\bw}{\bm{w}}
\newcommand{\dimx}{d_{\bx}}
\newcommand{\dimu}{d_{\bu}}
\newcommand{\dimw}{d_{\bw}}
\newcommand{\by}{\bm{y}}
\newcommand{\dimphi}{d_{\bphi}}

\newcommand{\Piexp}{\Pi_{\mathrm{exp}}}

\newcommand{\bGamtil}{\widetilde{\bGamma}}
\newcommand{\bGambar}{\bar{\bGamma}}

\newcommand{\dimzeta}{d_{\btheta}}

\newcommand{\K}{\pi}
\newcommand{\Kzeta}{\pi^{\btheta}}
\newcommand{\fw}{f_{\bw}}
\newcommand{\bSigma}{\bm{\Sigma}}
\newcommand{\lammin}{\lambda_{\min}}
\newcommand{\lamminst}{\lammin^\star}
\newcommand{\bpsi}{\bm{\psi}}

\newcommand{\f}[2]{f_{#1,#2}}

\newcommand{\rest}{r_{\mathrm{est}}}
\newcommand{\Lcost}{L_{\cost}}
\newcommand{\Lphi}{L_{\bphi}}
\newcommand{\Lzeta}{L_{\btheta}}
\newcommand{\LKst}{L_{\Kst}}

\newcommand{\Abar}{\bar{A}}

\newcommand{\cmax}{c_{\max}}
\newcommand{\ball}{\textsc{Ball}}
\newcommand{\dimpsi}{d_{\bpsi}}
\newcommand{\bOmegahat}{\widehat{\bOmega}}
\newcommand{\bSigtil}{\widetilde{\bSigma}}
\newcommand{\Bphi}{B_{\bphi}}
\newcommand{\pR}{p_{\cR}}
\newcommand{\Khat}{\widehat{\pi}}
\newcommand{\Pimineig}{\Pi_{\mineigalg}}

\newcommand{\bLamtil}{\widetilde{\bLambda}}
\newcommand{\dimtheta}{d_{\bm{\theta}}}
\newcommand{\lammax}{\lambda_{\max}}

\newcommand{\thetast}{\bm{\theta}^\star}
\newcommand{\Kfw}{T_{\mathrm{fw}}}
\newcommand{\bLamfw}{\bGamma_{\mathrm{fw}}}
\newcommand{\Pifw}{\Pi_{\mathrm{fw}}}

\newcommand{\Piout}{\Pi_{\mathrm{out}}}
\newcommand{\Kout}{T_{\mathrm{out}}}
\newcommand{\cEexp}{\cE_{\mathrm{exp}}}
\newcommand{\ist}{i^\star}
\newcommand{\cEop}{\cE_{\op}}
\newcommand{\lamun}{\underline{\lambda}}

\newcommand{\ellf}{\ell_{T}}
\newcommand{\Ahatm}{\Ahat^{-}}
\newcommand{\Ntil}{\widetilde{N}}
\newcommand{\Tbar}{\bar{T}}
\newcommand{\Clot}{C_{\mathrm{lot}}}

\newcommand{\Clb}{C_{\mathrm{lb}}}
\newcommand{\Picon}{\Pi^\star}
\renewcommand{\vec}{\mathrm{vec}}
\newcommand{\bLamchk}{\check{\bLambda}}
\newcommand{\Cpoly}{C_{\poly}}

\newcommand{\Ktheta}{\K^{\btheta}}
\newcommand{\bthetatilst}{\widetilde{\btheta}_\star}
\newcommand{\omegaexp}{\omega_{\mathrm{exp}}}
\newcommand{\bv}{\bm{v}}
\newcommand{\bOmpsi}{\bOmega_{\bpsi}}
\newcommand{\bOmhatpsi}{\bOmegahat_{\bpsi}}
\newcommand{\rcost}{r_{\mathrm{cost}}}
\newcommand{\rtheta}{r_{\btheta}}
\newcommand{\rmu}{r_{\mu}}
\newcommand{\cOst}{\cO^{\star}}
\newcommand{\dimpsib}{d}
\newcommand{\bz}{\bm{z}}
\newcommand{\Tout}{T_{\mathrm{out}}}
\newcommand{\BA}{B_{A}}
\newcommand{\Toed}{T^{\mathrm{oed}}}
\newcommand{\sft}{\mathtt{t}}
\newcommand{\sfT}{\mathtt{T}}
\newcommand{\bugoal}{\bu_{\mathrm{goal}}}
\newcommand{\Atil}{\widetilde{A}}
\newcommand{\Nfw}{N_{\mathrm{fw}}}
\newcommand{\Kfwb}{K_{\mathrm{fw}}}
\newcommand{\Tfw}{T_{\mathrm{fw}}}
\newcommand{\bmu}{\bm{\mu}}
\newcommand{\bmust}{\bmu^\star}
\newcommand{\bmuhat}{\widehat{\bmu}}
\newcommand{\dimmu}{d_{\bmu}}
\newcommand{\bGamout}{\bGamma_{\mathrm{out}}}
\newcommand{\sigw}{\sigma_{\bw}}
\newcommand{\bthetatil}{\widetilde{\btheta}}
\newcommand{\Delst}{\Delta^\star}

\newlength\tindent
\allowdisplaybreaks
\newcommand{\loose}{\looseness=-1}

\newcommand{\awcomment}[1]{}
\newcommand{\gscomment}[1]{}
\newcommand{\kevin}[1]{}
\newcommand{\arxivedit}[1]{}

\title{Optimal Exploration for Model-Based RL in Nonlinear Systems}
\author{Andrew Wagenmaker\footnote{University of Washington, Seattle. Email: \texttt{ajwagen@cs.washington.edu}} \and Guanya Shi\footnote{University of Washington, Seattle. Email: \texttt{guanyas@cs.washington.edu}} \and Kevin Jamieson\footnote{University of Washington, Seattle. Email: \texttt{jamieson@cs.washington.edu}} }
\date{}

\begin{document}

\maketitle

\begin{abstract}

Learning to control unknown nonlinear dynamical systems is a fundamental problem in reinforcement learning and control theory. A commonly applied approach is to first explore the environment (exploration), learn an accurate model of it (system identification), and then compute an optimal controller with the minimum cost on this estimated system (policy optimization). While existing work has shown that it is possible to learn a uniformly good model of the system~\citep{mania2020active}, in practice, if we aim to learn a good controller with a low cost on the actual system, certain system parameters may be significantly more critical than others, and we therefore ought to focus our exploration on learning such parameters.

In this work, we consider the setting of nonlinear dynamical systems and seek to formally quantify, in such settings, (a) which parameters are most relevant to learning a good controller, and (b) how we can best explore so as to minimize uncertainty in such parameters. Inspired by recent work in linear systems~\citep{wagenmaker2021task}, we show that minimizing the controller loss in nonlinear systems translates to estimating the system parameters in a particular, task-dependent metric. Motivated by this, we develop an algorithm able to efficiently explore the system to reduce uncertainty in this metric, and prove a lower bound showing that our approach learns a controller at a near-instance-optimal rate. Our algorithm relies on a general reduction from policy optimization to optimal experiment design in arbitrary systems, and may be of independent interest. We conclude with experiments demonstrating the effectiveness of our method in realistic nonlinear robotic systems\footnote{Code: \url{https://github.com/ajwagen/nonlinear_sysid_for_control}}.

\end{abstract}


\section{Introduction}

Controlling nonlinear dynamical systems is a core problem in robotics, cyber-physical systems, and beyond, and a significant body of work in both the control theory and reinforcement learning communities has sought to address this challenge~\citep{slotine1991applied,aastrom2013adaptive,sutton2018reinforcement}. In many real-world scenarios~\citep{shi2019neural,ljung1998system,nguyen2011model,brunke2022safe}, the dynamics of the system of interest is unknown, or only a coarse model of them is available, which significantly increases the challenge of control---not only must we control such systems, we must \emph{learn} to control them. While a variety of methods exist to address this challenge, a commonly applied approach is to first perform \emph{system identification}, learning an accurate model of the system's dynamics, and then use this model to obtain a controller. Despite its promising potential, there are still several fundamental questions that must be answered to make this approach practically effective.

\iftoggle{arxiv}{\paragraph{\textit{Which parameters are most relevant to learning a good controller?}} }{\textbf{\textit{Which parameters are most relevant to learning a good controller?}}}
Beyond some special cases, little work has been done characterizing how the estimation error from system identification translates to end-to-end suboptimality in the resulting controller of our nonlinear systems. 
In particular, certain parameters of the system or regions of the state space may be irrelevant to learning a good controller, and coarse estimates of these parameters would suffice, while other parameters may be critical to learning a good controller, and we must therefore estimate these parameters very accurately in order to effectively control the system. 
In the context of this work, where nonlinearities are considered, the heterogeneity of the parameters is further accentuated. For instance, around a point of equilibrium, some system parameters might be completely inactive, having no impact on the dynamics
(see the example in \Cref{sec:motivating_example} for an illustration of this).

\iftoggle{arxiv}{\paragraph{\textit{How can we best explore so as to minimize uncertainty in relevant parameters?}}}{\textbf{\textit{How can we best explore so as to minimize uncertainty in relevant parameters?}}}
Even if we are able to determine which parameters are most important for obtaining a good controller on the true system, it is not obvious how to use this information. How can we direct our system identification phase in order to focus on learning these parameters as quickly as possible, without spending time estimating the parameters of the system less critical for control? This is fundamentally a question of \emph{exploration}. While it is known in linear systems that random excitation will efficiently explore \citep{simchowitz2018learning}, exploration in nonlinear systems is significantly more challenging since, in order to excite all parameters of interest, non-trivial planning may be required to ensure all relevant states are reached (as is the case in the example considered in \Cref{sec:motivating_example}). 

\iftoggle{arxiv}{\vspace{1em} \noindent We address both these questions in a particular class of nonlinear systems parameterized as:}
{We address both these questions in a particular class of nonlinear systems parameterized as:}
\begin{align}\label{eq:system}
\bx_{h+1} = \Ast \bphi(\bx_h,\bu_h) + \bw_h.
\end{align}
Here $\bx_h \in \R^{\dimx}$ denotes the state of the system, $\bu_h \in \cU \subseteq \R^{\dimu}$ the input, $\bw_h \sim \cN(0, \sigw^2 \cdot I)$ random noise, $\bphi(\cdot,\cdot) \in \R^{\dimphi}$ a (possibly nonlinear, known) feature map, and $\Ast \in \R^{\dimx \times \dimphi}$ the (unknown) system parameter.
Systems of this form are able to model a variety of real-world settings~\citep{shi2021meta,o2022neural,boffi2021regret,song2021pc,richards2021adaptive}\footnote{
In real-world settings, $\bphi$ is typically (1) from physics (i.e., the system structure is known but some parameters such as drag coefficient are unknown~\citep{slotine1991applied}), (2) learned using representation learning or meta-learning~\citep{o2022neural,richards2021adaptive}, and/or (3) from random features (e.g., any sufficiently regular, smooth nonlinear system $\bm{f}(\bx,\bu)$ can be modeled by~\eqref{eq:system} using $N$ random features up to a $1/\sqrt{N}$ error~\citep{rahimi2008uniform}).
}, 
and have been the subject of recent attention in the reinforcement learning community \citep{mania2020active,kakade2020information,song2021pc}, yet the aforementioned questions have remained unanswered.
Towards addressing this, in this work we make the following contributions:

\begin{enumerate}[leftmargin=*]
\item For systems of the form \eqref{eq:system}, given some cost of interest which we wish to find a controller to minimize, we (a) formally characterize how estimation error translates into suboptimality in the learned controller, under the \emph{certainty equivalent} control rule 
and (b) provide a lower bound on the loss of \emph{any} (sufficiently regular) control rule learned from $T$ rounds of interaction with \eqref{eq:system}.

\item Motivated by this characterization, we present an algorithm which achieves the \emph{instance-optimal} rate, with controller loss matching our lower bound. To the best of our knowledge, this is the first statistically optimal algorithm in the setting of nonlinear dynamical systems. Our algorithm relies on a generic reduction from policy optimization to optimal exploration in \emph{arbitrary} dynamical systems (not necessarily of the form \eqref{eq:system}), which may be of independent interest. 

\item \iftoggle{arxiv}{We present numerical experiments on several realistic nonlinear systems which illustrate that our approach---efficiently exploring to reduce uncertainty in parameters most relevant to learning a controller---yields significant gains in practice.}{We present numerical experiments on several realistic nonlinear systems which illustrate that our approach---efficiently exploring to reduce uncertainty in parameters most relevant to learning a controller---yields significant gains in practice.}
\end{enumerate}
\iftoggle{arxiv}{
Our work builds on the recent work of \cite{wagenmaker2021task}, which addresses a similar set of challenges in the linear dynamical systems setting---we extend this work to the nonlinear  setting. To further motivate our approach, we consider the following example.
}{
To further motivate our approach, we consider the following example.
}

\iftoggle{arxiv}{
\begin{wrapfigure}{r}{0.5\textwidth}
\vspace{-3em}
  \begin{center}
    \includegraphics[width=1.0\linewidth]{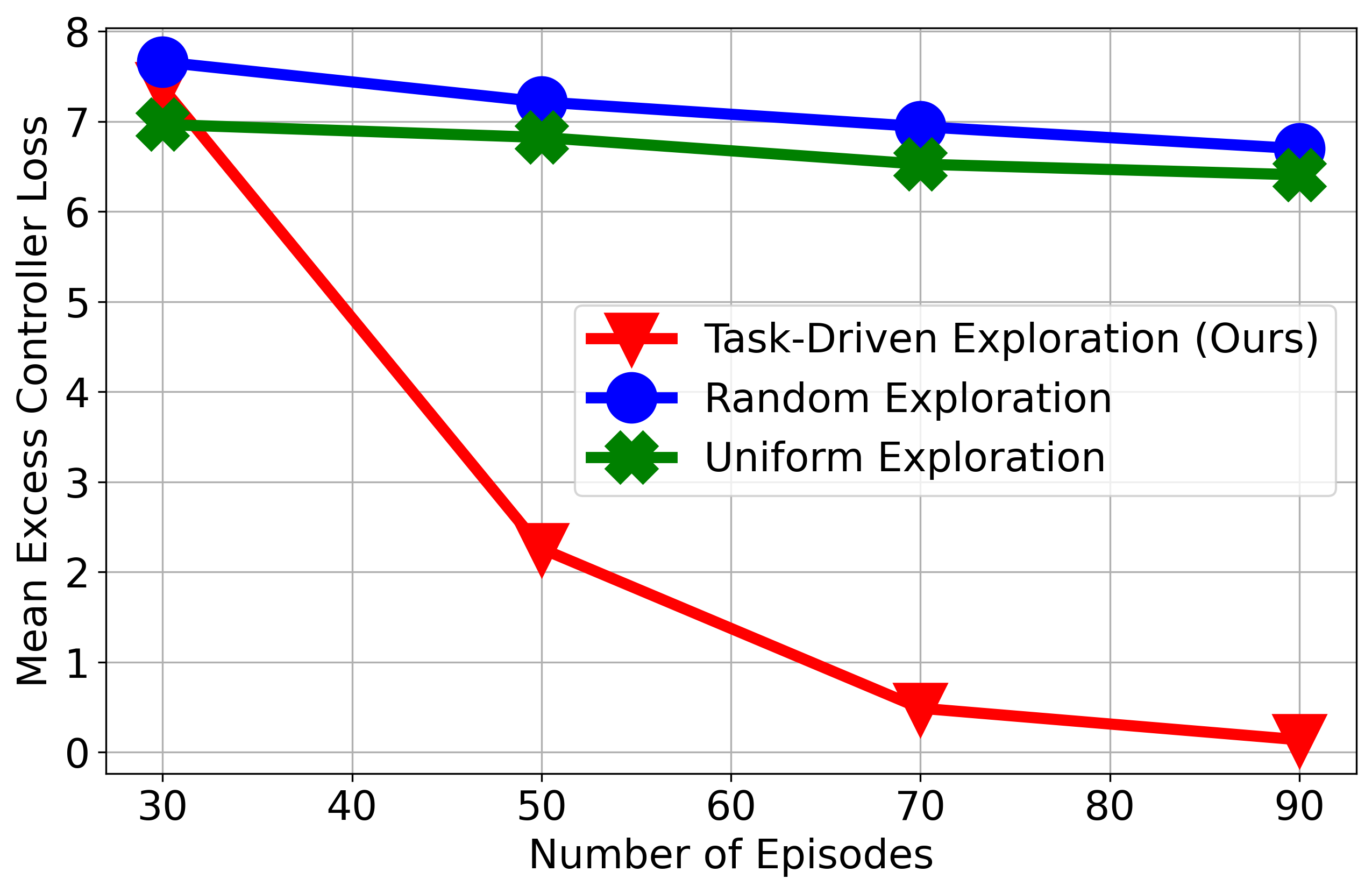}
    \vspace{-2em}
    \caption{Performance on Motivating Example}
\label{fig:motivating}
  \end{center}
\end{wrapfigure}
}{
\begin{wrapfigure}{r}{0.4\textwidth}
\vspace{-4em}
  \begin{center}
    \includegraphics[width=1.0\linewidth]{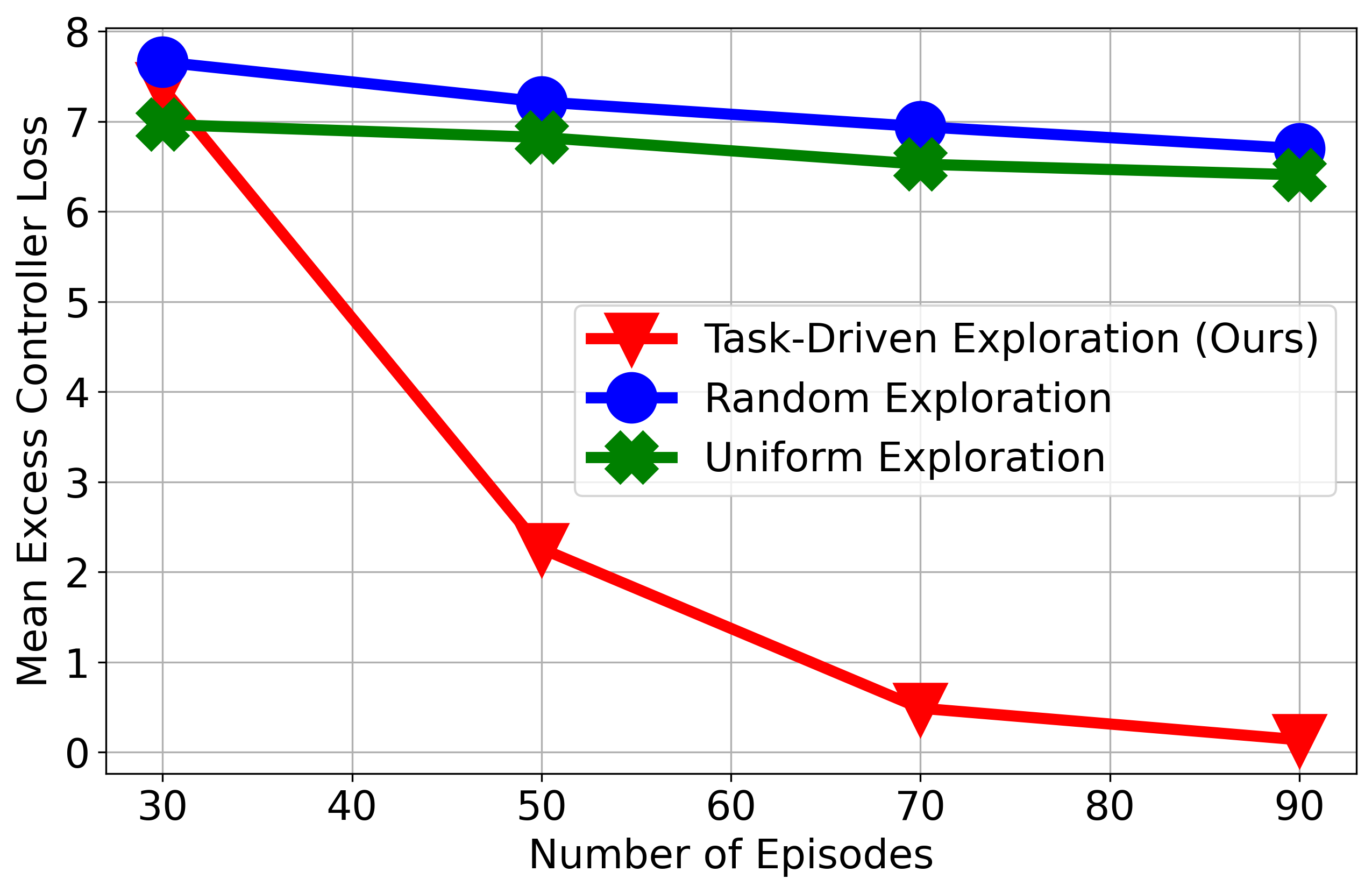}
    \vspace{-2em}
    \caption{Performance on Motivating Example}
\label{fig:motivating}
  \end{center}
\end{wrapfigure}
}

\subsection{Motivating Example}\label{sec:motivating_example}

To motivate the need for effective exploration, we consider a simple 1-D system with nonlinear dynamics given by:
\iftoggle{arxiv}{\begin{align*}
\bx_{h+1} = a_1 \bx_h + a_2 \bu_h + \sum_{i=1}^{10} a_{i+2} \bphi_i(\bx_h) + \bw_h
\end{align*}}
{\begin{align*}
\bx_{h+1} = a_1 \bx_h + a_2 \bu_h + \tsum_{i=1}^{10} a_{i+2} \bphi_i(\bx_h) + \bw_h
\end{align*}}
where $\bphi_i(\bx) = \max \{ 1 - 100 ( \bx - c_i)^2, 0 \}$
for some $c_i$. We choose $a_1 = 0.8, a_2 = 1$, and  $a_3 = \ldots = a_{12} = -3$.
We assume $a_{1:12}$ are unknown, $(\bphi_i)_{i=1}^{10}$ is known, and set
\begin{align*}
\cost(\bx,\bu) = (\bx - c_1)^2 + 100^{-1} \cdot \bu^2.
\end{align*}
With this choice of cost, the optimal controller will attempt to direct the state $\bx$ to the equilibrium point $c_1$ and maintain this position. Note that, with our choice of $\bphi_i$, $\bphi_i(\bx) \neq 0$ only when $\bx$ is very close to $c_i$. This renders the parameters $a_{4:12}$ irrelevant to learning the optimal controller, since $\bphi_2,\ldots,\bphi_{10}$ will be inactive if we are playing optimally, but learning $a_3$ is critical to performing optimally, as its value significantly changes the dynamics at the goal state.

We illustrate the result of running on this system in \Cref{fig:motivating}, comparing our proposed approach (\texttt{Task-Driven Exploration}, \Cref{alg:main}) to the approach which chooses $\bu_h \sim \cN(0,\sigma_{\bu}^2)$ (\texttt{Random Exploration}), and the approach proposed in \cite{mania2020active} (\texttt{Uniform Exploration}) which seeks to explore so as to estimate $a_{1:12}$ uniformly well. As can be seen, neither of these latter two approaches are able to learn a good controller, while our approach easily finds a near-optimal controller.
The failure modes of each of these approaches is somewhat different. Here \texttt{Random Exploration} fails since the chance of reaching the point $\bx_h \sim c_1$ is extremely small if the input is random noise---reaching $c_1$ requires playing a particular sequence of actions which are very unlikely to be played if $\bu_h$ is chosen randomly. 
The \texttt{Uniform Exploration} approach does, in contrast, plan and, given enough time, is guaranteed to estimate all parameters accurately. However, as it aims to estimate all parameters uniformly well, it will attempt to estimate $a_{4:12}$ accurately despite their irrelevance to control, which will slow down the rate at which it is able to estimate $a_3$.
Only our approach, which both plans and takes into account the cost while exploring, is able to reach $a_3$ enough times to efficiently estimate it, and learn a good controller.

This example illustrates that it is critical both to explore efficiently, and also to let the objective---learning a good controller---guide this exploration.
We emphasize that the behavior in this example is only exhibited in nonlinear systems---though taking into account the task while exploring in linear systems is known to yield provable improvements \citep{wagenmaker2021task}, even playing random noise allows every direction to be learned in such systems. In nonlinear systems, however, this is not the case---one may fail to learn completely unless careful planning is performed.

\section{Related Work}

\iftoggle{arxiv}{
\paragraph{Online learning and control.}
Recently, there has been increased interest in studying online learning and control from a learning-theoretic perspective, largely for settings with linear systems such as online LQR or LQG with unknown dynamics~\citep{abbasi2011improved,simchowitz2018learning,simchowitz2019learning,mania2019certainty,cohen2019learning,dean2020sample,yu2020power,wagenmaker2020active,simchowitz2020naive,simchowitz2020improper}. 
In the nonlinear setting, \citep{foster2020learning,oymak2019stochastic,sattar2022non} provide formal guarantees on system identification in several different classes of nonlinear systems, yet they only consider noiseless systems, or systems that are significantly easier to excite than \eqref{eq:system} (rendering the problem of exploration significantly easier).
\cite{kakade2020information} study systems of the form \eqref{eq:system}, but consider only the regret minimization problem. While their bounds would yield a polynomial complexity via an online-to-batch conversion, our characterization is significantly tighter.
The most relevant work, \cite{mania2020active}, proposes an active learning approach to identify unknown parameters in \eqref{eq:system}, with the goal of minimizing the Euclidean distance in the parameter space. However, as we show, learning a uniformly good model could be significantly worse than learning a model with the goal task in mind.
Also very related to our work is \cite{wagenmaker2021task}, which seeks to answer a similar set of questions as what we consider: performing system identification in order to learn a good controller. This work is restricted to the setting of linear dynamics, however, and does not address the additional complexity of exploration in nonlinear systems.

\paragraph{System identification, dual control, and iterative learning control.}
There is a large body of classical work in system identification~\citep{ljung1998system}, and our work can be seen as an instance of \emph{active} system identification. While a variety of approaches have been proposed which study similar problems \citep{mehra1974optimal,gerencser2005adaptive,katselis2012application,manchester2010input,rojas2007robust,goodwin1977dynamic,lindqvist2001identification,gerencser2007adaptive}, then tend to only consider linear systems, or lack rigorous theoretical guarantees.
Recently deep learning approaches have also been applied in system identification~\citep{shi2019neural,nguyen2011model,brunke2022safe,williams2017information,shi2021neural}. In these works, the system identification phase is separate from the downstream controller design. Instead, in the control community, estimating parameters while simultaneously or iteratively optimizing for performance has been formulated as a dual control problem~\citep{feldbaum1960dual,mesbah2018stochastic} or an iterative learning control problem~\citep{bristow2006survey}. However, both settings focus on stability, robustness, or asymptotic convergence whereas our work quantifies the end-to-end suboptimality gap with a statistically optimal algorithm.

\paragraph{Model-based reinforcement learning.}
This paper falls into the broad category of model-based reinforcement learning (MBRL), where an agent explores the environment to learn a model and then computes an optimal policy using the learned model. 
On the empirical side, deep MBRL has made exciting progress in many domains~\citep{kaiser2019model,yu2020mopo,chua2018deep}. Several task-aware methods have been designed to improve MBRL's performance, such as uncertainty-aware policy optimization~\citep{yu2020mopo,chua2018deep} and active exploration to reduce model uncertainty~\citep{nakka2020chance}, yet these works lack formal guarantees.
On the theoretical side, a variety of different model-based approaches exist \citep{osband2014model,sun2019model,agarwal2020model,zhou2021nearly,zanette2019tighter,azar2017minimax,song2021pc}; however, the majority of these consider restricted settings such as tabular or linear MDPs. Of particular interest is the work of \cite{song2021pc} which presents a result in systems of the form \eqref{eq:system}. While they show that polynomial sample complexity is possible, our results yield a significantly tighter characterization.

\paragraph{Adaptive nonlinear control.} Adaptive nonlinear control also seeks to control an unknown nonlinear system with parametric uncertainties~\citep{slotine1991applied,aastrom2013adaptive}. In particular, the key idea of model-reference adaptive control (MRAC) bears affinity to this paper, in that the adaptation law in MRAC adapts unknown parameters in a task-aware manner, by relating the tracking error with the estimated parameter in a closed loop. In fact, the parameter estimation error in MRAC converges \emph{only when necessary}, i.e., when the task is ``rich'' enough (the formal condition is called persistent excitation~\citep{aastrom2013adaptive,slotine1991applied}). There are two main differences between MRAC and our work. First, adaptive control does not explicitly optimize a cost function. The objective of adaptive control is often tracking error convergence and Lyapunov stability, whereas our framework allows general cost functions. Moreover, adaptive control theory typically focuses on asymptotic convergence, but we give non-asymptotic optimality guarantees. Second, adaptive control has by and large been limited to specific system classes (e.g., fully-actuated systems~\citep{aastrom2013adaptive,richards2021adaptive}) and policy classes (e.g., policy to directly cancel out the matched uncertainty~\citep{o2022neural,boffi2021regret}), whereas our framework allows more general systems and policy classes.

}{

\paragraph{Learning for control.}
Recently, there has been increased interest in studying control problems from a learning-theoretic perspective, largely for linear system settings such as online LQR or LQG with unknown dynamics~\cite{abbasi2011improved,simchowitz2018learning,simchowitz2019learning,mania2019certainty,cohen2019learning,dean2020sample,yu2020power,wagenmaker2020active,simchowitz2020naive,simchowitz2020improper}. 
In the nonlinear setting, \cite{foster2020learning,oymak2019stochastic,sattar2022non} provide formal guarantees on system identification in several different classes of nonlinear systems, yet they only consider noiseless systems, or systems that are significantly easier to excite than \eqref{eq:system}.
\cite{kakade2020information} study systems of the form \eqref{eq:system}, but consider only the regret minimization problem. While their bounds would yield a polynomial complexity via an online-to-batch conversion, our characterization is significantly tighter.
The most relevant work~\cite{mania2020active} proposes an active learning approach to identify unknown parameters in \eqref{eq:system}, with the goal of minimizing the Euclidean distance in the parameter space. However, as shown in \Cref{sec:motivating_example}, this approach could be significantly worse than learning a model with the goal task in mind.
Also very related to our work is \cite{wagenmaker2021task}, which seeks to answer a similar set of questions: performing system identification in order to learn a good controller. This work is restricted to the setting of linear dynamics, however, and does not address the additional complexities of nonlinear systems.\loose

\vspace{-1em}
\paragraph{System identification, dual control, and iterative learning control.}
There is a large body of classical work in system identification~\cite{ljung1998system}, and our work can be seen as an instance of \emph{active} system identification. While a variety of approaches have been proposed for active system identification \cite{mehra1974optimal,gerencser2005adaptive,katselis2012application,manchester2010input,rojas2007robust,goodwin1977dynamic,lindqvist2001identification,gerencser2007adaptive}, these tend to only consider linear systems, or lack rigorous theoretical guarantees.
Recently deep learning approaches have also been applied in system identification~\cite{shi2019neural,nguyen2011model,brunke2022safe,williams2017information,shi2021neural}. In these works, the system identification phase is separate from the downstream controller design. Instead, in the control community, estimating parameters while simultaneously optimizing for performance has been formulated as a dual or iterative learning control problem~\cite{feldbaum1960dual,mesbah2018stochastic,bristow2006survey}, yet these settings focus on stability or asymptotic convergence whereas our work quantifies the end-to-end suboptimality gap. 

\vspace{-1em}
\paragraph{Model-based reinforcement learning.}
This paper falls into the broad category of model-based reinforcement learning (MBRL), where an agent explores the environment to learn a model and then computes an optimal policy using the learned model. 
On the empirical side, deep MBRL has made exciting progress in many domains~\cite{kaiser2019model,yu2020mopo,chua2018deep}, and several task-aware methods have been designed to improve MBRL's performance~\cite{yu2020mopo,chua2018deep,nakka2020chance}, yet these works lack formal guarantees.
On the theoretical side, a variety of different model-based approaches exist \cite{osband2014model,sun2019model,agarwal2020model,zhou2021nearly,zanette2019tighter,azar2017minimax,song2021pc}; however, the majority of these consider restricted settings such as tabular or linear MDPs. Of particular interest is the work of \cite{song2021pc} which presents a result in systems of the form \eqref{eq:system}. While they show that polynomial sample complexity is possible, our results yield a significantly tighter characterization.

\vspace{-1em}
\paragraph{Adaptive nonlinear control.} Adaptive nonlinear control also seeks to control an unknown nonlinear system with parametric uncertainties~\cite{slotine1991applied,aastrom2013adaptive}. In particular, the key idea of model-reference adaptive control (MRAC) bears affinity to this paper, in that the adaptation law in MRAC adapts unknown parameters in a task-aware manner. There are two main differences between MRAC and our work. First, adaptive control does not explicitly optimize a cost function---the objective is typically tracking error convergence and Lyapunov stability, whereas our framework allows general cost functions---and the focus is typically on asymptotic convergence, while we give non-asymptotic optimality guarantees. Second, adaptive control has by and large been limited to specific system classes (e.g., fully-actuated systems~\cite{aastrom2013adaptive,richards2021adaptive}) and policy classes (e.g., policy to directly cancel out the matched uncertainty~\cite{o2022neural,boffi2021regret}), whereas our framework allows more general systems and policy classes.
}


\section{Preliminaries}\label{sec:prelim}

\paragraph{Notation.}
$\| \cdot \|_\op$ denotes the operator norm (matrix 2-norm), $\| \cdot \|_\fro$ the Frobenius norm, and $\| \cdot \|_{M}$ the Mahalanobis norm, defined as $\| \bx \|_M := \sqrt{\bx^\top M \bx}$ for $M \succeq 0$. $\vec(A)$ denotes the vectorization of matrix $A$.
$\cB_{p}(A;r) := \{ A' : \| A - A' \|_{p} \le r \}$. $[H] = \{1,2,\ldots,H\}$. \iftoggle{arxiv}{$\simplex_\cX$ denotes the set of distributions over set $\cX$. We let $\cS^{d-1}$ refer to the unit ball in $d$ dimensions and $\bbS_+^d$ (resp. $\bbS_{++}^d$) the set of positive semi-definite matrices (resp. positive definite matrices) in $\R^{d \times d}$.}{} We let $\Exp_{A}[\cdot]$ denote the expectation over trajectories induced on system with parameter $A$, and $\Exp_{A,\K}[\cdot]$ the expectation induced when policy $\K$ is played. 
\iftoggle{arxiv}{Throughout, $\cO(\cdot)$ denotes standard big-O notation, $\cOtil(\cdot)$ hides additional logarithmic factors, and we use $\lesssim$ informally to highlight key parameters in an inequality.}{$\poly(\cdot)$ denotes some term that is polynomial in its arguments, with exponents absolute constants. We use $\lesssim$ informally to highlight key parameters in an inequality.}

\paragraph{Setting.}
In this work, we are interested in systems of the form \eqref{eq:system}.
We consider the episodic setting, where episodes are of length $H$, and assume that each episodes starts from a given state $\bx_1$. 
We also assume $\| \Ast \|_\op \le \BA$ for some known $\BA > 0$. 
We note that the setting considered here encompasses many real-world systems of interest in robotics and control (e.g., \citep{o2022neural,song2021pc,shi2021meta,richards2021adaptive} and \Cref{sec:experiments}).\loose

The goal of the learner is to find a policy (controller) $\K = (\K_h)_{h=1}^H$ which achieves minimal cost on \eqref{eq:system}, for the cost defined by some (known) function $(\cost_h(\cdot,\cdot))_{h=1}^H$, with $\cost_h : \R^{\dimx} \times \cU \rightarrow \R_+$. For a given policy $\K$, we define the expected cost on system $A$ as
\iftoggle{arxiv}{\begin{align*}
\cJ(\K;A) := \Exp_{A,\K} \left [ \sum_{h=1}^H \cost_h(\bx_h,\bu_h) \right ].
\end{align*}}
{\begin{align*}
\cJ(\K;A) := \Exp_{A,\K} \left [ \tsum_{h=1}^H \cost_h(\bx_h,\bu_h) \right ].
\end{align*}}
We consider the following interaction protocol:
\iftoggle{arxiv}{\begin{enumerate}
\item Learner interacts with system \eqref{eq:system} for $T$ episodes, at every episode playing an exploration policy $\piexp \in \Piexp$.
\item After $T$ episodes, the learner proposes a policy $\Khat_T \in \Picon$.
\item The learner suffers cost $\cJ(\Khat_T;\Ast)$. 
\end{enumerate}}{
\begin{enumerate}
\item Learner interacts with \eqref{eq:system} for $T$ episodes, at each episode playing a policy $\piexp \in \Piexp$.\loose
\item After $T$ episodes, the learner proposes a policy $\Khat_T \in \Picon$.
\item The learner suffers cost $\cJ(\Khat_T;\Ast)$. 
\end{enumerate}}
The goal of the learner is therefore first to explore and, after $T$ episodes of exploration, to propose its best guess at the optimal controller for \eqref{eq:system}, $\Khat_T$.
Here we take $\Piexp$ to be a (known) set of admissible exploration policies (for example, policies with bounded input power), and $\Picon$ a (known) set of admissible control policies. We assume that policies in $\Picon$ are deterministic, but allow for randomized policies in $\Piexp$. Policies may be either open- or closed-loop. \iftoggle{arxiv}{Note that we do not assume $\Picon = \Piexp$---in general $\Piexp$ need not be equal to $\Picon$.}

\iftoggle{arxiv}{
\paragraph{System Notation.}
\iftoggle{arxiv}{Before proceeding, we introduce several additional pieces of notation. First, we}{We} let $\cT$ denote the space of all possible state-input trajectories, $\cT \subseteq (\R^{\dimx} \times \cU)^H \times \R^{\dimx}$, and, for any $\traj \in \cT$, let $\traj_{1:h}$ denote the first $h$ states and inputs in $\traj$. \iftoggle{arxiv}{Second, for}{For} any policy $\pi$, we denote
\begin{align*}
\bLambda_{A,\pi} := \Exp_{A,\pi} \left [ \sum_{h=1}^H \bphi(\bx_h,\bu_h) \bphi(\bx_h,\bu_h)^\top \right ]
\end{align*}
the expected covariance induced by playing $\pi$ on system $A$. In particular, we set $\bLambda_\pi := \bLambda_{\Ast,\pi}$. 
We also denote $\bLamchk := I_{\dimx} \otimes \bLambda$ the Kronecker product of $I_{\dimx}$ and $\bLambda$.
\iftoggle{arxiv}{Finally, we let $\bOmega$ denote the set of all possible covariance matrices induced by playing mixtures of policies in $\Piexp$:
\begin{align*}
\bOmega := \left \{ \Exp_{\pi \sim \omega}[\bLambda_\pi] \ : \ \omega \in \simplex_{\Piexp} \right \},
\end{align*}
where $\simplex_{\Piexp}$ denotes the set of distributions over $\Piexp$.}{
Finally, we let $\bOmega$ denote the set of all possible covariance matrices induced by playing mixtures of policies in $\Piexp$, $\bOmega := \{ \Exp_{\pi \sim \omega}[\bLambda_\pi] \ : \ \omega \in \simplex_{\Piexp} \}$, for $\simplex_{\Piexp}$ the set of distributions over $\Piexp$. }
}{
\paragraph{System Notation.}
\iftoggle{arxiv}{Before proceeding, we introduce several additional pieces of notation. First, we}{We} let $\cT$ denote the space of all possible state trajectories, $\cT \subseteq \R^{\dimx \times (H+1)}$, and, for any $\traj \in \cT$, let $\traj_{1:h}$ denote the first $h$ states and inputs in $\traj$. \iftoggle{arxiv}{Second, for}{For} any policy $\pi$, we denote
\begin{align*}
\bLambda_{A,\pi} := \Exp_{A,\pi} \left [ \tsum_{h=1}^H \bphi(\bx_h,\bu_h) \bphi(\bx_h,\bu_h)^\top \right ]
\end{align*}
the expected covariance induced by playing $\pi$ on system $A$, $\bLambda_\pi := \bLambda_{\Ast,\pi}$, and $\bLamchk := I_{\dimx} \otimes \bLambda$ the Kronecker product of $I_{\dimx}$ and $\bLambda$.
\iftoggle{arxiv}{Finally, we let $\bOmega$ denote the set of all possible covariance matrices induced by playing mixtures of policies in $\Piexp$:
\begin{align*}
\bOmega := \{ \Exp_{\pi \sim \omega}[\bLambda_\pi] \ : \ \omega \in \simplex_{\Piexp} \},
\end{align*}
where $\simplex_{\Piexp}$ denotes the set of distributions over $\Piexp$.}{
Finally, we let $\bOmega$ denote the convex hull of covariance matrices induced by $\Piexp$, $\bOmega := \{ \Exp_{\pi \sim \omega}[\bLambda_\pi] \ : \ \omega \in \simplex_{\Piexp} \}$, for $\simplex_{\Piexp}$ the set of distributions over $\Piexp$. \loose }}

\subsection{Regularity Assumptions}

In order to make learning in \eqref{eq:system} tractable, we need several regularity assumptions. 
\iftoggle{arxiv}{We first introduce assumptions on the boundedness of the feature map $\bphi$, the boundedness of the cost, and the achievable minimum eigenvalue.}

\begin{asm}[Bounded Features]\label{asm:bounded_features}
For all $\bx \in \R^{\dimx}$ and $\bu \in \cU$, we have $\| \bphi(\bx,\bu) \|_2 \le \Bphi$.
\end{asm}

\begin{asm}[Bounded Cost]\label{asm:bounded_cost}
There exists some $\rcost(\Ast) > 0$ such that, for all $A \in \cB_{\fro}(\Ast;\rcost(\Ast))$ and all $\K \in \Picon$, we have $\Exp_{A,\K}[ ( \sum_{h=1}^H \cost_h(\bx_h,\bu_h))^2] \le \Lcost$.
\end{asm}

\begin{asm}[Uniform Feature Excitation]\label{asm:full_rank_cov}
There exists $\omega \in \simplex_{\Piexp}$ such that $\lammin(\Exp_{\piexp \sim \omega}[\bLambda_{\piexp}]) \ge \lamminst$ for some $\lamminst > 0$.
\end{asm}

We remark that these assumptions have appeared before in work on systems of the form \eqref{eq:system} \citep{mania2020active,kakade2020information}. 
In order to precisely quantify the optimal rates of learning, we require that our system satisfy certain smoothness assumptions. First, we require that $\bphi(\cdot,\cdot)$ is differentiable in its second argument.

\begin{asm}[Smooth Nonlinearity]\label{asm:smooth_phi}
For all $\bx \in \R^{\dimx}$ and $\bu \in \cU$, $\bphi(\bx,\bu)$ is four-times differentiable in $\bu$. Furthermore, $\| \nabla^{(i)}_{\bu} \bphi(\bx,\bu) \|_\op \le \Lphi$, $\forall i \in \{ 1,2,3,4 \}$, $\bx \in \R^{\dimx}$, and $\bu \in \cU$.   
\end{asm}

\iftoggle{arxiv}{
We also require that the class of admissible control policies, $\Picon$, has the following parametric form:
\begin{align*}
\Picon = \{ \Kzeta \ : \ \btheta \in \R^{\dimtheta} \}
\end{align*}
and that the parameterization is smooth in the following sense.
}{
We also require that the class of admissible control policies, $\Picon$, has a parametric form, $\Picon = \{ \Kzeta \ : \ \btheta \in \R^{\dimtheta} \}$,
and that the parameterization is smooth in the following sense.}
\begin{asm}[Smooth Controller Class]\label{asm:smooth_controller}
$\Kzeta_h(\traj_{1:h})$ is four-times differentiable in $\btheta$ for all $\traj \in \cT$ and $h \in [H]$. Furthermore, $\| \nabla_{\btheta}^{(i)} \Kzeta_h(\traj_{1:h}) \|_\op \le \Lzeta$ for $\forall i \in \{1,2,3, 4\}$, $\btheta \in \R^{\dimtheta}$, and $\traj \in \cT$.  
\end{asm}
\Cref{asm:smooth_controller} is satisfied for commonly considered classes of controllers, such as linear controllers, but is also satisfied by more complex classes such as neural network controllers. While the learner may propose any $\Khat_T \in \Picon$, we are particularly interested in the \emph{certainty equivalence} decision rule (i.e., the learner decides $\Khat_T$ as if the estimated system is the actual one), defined as: 
\begin{align}\label{eq:Kst}
\textstyle \Kst(A) := \pi^{\bthetast(A)} \quad \text{for} \quad \bthetast(A) := \argmin_{\btheta \in \R^{\dimtheta}} \cJ(\K^{\btheta};A).
\end{align}
To ensure that $\Kst(A)$ is well-defined and sufficiently regular, we make the following assumption.
\begin{asm}[Unique Optimal Controller]\label{asm:unique_min}
We assume that the global minimum of $\cJ(\pi^{\btheta};\Ast)$, $\bthetast(\Ast)$, is unique, and that $\nabla_{\btheta}^2 \cJ(\pi^{\btheta};\Ast)|_{\btheta = \bthetast(\Ast)} \succ 0$.
\end{asm}
In general, the policy optimization problem in \eqref{eq:Kst} may not be computationally tractable. As we show in \Cref{sec:smooth_system}, the globally optimal decision rule of \eqref{eq:Kst} can be replaced with a locally optimal decision rule (i.e. $\Kst(A)$ a local minimum of $\cJ(\K;A)$). 
Furthermore, \Cref{asm:unique_min} can be replaced by assuming the differentiability of $\bthetast(A)$ with respect to $A$ for $A$ near $\Ast$. For ease of exposition, in the main text we assume that \Cref{asm:unique_min} holds and that $\Kst(A)$ is defined as in \eqref{eq:Kst}.
With these definitions and under \Cref{asm:bounded_features,asm:bounded_cost,asm:smooth_phi,asm:smooth_controller},
we can show that $\cJ(\pi^{\btheta};\Ast)$ is differentiable in $\btheta$ and, combined with \Cref{asm:unique_min}, that $\bthetast(A)$ is differentiable in $A$, for $A \in \cB_{\fro}(\Ast;\rtheta(\Ast))$ and some $\rtheta(\Ast) > 0$. 
We let $\LKst$ denote an upper bound on the norm of the derivatives of $\bthetast(A)$. \iftoggle{arxiv}{We always take $\Bphi,\Lcost,\Lphi,\Lzeta,\LKst \ge 1$. Additional discussion on the setting of $\pist(A)$ and the scaling of $\rtheta(\Ast)$ and $\LKst$ is given in \Cref{sec:smooth_system}.}{}

\section{Optimal Exploration in Nonlinear Systems}\label{sec:main_results}
In this work, we are interested in characterizing the instance-optimal rates of learning a controller $\K \in \Picon$ which minimizes the loss $\cJ(\K;\Ast)$. 
The following result, a generalization of Proposition 8.2 of \cite{wagenmaker2021task} to nonlinear systems, is the starting point of our analysis\iftoggle{arxiv}{, and precisely quantifies how estimation error translates to controller loss.}{.}

\begin{proposition}[Informal]\label{prop:quadratic_loss_approx}
Under \Cref{asm:bounded_features,asm:bounded_cost,asm:smooth_phi,asm:smooth_controller,asm:unique_min} and on the system \eqref{eq:system}, we have
\begin{align*}
\cJ(\Kst(\Ahat); \Ast) - \cJ(\Kst(\Ast); \Ast) = \| \vec(\Ast-\Ahat) \|_{\cH(\Ast)}^2 + \cOst(\| \Ast - \Ahat \|_\fro^3) 
\end{align*}
for
\begin{align*}
\cH(\Ast) := \nabla_A^2 \cJ(\Kst(A);\Ast)|_{A = \Ast}
\end{align*}
and where $\cOst(\cdot)$ hides factors polynomial in the regularity parameters of \Cref{asm:bounded_features,asm:full_rank_cov,asm:bounded_cost,asm:smooth_phi,asm:smooth_controller,asm:unique_min}.
\end{proposition}
The quantity $\cH(\Ast) := \nabla_A^2 \cJ(\Kst(A);\Ast)|_{A = \Ast}$, referred to as the \emph{model-task Hessian} in \cite{wagenmaker2021task}, corresponds to the \emph{curvature} of the loss of the certainty-equivalence controller $\Kst(A)$ around $A \leftarrow \Ast$. It precisely quantifies how estimation error in each coordinate of $\Ast$ translates into suboptimality of the controller---providing an answer to our question of which parameters are most relevant to learning a good controller---and reduces the problem of minimizing the controller loss to estimating $\Ast$ in a particular norm. The following result gives a bound on this estimation error, $\| \vec(\Ast-\Ahat) \|_{\cH(\Ast)}^2$.

\begin{proposition}[Informal]\label{prop:est_error_informal}
\iftoggle{arxiv}{
Consider interacting with \eqref{eq:system} for $T$ episodes, and let 
$$\bLambda_T = \sum_{t=1}^T \sum_{h=1}^H \bphi(\bx_h^t,\bu_h^t) \bphi(\bx_h^t,\bu_h^t)^\top$$ 
denote the observed covariates and
\begin{align*}
\Ahat = \argmin_{A} \sum_{t=1}^T \sum_{h=1}^H \| \bx_{h+1}^t - A \bphi(\bx_h^t,\bu_h^t) \|_2^2
\end{align*}}{
Consider interacting with \eqref{eq:system} for $T$ episodes, and let $\bLambda_T = \sum_{t=1}^T \sum_{h=1}^H \bphi(\bx_h^t,\bu_h^t) \bphi(\bx_h^t,\bu_h^t)^\top$ denote the observed covariates and
\begin{align*}
\textstyle \Ahat = \argmin_{A} \tsum_{t=1}^T \tsum_{h=1}^H \| \bx_{h+1}^t - A \bphi(\bx_h^t,\bu_h^t) \|_2^2
\end{align*}}
the least-squares estimate of $\Ast$. Recalling that $\bLamchk_T = I_{\dimx} \otimes \bLambda_T$,  we have, with high probability:
\begin{align*}
\| \vec(\Ast-\Ahat) \|_{\cH(\Ast)}^2 \lesssim \sigw^2 \cdot \tr( \cH(\Ast) \bLamchk_T^{-1}).
\end{align*}
\end{proposition}

\subsection{Algorithm and Upper Bound}
\iftoggle{arxiv}{\Cref{prop:est_error_informal} motivates our algorithmic approach: explore to collect covariates $\bLambda_T$ minimizing $\tr(\cH(\Ast) \bLamchk_T^{-1})$. There are two primary challenges to achieving this: we do not know $\cH(\Ast)$, as it depends on the (unknown) parameter $\Ast$ and, even if we did know $\cH(\Ast)$, it is not clear how to explore so as to collect data minimizing $\tr(\cH(\Ast) \bLamchk_T^{-1})$. We address both of these challenges with our main algorithm, \Cref{alg:main}. 
}{
\Cref{prop:est_error_informal} motivates our algorithmic approach: explore to collect covariates $\bLambda_T$ minimizing $\tr(\cH(\Ast) \bLamchk_T^{-1})$. There are two primary challenges to achieving this: we do not know $\cH(\Ast)$, as it depends on the (unknown) parameter $\Ast$ and, even if we did know $\cH(\Ast)$, it is not clear how to explore so as to collect data minimizing $\tr(\cH(\Ast) \bLamchk_T^{-1})$. We address both of these challenges in \Cref{alg:main}. \loose}

\begin{algorithm}[h]
\begin{algorithmic}[1]
\State \textbf{inputs:} episodes $T$, $(\cost_h)_{h=1}^H$, confidence $\delta$, control policies $\Picon$, exploration policies $\Piexp$
\State $\Ahat^1 \leftarrow $ 0, $\ellf \leftarrow \lceil \log_2 T/8 \rceil$, $T_\ell \leftarrow 2^\ell$
\For{$\ell = 1,2,3, \ldots, \ellf$}
	\State Compute estimate of model-task Hessian: $\cH_\ell \leftarrow \cH(\Ahat^\ell)$
	\State Run \expdesign on $\Phi_\ell(\bLambda) \leftarrow \tr(\cH_\ell \cdot \bLambda^{-1})$ to learn exploration policies $\Pi_\ell \subseteq \Piexp$
	\State Rerun each policy in $\Pi_\ell$ $N_\ell = \lceil T_\ell / |\Pi_\ell| \rceil$ times, denote collected data $\frakD_\ell$
	\State Estimate $\Ast$: $\Ahat^{\ell+1} = \argmin_A \sum_{h=1}^H \sum_{(\bx_{h+1},\bu_h,\bx_h) \in \frakD_\ell} \| \bx_{h+1} - A \bphi(\bx_{h},\bu_h) \|_2^2 $
\EndFor
\State \textbf{return} $\Khat_T \leftarrow \Kst(\Ahat^{\ellf+1}) \in \Picon$
\end{algorithmic}
\caption{Optimal Exploration in Nonlinear Systems (informal)}
\label{alg:main}
\end{algorithm}

\Cref{alg:main} proceeds in epochs of exponentially increasing length. At each epoch it first approximates $\cH(\Ast)$
by computing the model-task Hessian of the estimated system, $\Ahat^\ell$. Using this approximatiom of $\cH(\Ast)$, it seeks to explore to minimize $\tr(\cH(\Ahat^\ell) \bLamchk_T^{-1})$. This exploration routine is encapsulated in the \expdesign (dynamic optimal experiment design) function, an adaptive experiment-design routine inspired by recent work in reinforcement learning \citep{wagenmaker2022instance} and described in more detail in \Cref{sec:exp_design}. \expdesign returns a set of exploration policies, $\Pi_\ell$, which we run to collect data $\frakD_\ell$. As we will show, the collected covariates, $\bLambda_\ell := \sum_{h=1}^H \sum_{(\bu_h,\bx_h) \in \frakD_\ell} \bphi(\bx_h,\bu_h) \bphi(\bx_h,\bu_h)^\top$, satisfy \loose
\begin{align*}
\textstyle \tr(\cH(\Ahat^\ell) \bLamchk_\ell^{-1}) \lesssim T_\ell^{-1} \cdot \min_{\bLambda \in \bOmega} \tr(\cH(\Ahat^\ell) \bLamchk^{-1}),
\end{align*}
which implies that \expdesign collects data minimizing $\tr(\cH(\Ahat^\ell) \bLamchk_\ell^{-1})$ at a near-optimal rate. Given the data $\frakD_\ell$, we form the least-squares estimate of $\Ast$, $\Ahat^{\ell+1}$, and the process repeats. After running for $T$ episodes, the certainty-equivalence controller on the last estimate obtained, $\Khat_T = \Kst(\Ahat^{\ellf+1})$, is returned. 
The following result bounds the suboptimality of $\Khat_T$ as compared to $\pist(\Ast)$. 

\begin{theorem}\label{thm:main}
Under \Cref{asm:bounded_features,asm:full_rank_cov,asm:bounded_cost,asm:smooth_phi,asm:smooth_controller,asm:unique_min}, if $T \ge \Cpoly \cdot \max \{1,\rcost(\Ast)^{-2}, \rtheta(\Ast)^{-2} \}$, then with probability at least $1-\delta$, \Cref{alg:main} explores with policies in $\Piexp$ at every episode, runs for at most $T$ episodes, and returns $\Khat_T \in \Picon$ satisfying:
\begin{align*}
\cJ(\Khat_T;\Ast) - \cJ(\Kst(\Ast);\Ast) \le  \frac{\sigw^2}{T} \cdot \min_{\bLambda \in \bOmega} \tr \left ( \cH(\Ast) \bLamchk^{-1} \right )  \cdot  C \log \frac{6 \dimx \dimphi}{\delta} + \frac{\Cpoly}{T^{3/2}}
\end{align*}
where we recall $\bOmega$ is the set of possible expected covariates on \eqref{eq:system}, $C$ is a universal constant, and 
\begin{align*}
\Cpoly = \poly  ( \dimphi, \dimx, H, \BA, \Bphi, \Lphi, \Lzeta, \Lcost, \LKst, \sigw, \sigw^{-1}, \tfrac{1}{\lamminst}, \log \tfrac{T}{\delta}  ).
\end{align*}
\end{theorem} 
\iftoggle{arxiv}{\Cref{thm:main} shows that \Cref{alg:main} is able to explore so as to optimally minimize the exploration loss $\tr(\cH(\Ast) \bLamchk_T^{-1})$, up to a lower-order term scaling as $T^{-3/2}$ and polynomially in system parameters. While \Cref{prop:quadratic_loss_approx,prop:est_error_informal} together show that collecting data which minimizes $\tr(\cH(\Ast) \bLamchk_T^{-1})$ is in some sense fundamental to minimizing the cost of the certainty equivalent controller, it is not clear that this is necessary. In the following section, we show that this is indeed the case.}{
\Cref{thm:main} shows that \Cref{alg:main} is able to explore so as to optimally minimize the exploration loss $\tr(\cH(\Ast) \bLamchk_T^{-1})$, up to a lower-order term scaling as $T^{-3/2}$ and polynomially in system parameters. While \Cref{prop:quadratic_loss_approx,prop:est_error_informal} show that collecting data which minimizes $\tr(\cH(\Ast) \bLamchk_T^{-1})$ is in some sense fundamental, it is not clear it is necessary. We next show that it is indeed necessary.\loose}

\iftoggle{arxiv}{
\begin{remark}[Comparison to \tople Algorithm of \cite{wagenmaker2021task}]
\Cref{alg:main} bears many similarities to the \tople algorithm of \cite{wagenmaker2021task}, which performs an analogous task-driven exploration routine, but in the setting of linear dynamical systems. As noted in \Cref{sec:motivating_example}, the key challenge present in the nonlinear case compared to the linear is that, while in the linear case random noise will excite every direction, in the nonlinear case, the learner must actually traverse the system in order to reach the states that will excite the nonlinear modes. Though the overall structure of \Cref{alg:main} is similar to \tople, this added challenge requires a much more powerful exploration routine, encapsulated in the \expdesign function and described in more detail in \Cref{sec:exp_design}.
\end{remark}

\begin{remark}[Computational Efficiency of \Cref{alg:main}]
The primary computational burden of \Cref{alg:main} is in the computation of $\cH(\Ahat^\ell)$---which involves differentiating $\pist(\Ahat^\ell)$---the computation of $\pist(\Ahat^{\ell_T+1})$, and the \expdesign subroutine. 
In general, if we define $\pist(A)$ as  in \eqref{eq:Kst}, it may not be efficiently computable, as it involves solving a possibly non-convex optimization problem. 
However, as we show in \Cref{sec:smooth_system}, we can instead set $\pist(A)$ to correspond to a local minimum rather than a global minimum of the loss, which will render it efficiently computable (though note that \Cref{thm:main} will still in this case only bound the suboptimality of $\pihat_T$ as compared to $\pist(\Ast)$).
We discuss the computational efficiency of \expdesign in more detail in \Cref{sec:exp_design}, but note that in general it may not be computationally efficient as it relies on calls to the \LC algorithm of \cite{kakade2020information}, which requires access to a computational oracle. 
Despite these computational challenges, in \Cref{sec:experiments} we demonstrate that in practice, by making several reasonable approximations, \Cref{alg:main} can be implemented efficiently, and that this efficient implementation performs very well on realistic systems. \loose
\end{remark}
}{}

\subsection{Lower Bounds on Learning Controllers}\label{sec:lb}

Our goal is to show that, up to constants and lower-order terms, the bound given in \Cref{thm:main} is not improvable, regardless of which controller estimate we use. To obtain such lower bounds, we need several additional assumptions. In particular, we require that the loss $\cJ(\Ktheta;A)$ grows quadratically in the distance $\btheta$ is from $\bthetast(A)$, and strengthen \Cref{asm:full_rank_cov} to ensure \eqref{eq:system} is sufficiently easy to excite. Formal statements of these conditions are given in \Cref{sec:lb_proofs}. Our lower bound is as follows.

\begin{theorem}[Informal]\label{thm:lb_informal}
Under \Cref{asm:bounded_features,asm:full_rank_cov,asm:bounded_cost,asm:smooth_phi,asm:smooth_controller,asm:unique_min} and the additional regularity assumptions mentioned above, as long as $T \ge \Clb$, for any $\omegaexp \in \simplex_{\Piexp}$, we have
\begin{align*}
\min_{\Khat} \max_{A \in \cB_T} \Exp_{\frakD_T \sim A,\omegaexp} [\cJ(\Khat(\frakD_T);A) - \cJ(\Kst(A);A)] \ge \frac{\sigw^2}{3T} \cdot \min_{\bLambda \in \bOmega} \tr(\cH(\Ast) \bLamchk^{-1}) - \frac{\Clb}{T^{5/4}} 
\end{align*}
for $\cB_T := \cB_{\fro}(\Ast; \cO(T^{-5/6}))$, $\Exp_{\frakD_T \sim A,\omegaexp}[\cdot] = \Exp_{\piexp \sim \omegaexp}[\Exp_{\frakD_T \sim A, \piexp}[\cdot]]$ the expectation over trajectories generated by running policies $\pi \sim \omegaexp$ on system $A$ for $T$ episodes, $\pihat$ any mapping from observations to policies in $\Picon$, and $\Clb$ some value scaling polynomially in problem parameters.
\end{theorem}

Note that this lower bound holds for \emph{any} $\Ast$ and mapping $\bphi$, as long as our assumptions are met. Up to constants and lower-order terms, the scaling of \Cref{thm:lb_informal} matches that of \Cref{thm:main}---both scale with $\min_{\bLambda \in \bOmega} \tr(\cH(\Ast) \bLamchk^{-1})$---which implies that \Cref{alg:main} is indeed optimal (under certain additional regularity conditions). 
To the best of our knowledge, this is the first result characterizing the optimal statistical rates for learning in nonlinear dynamical systems.
We emphasize that \Cref{thm:lb_informal} holds for \emph{any} decision rule $\Khat$---it does not require that we use the certainty equivalence decision rule. As \Cref{alg:main} does rely on certainty equivalence, this result also implies that the certainty equivalence decision rule is optimal for (certain classes of) nonlinear dynamical systems.

\iftoggle{arxiv}{
The proof of \Cref{thm:lb_informal} builds on the work \cite{wagenmaker2021task}, which shows a similar result for linear dynamical systems. It critically relies on our quadratic decomposition of the controller loss in \Cref{prop:quadratic_loss_approx}, which reduces the problem of obtaining a lower bound on controller loss to a lower bound on estimating $\Ast$ in the $\cH(\Ast)$ norm. Given this, the result can be obtained by applying lower bounds on regression in general norms. 
}{

}

\section{Optimal Experiment Design in Arbitrary Dynamical Systems}\label{sec:exp_design}

We turn now to the \expdesign routine, which is the key algorithmic tool we use to prove \Cref{thm:main}. \expdesign is a general reduction from policy optimization to optimal experiment design in arbitrary dynamical systems, and is an extension of a recently proposed approach for experiment design in linear MDPs \citep{wagenmaker2022instance}. This section may be of independent interest.

To illustrate the generality of this reduction, in this section we consider the following system: 
\begin{align}\label{eq:system_general}
\bx_{h+1} = f_h(\bx_h,\bu_h,\bw_h), \quad h = 1,2, \ldots , H,
\end{align}
where $\bx_h \in \cX \subseteq \R^{\dimx}$ denotes the state, $\bu_h \in \cU \subseteq \R^{\dimu}$ the input, and $\bw_h \in \R^{\dimw}$ the noise. We take the dynamics $(f_h)_{h=1}^H$ to be unknown and arbitrary.
We assume there is some known featurization of our system that is of interest, $\bphi(\bx,\bu) \rightarrow \R^{\dimphi}$, and an experiment design object on this featurization, $\Phi : \R^{\dimphi \times \dimphi} \rightarrow \R$. Our goal is to collect some set of trajectories $\{ \traj_t \}_{t=1}^T$ which minimizes $\Phi$:
\begin{align*}
\Phi \big ( \tfrac{1}{TH} \cdot  \tsum_{t=1}^T \tsum_{h=1}^H \bphi(\bx_h^t,\bu_h^t) \bphi(\bx_h^t,\bu_h^t)^\top \big ). 
\end{align*}
As an example, if $\Phi(\bLambda) = \logdet (\bLambda)$, this reduces to $D$-optimal design, and if $\Phi(\bLambda) = \tr(\cH \cdot \bLambda^{-1})$, the setting considered in \Cref{sec:main_results}, this reduces to weighted $A$-optimal design.
As before, we assume we have access to some set of exploration policies $\Piexp$, and define $\bLambda_\pi$ and $\bOmega$ as in \Cref{sec:prelim}, but with respect to this new feature map $\bphi$ and system \eqref{eq:system_general}.
We also define $\bOmegahat$ to be the space of all possible covariance matrices: 
\begin{align*}
\bOmegahat := \big \{ \tsum_{h = 1}^H \bphi(\bx_h,\bu_h) \bphi(\bx_h,\bu_h)^\top \ : \ \bx_h \in \cX, \bu_h \in \cU, \forall h \in [H]  \big \}.
\end{align*}

\noindent To facilitate efficient experiment design in this setting, we will make the following assumption on $\Phi$.
\begin{asm}[Regularity of $\Phi$]\label{asm:regular_phi}
$\Phi$ is regular in the following sense:
\begin{enumerate}[leftmargin=*]
\item $\Phi$ is convex, differentiable, and $\beta$-smooth in the norm $\| \cdot \|$ (with dual-norm $\| \cdot \|_*$):
\begin{align*}
\| \nabla_{\bLambda} \Phi(\bLambda) - \nabla_{\bLambda'} \Phi(\bLambda') \|_* \le \beta \cdot \| \bLambda - \bLambda' \|, \quad \forall \bLambda,\bLambda' \in \bOmegahat.
\end{align*} 
\item There exists some $M < \infty$ satisfying $\sup_{\bLambda \in \bOmegahat} \sup_{\bx \in \cX, \bu \in \cU} | \bphi(\bx,\bu)^\top \nabla_{\bLambda} \Phi(\bLambda) \bphi(\bx,\bu) | \le M.$
\end{enumerate}
\end{asm}

\noindent The key algorithmic assumption we make is access to a regret minimization oracle on \eqref{eq:system_general}.

\begin{asm}[Regret Minimization Oracle]\label{asm:regret_min}
Let $\cost_h(\bx,\bu) = \bphi(\bx,\bu)^\top Q_h\bphi(\bx,\bu)$ for some $Q_h \in \R^{\dimphi \times \dimphi}$ such that $|\sum_h \cost_h(\bx_h,\bu_h)| \le 1$ for all $\bx_h \in \cX,\bu_h \in \cU$. We assume we have access to some learner $\regalg$ which is able to achieve low regret on costs $\{ \cost_h(\cdot,\cdot) \}_{h=1}^H$ with respect to policy class $\Piexp$. That is, with probability at least $1-\delta$:
\iftoggle{arxiv}{\begin{align*}
\sum_{t=1}^T \Exp_{f,\pi_t} \left [\sum_{h=1}^H \cost_h(\bx_h^t,\bu_h^t) \right ] - T \cdot \min_{\pi \in \Piexp} \Exp_{f,\pi} \left [ \sum_{h=1}^H \cost_h(\bx_h,\bu_h) \right ] \le \CR \cdot \log^{\pR} \tfrac{T}{\delta} \cdot T^\alpha
\end{align*}}{
\begin{align*}
\textstyle \tsum_{t=1}^T \Exp_{f,\pi_t} [\tsum_{h=1}^H \cost_h(\bx_h^t,\bu_h^t) ] - T \cdot \min_{\pi \in \Piexp} \Exp_{f,\pi} [ \tsum_{h=1}^H \cost_h(\bx_h,\bu_h)  ] \le \CR \cdot \log^{\pR} \tfrac{T}{\delta} \cdot T^\alpha
\end{align*}}
for some $\CR > 0$, $\pR > 0$, and $\alpha \in (0,1)$, and where $\pi_t$ is the policy $\regalg$ plays at episode $t$. 
\end{asm}

Note that the regret minimization algorithm satisfying \Cref{asm:regret_min} may be arbitrary. 
For example, for linear systems, we could apply provably efficient algorithms for the Linear Quadratic Regulator \citep{simchowitz2020naive,mania2019certainty}; for nonlinear systems of the form \eqref{eq:system} we could apply the \LC algorithm of \citep{kakade2020information}; for more general settings of reinforcement learning with function approximation, algorithms such as \textsc{BiLin-UCB} \citep{du2021bilinear} or \textsc{E2D} \citep{foster2021statistical} could be applied. In practice, though they may not formally satisfy the guarantee of \Cref{asm:regret_min}, deep RL approaches could be used. We have the following result.

\begin{theorem}\label{lem:fwregret}
Fix $T > 0$ and denote $R := \sup_{\bLambda,\bLambda' \in \bOmegahat} \| \bLambda - \bLambda' \|$. Under \Cref{asm:regular_phi}, and assuming we have access to a learner $\regalg$ satisfying \Cref{asm:regret_min} with $\alpha = 1/2$, \expdesign runs for $T$ episodes on \eqref{eq:system_general}, and with probability at least $1-\delta$ collects data $\{ (\bx_h^t,\bu_h^t) \}_{h \in [H], t \in [T]}$ satisfying
\iftoggle{arxiv}{\begin{align*}
\Phi \bigg ( \frac{1}{T} \cdot \sum_{t = 1}^T \sum_{h=1}^H \bphi(\bx_h^t,\bu_h^t) \bphi(\bx_h^t,\bu_h^t)^\top \bigg ) -  \min_{\bLambda \in \bOmega} \Phi( \bLambda) \le \frac{\beta R^2 \log T +  HM(\CR \log^{\pR} \frac{2T}{\delta} + 3\log^{1/2} \frac{4T}{\delta})}{T^{1/3}}
\end{align*}
where $R = \sup_{\bLambda,\bLambda' \in \bOmegahat} \| \bLambda - \bLambda' \|$.}{
\begin{align*}
\Phi \bigg ( \frac{1}{T} \cdot \sum_{t = 1}^T \sum_{h=1}^H \bphi_h^t (\bphi_h^t)^\top \bigg ) -  \min_{\bLambda \in \bOmega} \Phi( \bLambda) \le \frac{\beta R^2 \log T +  HM(\CR \log^{\pR} \frac{2T}{\delta} + 3\log^{1/2} \frac{4T}{\delta})}{T^{1/3}}
\end{align*}
where $\bphi_h^t := \bphi(\bx_h^t,\bu_h^t)$, and we recall $\bOmega$ is the set of possible expected covariates on \eqref{eq:system_general}.}
\end{theorem}

\Cref{lem:fwregret} shows that, given access only to a regret minimization oracle, it is possible to solve experiment design problems on arbitrary dynamical systems. 
The requirement that $\alpha = 1/2$ is for expositional purposes only---we generalize this result to arbitrary $\alpha$ (and more general feature maps) in \Cref{sec:nl_exp_design}.
Under certain conditions, it can be shown that, if the exploration policies \expdesign runs to collect $\frakD$ are \emph{rerun}, the newly collected data satisfies a similar guarantee as \Cref{lem:fwregret}. This lets us run \expdesign to learn an approximate solution of $\min_{\bLambda \in \bOmega} \Phi( \bLambda)$, and then rerun the learned policies as many times as desired to collect additional data approximately minimizing $\Phi$.\loose

\iftoggle{arxiv}{}{\vspace{-0.25em}}
\subsection{Overview of \expdesign Algorithm}\label{sec:expdesign_overview}
\iftoggle{arxiv}{}{\vspace{-0.25em}}

\expdesign is inspired by recent work on experiment design in reinforcement learning \citep{hazan2019provably,zahavy2021reward,wagenmaker2022instance,wagenmaker2022leveraging}, and can be seen as an extension of the \fwregret algorithm of \cite{wagenmaker2022instance} to arbitrary systems. We refer the reader to \cite{wagenmaker2022instance} for a more in-depth discussion of the \fwregret algorithm, and briefly sketch its extension to arbitrary systems here (see \Cref{sec:nl_exp_design} and \Cref{alg:regret_fw} for precise definitions).

\begin{algorithm}[H]
\begin{algorithmic}[1]
\State \textbf{input}: objective $\Phi$, episodes $T$, confidence $\delta$, regret algorithm $\regalg$, exploration policies $\Piexp$
\State Set $K \leftarrow \cO(T^{2/3}), N \leftarrow \cO(T^{1/3})$, $\gamma_n \leftarrow \frac{1}{n+1}$
\Statex {\color{blue} // $\bphi_h^{k,n} := \bphi(\bx_h^{k,n},\bu_h^{k,n})$ for $(\bx_h^{k,n},\bu_h^{k,n})$ the state-input at step $h$ of episode $k$ of iteration $n$}
\State Play any $\piexp \in \Piexp$ for $K$ episodes, set $\bLambda_0 \leftarrow \frac{1}{K} \sum_{k=1}^K \sum_{h=1}^H \bphi_h^{k,0} (\bphi_h^{k,0})^\top$ 
\For{$n = 1,2,\ldots,N$}
	\State Compute derivative of $\Phi(\bLambda_{n-1})$, $\Xi_{n} \leftarrow \nabla_{\bLambda} \Phi(\bLambda)|_{\bLambda = \bLambda_{n-1}}$
	\State Run $\regalg$ on cost $\cost_h^n(\bx,\bu) \leftarrow \frac{1}{M} \cdot \bphi(\bx,\bu)^\top (\Xi_{n}) \bphi(\bx,\bu)$ \label{line:fw_main_regret}
	for $K$ episodes
	\State $\bLambda_{n} \leftarrow (1-\gamma_{n}) \bLambda_{n-1} + \frac{\gamma_{n}}{K} \cdot \sum_{k=1}^K  \sum_{h=1}^H  \bphi_h^{k,n} (\bphi_h^{k,n})^\top$ \label{line:fw_main_lam_update}
\EndFor
\State \textbf{return} $\frac{1}{T} \sum_{n=0}^{N} \sum_{k = 1}^K \sum_{h=1}^H  \bphi_h^{k,n} (\bphi_h^{k,n})^\top$
\end{algorithmic}
\caption{Dynamic Optimal Experiment Design (\expdesign, Informal)}
\label{alg:regret_fw_informal}
\end{algorithm}

\iftoggle{arxiv}{
Conceptually, \expdesign runs a variant of conditional gradient descent on the objective $\Phi(\bLambda)$. At each iteration, $n$, it computes the gradient of the loss at the current iterate, $\Xi_{n} \leftarrow \nabla_{\bLambda} \Phi(\bLambda)|_{\bLambda = \bLambda_{n-1}}$. To run a standard gradient descent algorithm on this objective, we would simply update $\bLambda_{n-1}$ by taking a step in the direction $\Xi_{n}$. However, our objective is to minimize $\Phi$ over the constraint set, $\bOmega$. Thus, rather than taking a step in the direction $\Xi_n$, we wish to take a step in the direction of steepest descent \emph{within the constraint set}.

The challenge is that the constraint set in our setting, $\bOmega$, is \emph{unknown}, as it depends on the expectation over trajectories induced on the unknown dynamics $(f_h)_{h=1}^H$, and therefore we cannot directly compute this steepest descent direction. The key observation is that the computation of this steepest descent direction is equivalent to solving:
\begin{align*}
 \argmin_{\piexp \in \Piexp} \Exp_{f,\piexp} \left [ \sum_{h = 1}^H  \bphi(\bx_h,\bu_h)^\top (\Xi_n) \bphi(\bx_h,\bu_h) \right ].
\end{align*}
This is simply a policy optimization problem, however, and can be solved approximately by $\regalg$ under \Cref{asm:regret_min}. Thus, in the call to $\regalg$ on \Cref{line:fw_main_regret}, we approximate the steepest descent direction, and on \Cref{line:fw_main_lam_update}
update $\bLambda_{n-1}$ in this direction. Convergence of this procedure to the optimal value, $\min_{\bLambda \in \bOmega} \Phi(\bLambda)$, can then be shown by the standard analysis of conditional gradient descent. We remark that, under \Cref{asm:regular_phi} and \Cref{asm:regret_min}, this argument is completely generic and does not require that our system, \eqref{eq:system_general}, exhibit any additional properties. 
}{
\vspace{-1em}
Conceptually, \expdesign runs a variant of conditional gradient descent on the objective $\Phi(\bLambda)$. At each iteration, $n$, it computes the gradient of the loss at the current iterate, $\Xi_{n} \leftarrow \nabla_{\bLambda} \Phi(\bLambda)|_{\bLambda = \bLambda_{n-1}}$. To run a standard gradient descent algorithm on this objective, we would simply update $\bLambda_{n-1}$ by taking a step in the direction $\Xi_{n}$. However, our objective is to minimize $\Phi$ over the constraint set, $\bOmega$. Thus, rather than taking a step in the direction $\Xi_n$, we wish to take a step in the direction of steepest descent \emph{within the constraint set}.

The challenge is that the constraint set in our setting, $\bOmega$, is \emph{unknown}, as it depends on the expectation over trajectories induced on the unknown dynamics $(f_h)_{h=1}^H$, and therefore we cannot directly compute this steepest descent direction. The key observation is that the computation of this steepest descent direction is equivalent to solving:
\begin{align*}
\textstyle \argmin_{\piexp \in \Piexp} \Exp_{f,\piexp}[ \tsum_{h = 1}^H  \bphi(\bx_h,\bu_h)^\top (\Xi_n) \bphi(\bx_h,\bu_h)].
\end{align*}
This is simply a policy optimization problem, however, and can be solved approximately by $\regalg$ under \Cref{asm:regret_min}. Thus, in the call to $\regalg$ on \Cref{line:fw_main_regret}, we approximate the steepest descent direction, and on \Cref{line:fw_main_lam_update}
update $\bLambda_{n-1}$ in this direction. Convergence of this procedure to the optimal value, $\min_{\bLambda \in \bOmega} \Phi(\bLambda)$, can then be shown by the standard analysis of conditional gradient descent. We remark that, under \Cref{asm:regular_phi} and \Cref{asm:regret_min}, this argument is completely generic and does not require that our system, \eqref{eq:system_general}, exhibit any additional properties. 
}

\iftoggle{arxiv}{}{\vspace{-0.25em}}
\subsection{From \Cref{lem:fwregret} to \Cref{thm:main}}
\iftoggle{arxiv}{}{\vspace{-0.25em}}
In \Cref{alg:main}, our goal is to collect covariates, $\bLamchk_{T_\ell}^{-1}$, such that $\tr(\cH(\Ahat^{\ell}) \bLamchk_{T_\ell}^{-1})$ is as small as possible. To achieve this, we apply \expdesign to the objective $\Phi_\ell(\bLambda) = \tr(\cH(\Ahat^{\ell}) \bLamchk^{-1})$, with \Cref{asm:regret_min} instantiated by the \LC algorithm of \cite{kakade2020information}. By the guarantee given in \Cref{lem:fwregret}, after running for a number of episodes $N$ which scales polynomially in problem parameters,
\expdesign will collect covariates $\bLambda_N$ such that $\Phi_\ell(\frac{1}{N} \bLambda_N) \le 2 \cdot \min_{\bLambda \in \bOmega} \Phi_\ell(\bLambda)$, which implies $\tr(\cH(\Ahat^{\ell}) \bLamchk_N^{-1}) \le \frac{2}{N} \cdot \min_{\bLambda \in \bOmega} \tr(\cH(\Ahat^{\ell}) \bLamchk^{-1})$.
By rerunning the policies \expdesign used to collect this $\bLambda_N$ for $T_\ell/N$ additional times, we can ensure $\tr(\cH(\Ahat^{\ell}) \bLamchk_{T_\ell}^{-1}) \lesssim \frac{1}{T_\ell} \cdot \min_{\bLambda \in \bOmega} \tr(\cH(\Ahat^\ell) \bLamchk^{-1})$ as desired.

We remark that the \LC algorithm requires access to a computation oracle. As the focus of this work is primarily statistical, we leave addressing this computational challenge for future work. Furthermore, as we show in the following section, computationally efficient, sampling-based implementations of our approach are very effective in practice.
We remark as well that the objective we ultimately care about minimizing is $\tr(\cH(\Ast) \bLamchk^{-1})$. 
As we show, by including a small amount of uniform exploration, we can ensure that \iftoggle{arxiv}{$\cH(\Ahat^\ell)$ is not too far from $\cH(\Ast)$, and so}{} the suboptimality incurred optimizing $\tr(\cH(\Ahat^\ell) \bLamchk^{-1})$ instead of $\tr(\cH(\Ast) \bLamchk^{-1})$ only contributes to the lower-order terms of the final guarantee in \Cref{thm:main}.

\section{Experimental Results}\label{sec:experiments}

\begin{figure}
\begin{minipage}[c]{0.45\linewidth}
\includegraphics[width=\linewidth]{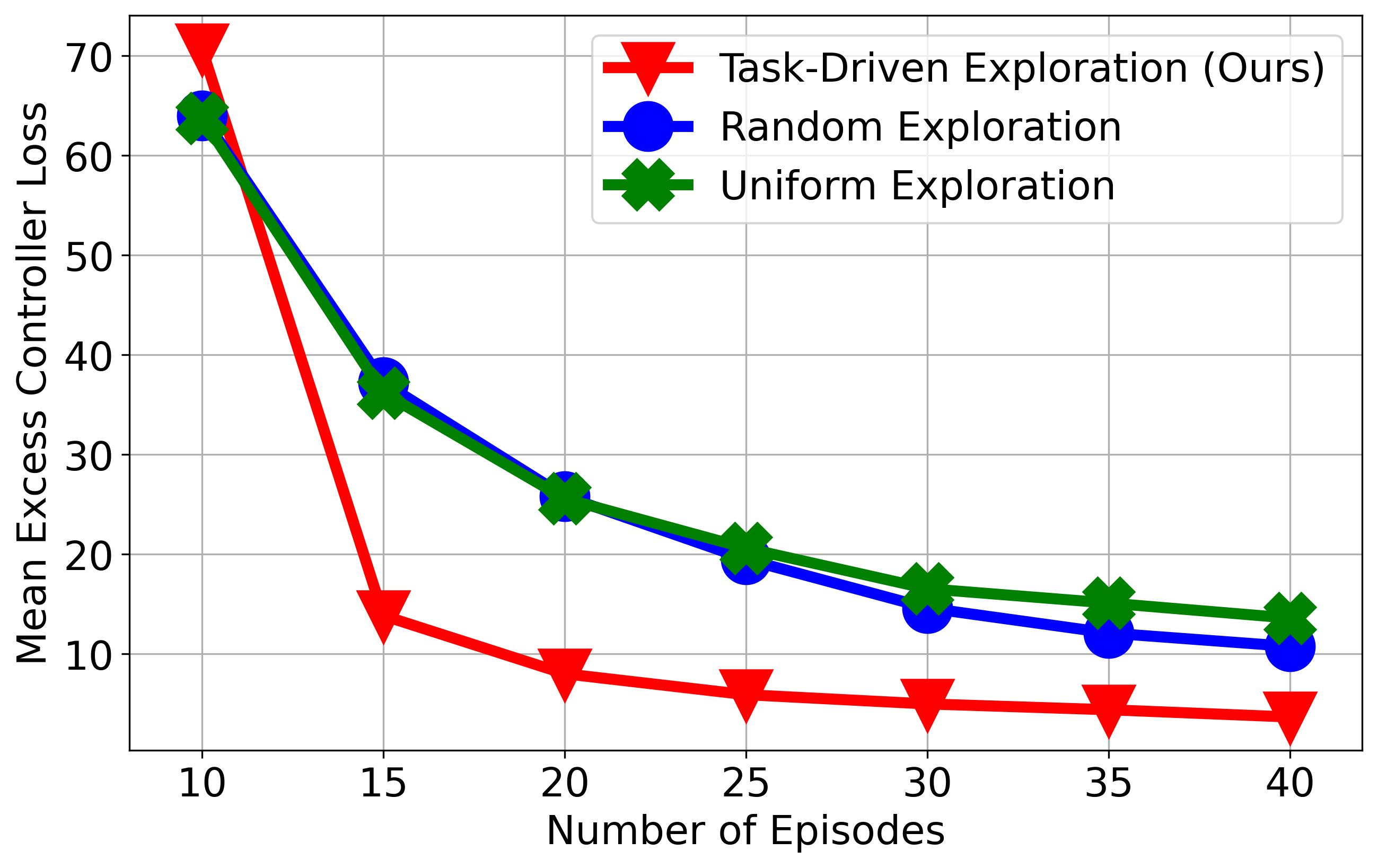}
\caption{Performance on Drone}
\label{fig:drone}
\end{minipage}
\hfill
\begin{minipage}[c]{0.45\linewidth}
\includegraphics[width=\linewidth]{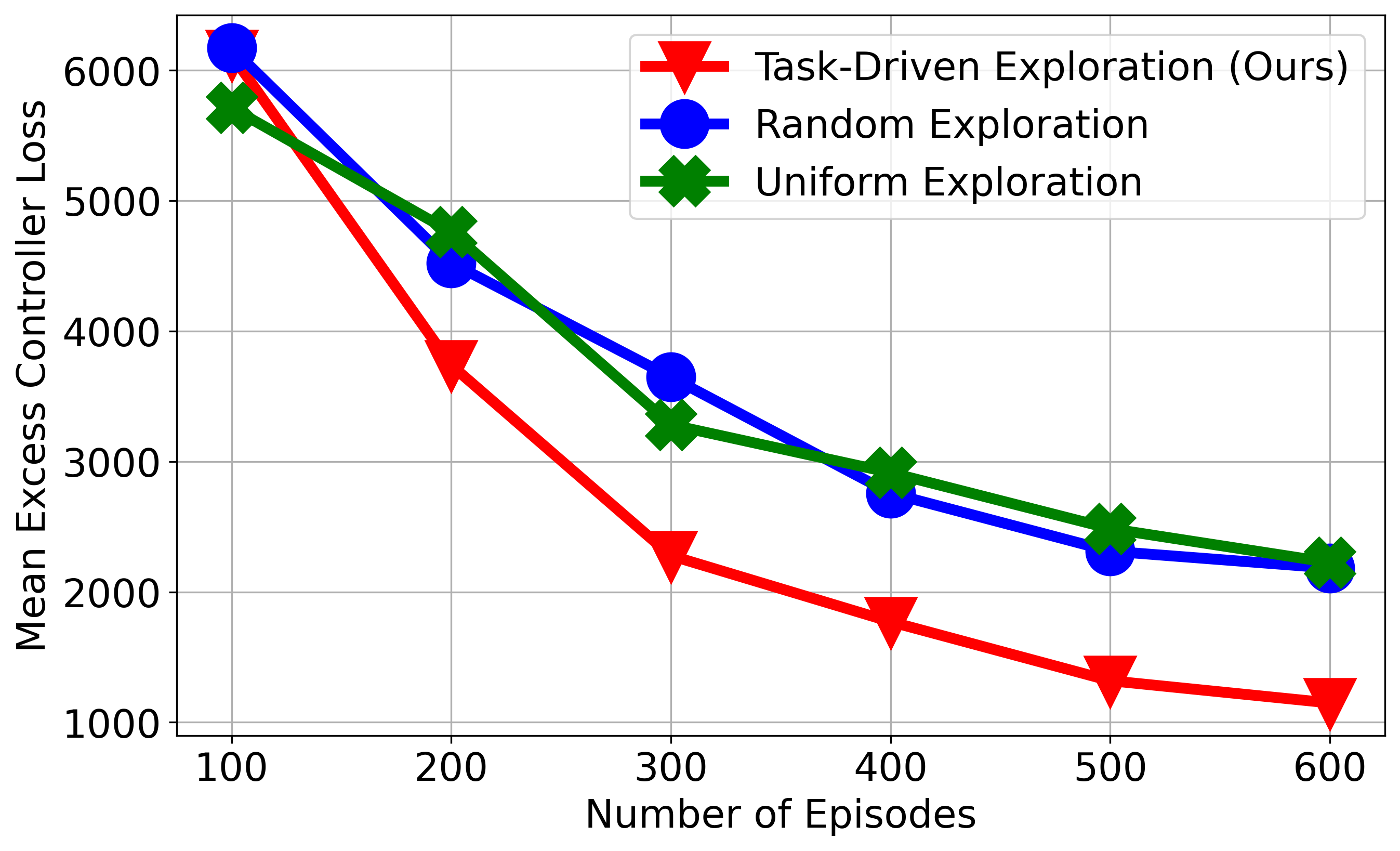}
\caption{Performance on Car}
\label{fig:car}
\end{minipage}%
\vspace{-1em}
\end{figure}

\iftoggle{arxiv}{}{
\begin{wrapfigure}{r}{0.44\textwidth}
\vspace{-5em}
  \begin{center}
    \includegraphics[width=1.0\linewidth]{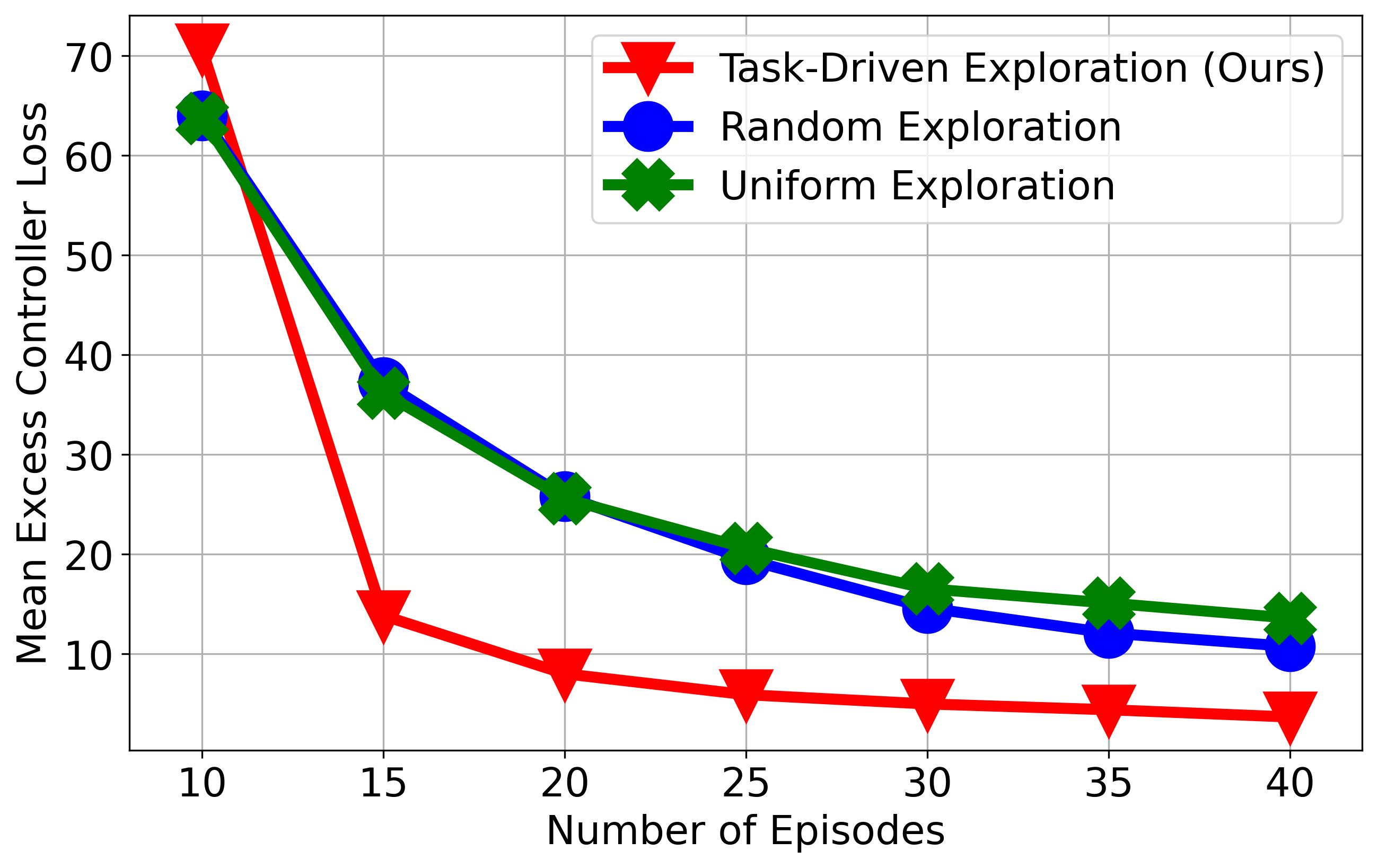}
    \vspace{-2em}
    \caption{Performance on Drone}
\label{fig:drone}
  \end{center}
\end{wrapfigure}}

Finally, we demonstrate the effectiveness of our proposed approach (\Cref{alg:main}, the \texttt{Task-Driven Exploration} method in \Cref{fig:motivating,fig:drone,fig:car}) on several systems motivated by robotic applications. 
We compare \Cref{alg:main} with an approach that plays $\bu_h \sim \cN(0,\sigma_{\bu}^2 \cdot I)$ (\texttt{Gaussian Exploration}), and an approach inspired by \cite{mania2020active} (\texttt{Uniform Exploration}), which seeks to estimate $\Ast$ uniformly well, playing inputs that reduce $\| \Ahat - \Ast \|_\op$.

\iftoggle{arxiv}{}{
\begin{wrapfigure}{r}{0.44\textwidth}
\vspace{-2.5em}
  \begin{center}
    \includegraphics[width=1.0\linewidth]{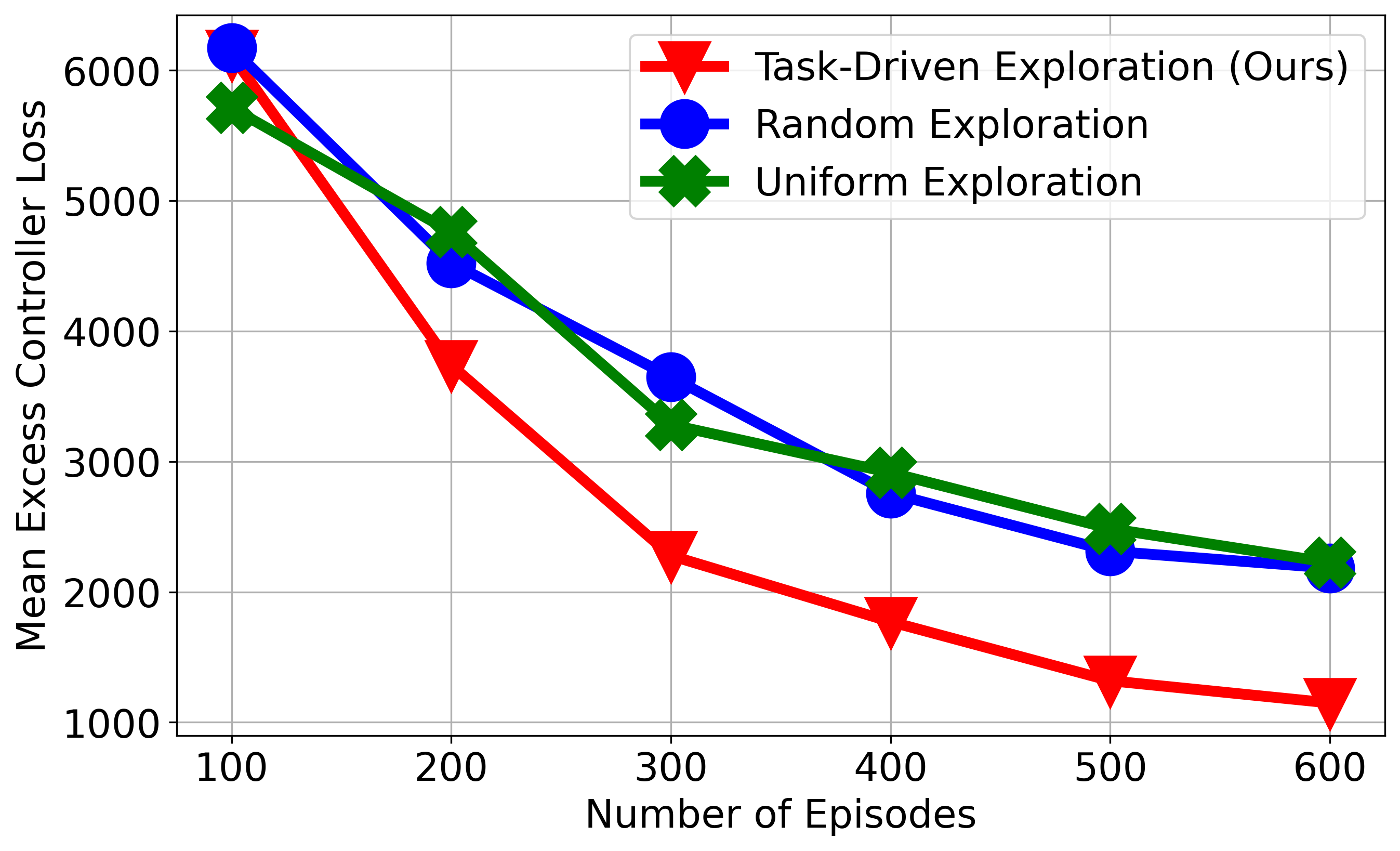}
    \vspace{-2em}
    \caption{Performance on Car}
\label{fig:car}
  \end{center}
  \vspace{-1em}
\end{wrapfigure}}

To benchmark the performance of these approaches, we consider an affine system with dynamics corresponding to that of a simplified 3-D drone (i.e., 3-D double integrator with a gravity term), and a nonlinear system with dynamics corresponding to that of a 2-D car. For both systems, we choose $H = 50$, and plot the value of $\cJ(\pihat_t;\Ast) - \cJ(\pist(\Ast);\Ast)$ for $\pihat_t$ the certainty-equivalence controller computed on the estimate of the system obtained at time $t$. For the drone, we let $\Picon$ be the class of linear-affine feedback controllers, and for the car, $\Picon$ is a set of nonlinear controllers with dimension 4. While the optimal controller for the drone can be computed in closed-form, for the car we rely on a sampling-based routine to find an approximately optimal controller. The model-task hessian $\cH(\Ahat^\ell)$ is computed via automatic differentiation.
For the exploration policies of \texttt{Task-Driven Exploration} and \texttt{Uniform Exploration}, $\Piexp$, we rely on MPC-style sampling based methods. For all approaches, we require that $\Exp_{A,\piexp}[\sum_{h=1}^H \| \bu_h \|_2^2] \le \gamma^2$ for some $\gamma^2 > 0$ and all $\piexp \in \Piexp$. 
On all examples, we implement \expdesign with $\regalg$ a posterior sampling-inspired version of the \LC of \cite{kakade2020information}.
\Cref{fig:motivating,fig:car} shows performance averaged over 100 trials, and \Cref{fig:drone} over 200 trials. Additional experimental details can be found in \Cref{sec:experiment_details}.

As illustrated in \Cref{fig:motivating,fig:drone,fig:car}, our approach yields a non-trivial gain over existing approaches on all systems. In particular, in \Cref{fig:drone,fig:car} it improves on the sample complexity of existing approaches by roughly a factor of 2---for example, in the drone system, reaching excess controller cost of $10$ after less than $20$ episodes, as compared to over $40$ episodes for existing approaches.

\iftoggle{arxiv}{
Our implementation is very modular, and any piece (for example, the parameterization of $\Piexp$ and $\Picon$, the policy optimizer, or the exploration routine) can be easily replaced with other procedures. Our results therefore highlight that, even when using, for example, a possibly suboptimal policy optimizer, exploring so as to minimize uncertainty in the model-task hessian yields a non-trivial gain. We expect that this would hold true regardless of the policy optimizer used---the model-task hessian will adapt to the structure of the policy optimizer, inducing the exploration that will minimize parameter uncertainty most relevant to the given optimizer. Integration of our approach with deep model-based RL approaches is an interesting direction for future work, but we believe the approach will scale to these settings as well.}
{
Our implementation is very modular, and any piece (for example, the parameterization of $\Piexp$ and $\Picon$, the policy optimizer, or the exploration routine) can be easily replaced with other procedures. Our results highlight that, even when using, for example, a possibly suboptimal policy optimizer, exploring so as to minimize uncertainty in the model-task hessian yields a non-trivial gain. We expect that this would hold true regardless of the policy optimizer used---the model-task hessian will adapt to the structure of the policy optimizer. Integration of our approach with deep model-based RL approaches is an interesting direction for future work, but we believe the approach will scale to these settings as well. \loose}

\iftoggle{arxiv}{
\subsection*{Acknowledgements}
AW would like to thank Kevin Tully for helpful discussions. The work of AW is supported by NSF HDR 62-0221. The work of KJ is supported in part by NSF TRIPODS 2023166 and CIF 2007036.
}{}

\newpage
\bibliographystyle{icml2022}
\bibliography{bibliography.bib}

\begin{thebibliography}{61}
\providecommand{\natexlab}[1]{#1}
\providecommand{\url}[1]{\texttt{#1}}
\expandafter\ifx\csname urlstyle\endcsname\relax
  \providecommand{\doi}[1]{doi: #1}\else
  \providecommand{\doi}{doi: \begingroup \urlstyle{rm}\Url}\fi

\bibitem[Abbasi-Yadkori et~al.(2011)Abbasi-Yadkori, P{\'a}l, and
  Szepesv{\'a}ri]{abbasi2011improved}
Abbasi-Yadkori, Y., P{\'a}l, D., and Szepesv{\'a}ri, C.
\newblock Improved algorithms for linear stochastic bandits.
\newblock \emph{Advances in neural information processing systems},
  24:\penalty0 2312--2320, 2011.

\bibitem[Agarwal et~al.(2020)Agarwal, Kakade, and Yang]{agarwal2020model}
Agarwal, A., Kakade, S., and Yang, L.~F.
\newblock Model-based reinforcement learning with a generative model is minimax
  optimal.
\newblock In \emph{Conference on Learning Theory}, pp.\  67--83. PMLR, 2020.

\bibitem[{\AA}str{\"o}m \& Wittenmark(2013){\AA}str{\"o}m and
  Wittenmark]{aastrom2013adaptive}
{\AA}str{\"o}m, K.~J. and Wittenmark, B.
\newblock \emph{Adaptive control}.
\newblock Courier Corporation, 2013.

\bibitem[Azar et~al.(2017)Azar, Osband, and Munos]{azar2017minimax}
Azar, M.~G., Osband, I., and Munos, R.
\newblock Minimax regret bounds for reinforcement learning.
\newblock In \emph{International Conference on Machine Learning}, pp.\
  263--272. PMLR, 2017.

\bibitem[Boffi et~al.(2021)Boffi, Tu, and Slotine]{boffi2021regret}
Boffi, N.~M., Tu, S., and Slotine, J.-J.~E.
\newblock Regret bounds for adaptive nonlinear control.
\newblock In \emph{Learning for Dynamics and Control}, pp.\  471--483. PMLR,
  2021.

\bibitem[Bristow et~al.(2006)Bristow, Tharayil, and Alleyne]{bristow2006survey}
Bristow, D.~A., Tharayil, M., and Alleyne, A.~G.
\newblock A survey of iterative learning control.
\newblock \emph{IEEE control systems magazine}, 26\penalty0 (3):\penalty0
  96--114, 2006.

\bibitem[Brunke et~al.(2022)Brunke, Greeff, Hall, Yuan, Zhou, Panerati, and
  Schoellig]{brunke2022safe}
Brunke, L., Greeff, M., Hall, A.~W., Yuan, Z., Zhou, S., Panerati, J., and
  Schoellig, A.~P.
\newblock Safe learning in robotics: From learning-based control to safe
  reinforcement learning.
\newblock \emph{Annual Review of Control, Robotics, and Autonomous Systems},
  5:\penalty0 411--444, 2022.

\bibitem[Chua et~al.(2018)Chua, Calandra, McAllister, and Levine]{chua2018deep}
Chua, K., Calandra, R., McAllister, R., and Levine, S.
\newblock Deep reinforcement learning in a handful of trials using
  probabilistic dynamics models.
\newblock \emph{Advances in neural information processing systems}, 31, 2018.

\bibitem[Cohen et~al.(2019)Cohen, Koren, and Mansour]{cohen2019learning}
Cohen, A., Koren, T., and Mansour, Y.
\newblock Learning linear-quadratic regulators efficiently with only $\sqrt{T}$
  regret.
\newblock \emph{arXiv preprint arXiv:1902.06223}, 2019.

\bibitem[Dean et~al.(2020)Dean, Mania, Matni, Recht, and Tu]{dean2020sample}
Dean, S., Mania, H., Matni, N., Recht, B., and Tu, S.
\newblock On the sample complexity of the linear quadratic regulator.
\newblock \emph{Foundations of Computational Mathematics}, 20\penalty0
  (4):\penalty0 633--679, 2020.

\bibitem[Dieudonn{\'e}(2011)]{dieudonne2011foundations}
Dieudonn{\'e}, J.
\newblock \emph{Foundations of modern analysis}.
\newblock Read Books Ltd, 2011.

\bibitem[Du et~al.(2021)Du, Kakade, Lee, Lovett, Mahajan, Sun, and
  Wang]{du2021bilinear}
Du, S., Kakade, S., Lee, J., Lovett, S., Mahajan, G., Sun, W., and Wang, R.
\newblock Bilinear classes: A structural framework for provable generalization
  in rl.
\newblock In \emph{International Conference on Machine Learning}, pp.\
  2826--2836. PMLR, 2021.

\bibitem[Feldbaum(1960)]{feldbaum1960dual}
Feldbaum, A.~A.
\newblock Dual control theory. i.
\newblock \emph{Avtomatika i Telemekhanika}, 21\penalty0 (9):\penalty0
  1240--1249, 1960.

\bibitem[Foster et~al.(2020)Foster, Sarkar, and Rakhlin]{foster2020learning}
Foster, D., Sarkar, T., and Rakhlin, A.
\newblock Learning nonlinear dynamical systems from a single trajectory.
\newblock In \emph{Learning for Dynamics and Control}, pp.\  851--861. PMLR,
  2020.

\bibitem[Foster et~al.(2021)Foster, Kakade, Qian, and
  Rakhlin]{foster2021statistical}
Foster, D.~J., Kakade, S.~M., Qian, J., and Rakhlin, A.
\newblock The statistical complexity of interactive decision making.
\newblock \emph{arXiv preprint arXiv:2112.13487}, 2021.

\bibitem[Gerencs{\'e}r \& Hjalmarsson(2005)Gerencs{\'e}r and
  Hjalmarsson]{gerencser2005adaptive}
Gerencs{\'e}r, L. and Hjalmarsson, H.
\newblock Adaptive input design in system identification.
\newblock In \emph{Proceedings of the 44th IEEE Conference on Decision and
  Control}, pp.\  4988--4993. IEEE, 2005.

\bibitem[Gerencs{\'e}r et~al.(2007)Gerencs{\'e}r, M{\aa}rtensson, and
  Hjalmarsson]{gerencser2007adaptive}
Gerencs{\'e}r, L., M{\aa}rtensson, J., and Hjalmarsson, H.
\newblock Adaptive input design for arx systems.
\newblock In \emph{2007 European Control Conference (ECC)}, pp.\  5707--5714.
  IEEE, 2007.

\bibitem[Goodwin \& Payne(1977)Goodwin and Payne]{goodwin1977dynamic}
Goodwin, G.~C. and Payne, R.~L.
\newblock \emph{Dynamic system identification: experiment design and data
  analysis}.
\newblock Academic press, 1977.

\bibitem[Hazan et~al.(2019)Hazan, Kakade, Singh, and
  Van~Soest]{hazan2019provably}
Hazan, E., Kakade, S., Singh, K., and Van~Soest, A.
\newblock Provably efficient maximum entropy exploration.
\newblock In \emph{International Conference on Machine Learning}, pp.\
  2681--2691. PMLR, 2019.

\bibitem[Kaiser et~al.(2019)Kaiser, Babaeizadeh, Milos, Osinski, Campbell,
  Czechowski, Erhan, Finn, Kozakowski, Levine, et~al.]{kaiser2019model}
Kaiser, L., Babaeizadeh, M., Milos, P., Osinski, B., Campbell, R.~H.,
  Czechowski, K., Erhan, D., Finn, C., Kozakowski, P., Levine, S., et~al.
\newblock Model-based reinforcement learning for atari.
\newblock \emph{arXiv preprint arXiv:1903.00374}, 2019.

\bibitem[Kakade et~al.(2020)Kakade, Krishnamurthy, Lowrey, Ohnishi, and
  Sun]{kakade2020information}
Kakade, S., Krishnamurthy, A., Lowrey, K., Ohnishi, M., and Sun, W.
\newblock Information theoretic regret bounds for online nonlinear control.
\newblock \emph{Advances in Neural Information Processing Systems},
  33:\penalty0 15312--15325, 2020.

\bibitem[Katselis et~al.(2012)Katselis, Rojas, Hjalmarsson, and
  Bengtsson]{katselis2012application}
Katselis, D., Rojas, C.~R., Hjalmarsson, H., and Bengtsson, M.
\newblock Application-oriented finite sample experiment design: A semidefinite
  relaxation approach.
\newblock \emph{IFAC Proceedings Volumes}, 45\penalty0 (16):\penalty0
  1635--1640, 2012.

\bibitem[Lindqvist \& Hjalmarsson(2001)Lindqvist and
  Hjalmarsson]{lindqvist2001identification}
Lindqvist, K. and Hjalmarsson, H.
\newblock Identification for control: Adaptive input design using convex
  optimization.
\newblock In \emph{Proceedings of the 40th IEEE Conference on Decision and
  Control (Cat. No. 01CH37228)}, volume~5, pp.\  4326--4331. IEEE, 2001.

\bibitem[Ljung(1998)]{ljung1998system}
Ljung, L.
\newblock \emph{System identification}.
\newblock Springer, 1998.

\bibitem[Manchester(2010)]{manchester2010input}
Manchester, I.~R.
\newblock Input design for system identification via convex relaxation.
\newblock In \emph{49th IEEE Conference on Decision and Control (CDC)}, pp.\
  2041--2046. IEEE, 2010.

\bibitem[Mania et~al.(2019)Mania, Tu, and Recht]{mania2019certainty}
Mania, H., Tu, S., and Recht, B.
\newblock Certainty equivalence is efficient for linear quadratic control.
\newblock \emph{Advances in Neural Information Processing Systems}, 32, 2019.

\bibitem[Mania et~al.(2022)Mania, Jordan, and Recht]{mania2020active}
Mania, H., Jordan, M.~I., and Recht, B.
\newblock Active learning for nonlinear system identification with guarantees.
\newblock \emph{J. Mach. Learn. Res.}, 23:\penalty0 32--1, 2022.

\bibitem[Mehra(1974)]{mehra1974optimal}
Mehra, R.
\newblock Optimal input signals for parameter estimation in dynamic
  systems--survey and new results.
\newblock \emph{IEEE Transactions on Automatic Control}, 19\penalty0
  (6):\penalty0 753--768, 1974.

\bibitem[Mesbah(2018)]{mesbah2018stochastic}
Mesbah, A.
\newblock Stochastic model predictive control with active uncertainty learning:
  A survey on dual control.
\newblock \emph{Annual Reviews in Control}, 45:\penalty0 107--117, 2018.

\bibitem[Nakka et~al.(2020)Nakka, Liu, Shi, Anandkumar, Yue, and
  Chung]{nakka2020chance}
Nakka, Y.~K., Liu, A., Shi, G., Anandkumar, A., Yue, Y., and Chung, S.-J.
\newblock Chance-constrained trajectory optimization for safe exploration and
  learning of nonlinear systems.
\newblock \emph{IEEE Robotics and Automation Letters}, 6\penalty0 (2):\penalty0
  389--396, 2020.

\bibitem[Nguyen-Tuong \& Peters(2011)Nguyen-Tuong and Peters]{nguyen2011model}
Nguyen-Tuong, D. and Peters, J.
\newblock Model learning for robot control: a survey.
\newblock \emph{Cognitive processing}, 12:\penalty0 319--340, 2011.

\bibitem[Osband \& Van~Roy(2014)Osband and Van~Roy]{osband2014model}
Osband, I. and Van~Roy, B.
\newblock Model-based reinforcement learning and the eluder dimension.
\newblock \emph{Advances in Neural Information Processing Systems}, 27, 2014.

\bibitem[Oymak(2019)]{oymak2019stochastic}
Oymak, S.
\newblock Stochastic gradient descent learns state equations with nonlinear
  activations.
\newblock In \emph{conference on Learning Theory}, pp.\  2551--2579. PMLR,
  2019.

\bibitem[O’Connell et~al.(2022)O’Connell, Shi, Shi, Azizzadenesheli,
  Anandkumar, Yue, and Chung]{o2022neural}
O’Connell, M., Shi, G., Shi, X., Azizzadenesheli, K., Anandkumar, A., Yue,
  Y., and Chung, S.-J.
\newblock Neural-fly enables rapid learning for agile flight in strong winds.
\newblock \emph{Science Robotics}, 7\penalty0 (66):\penalty0 eabm6597, 2022.

\bibitem[Rahimi \& Recht(2008)Rahimi and Recht]{rahimi2008uniform}
Rahimi, A. and Recht, B.
\newblock Uniform approximation of functions with random bases.
\newblock In \emph{2008 46th annual allerton conference on communication,
  control, and computing}, pp.\  555--561. IEEE, 2008.

\bibitem[Richards et~al.(2021)Richards, Azizan, Slotine, and
  Pavone]{richards2021adaptive}
Richards, S.~M., Azizan, N., Slotine, J.-J., and Pavone, M.
\newblock Adaptive-control-oriented meta-learning for nonlinear systems.
\newblock \emph{arXiv preprint arXiv:2103.04490}, 2021.

\bibitem[Rojas et~al.(2007)Rojas, Welsh, Goodwin, and Feuer]{rojas2007robust}
Rojas, C.~R., Welsh, J.~S., Goodwin, G.~C., and Feuer, A.
\newblock Robust optimal experiment design for system identification.
\newblock \emph{Automatica}, 43\penalty0 (6):\penalty0 993--1008, 2007.

\bibitem[Sattar \& Oymak(2022)Sattar and Oymak]{sattar2022non}
Sattar, Y. and Oymak, S.
\newblock Non-asymptotic and accurate learning of nonlinear dynamical systems.
\newblock \emph{The Journal of Machine Learning Research}, 23\penalty0
  (1):\penalty0 6248--6296, 2022.

\bibitem[Schulman et~al.(2015)Schulman, Levine, Abbeel, Jordan, and
  Moritz]{schulman2015trust}
Schulman, J., Levine, S., Abbeel, P., Jordan, M., and Moritz, P.
\newblock Trust region policy optimization.
\newblock In \emph{International conference on machine learning}, pp.\
  1889--1897. PMLR, 2015.

\bibitem[Schulman et~al.(2017)Schulman, Wolski, Dhariwal, Radford, and
  Klimov]{schulman2017proximal}
Schulman, J., Wolski, F., Dhariwal, P., Radford, A., and Klimov, O.
\newblock Proximal policy optimization algorithms.
\newblock \emph{arXiv preprint arXiv:1707.06347}, 2017.

\bibitem[Shi et~al.(2019)Shi, Shi, O’Connell, Yu, Azizzadenesheli,
  Anandkumar, Yue, and Chung]{shi2019neural}
Shi, G., Shi, X., O’Connell, M., Yu, R., Azizzadenesheli, K., Anandkumar, A.,
  Yue, Y., and Chung, S.-J.
\newblock Neural lander: Stable drone landing control using learned dynamics.
\newblock In \emph{2019 International Conference on Robotics and Automation
  (ICRA)}, pp.\  9784--9790. IEEE, 2019.

\bibitem[Shi et~al.(2021{\natexlab{a}})Shi, Azizzadenesheli, O'Connell, Chung,
  and Yue]{shi2021meta}
Shi, G., Azizzadenesheli, K., O'Connell, M., Chung, S.-J., and Yue, Y.
\newblock Meta-adaptive nonlinear control: Theory and algorithms.
\newblock \emph{Advances in Neural Information Processing Systems},
  34:\penalty0 10013--10025, 2021{\natexlab{a}}.

\bibitem[Shi et~al.(2021{\natexlab{b}})Shi, H{\"o}nig, Shi, Yue, and
  Chung]{shi2021neural}
Shi, G., H{\"o}nig, W., Shi, X., Yue, Y., and Chung, S.-J.
\newblock Neural-swarm2: Planning and control of heterogeneous multirotor
  swarms using learned interactions.
\newblock \emph{IEEE Transactions on Robotics}, 38\penalty0 (2):\penalty0
  1063--1079, 2021{\natexlab{b}}.

\bibitem[Simchowitz \& Foster(2020)Simchowitz and Foster]{simchowitz2020naive}
Simchowitz, M. and Foster, D.
\newblock Naive exploration is optimal for online lqr.
\newblock In \emph{International Conference on Machine Learning}, pp.\
  8937--8948. PMLR, 2020.

\bibitem[Simchowitz et~al.(2018)Simchowitz, Mania, Tu, Jordan, and
  Recht]{simchowitz2018learning}
Simchowitz, M., Mania, H., Tu, S., Jordan, M.~I., and Recht, B.
\newblock Learning without mixing: Towards a sharp analysis of linear system
  identification.
\newblock \emph{arXiv preprint arXiv:1802.08334}, 2018.

\bibitem[Simchowitz et~al.(2019)Simchowitz, Boczar, and
  Recht]{simchowitz2019learning}
Simchowitz, M., Boczar, R., and Recht, B.
\newblock Learning linear dynamical systems with semi-parametric least squares.
\newblock \emph{arXiv preprint arXiv:1902.00768}, 2019.

\bibitem[Simchowitz et~al.(2020)Simchowitz, Singh, and
  Hazan]{simchowitz2020improper}
Simchowitz, M., Singh, K., and Hazan, E.
\newblock Improper learning for non-stochastic control.
\newblock In \emph{Conference on Learning Theory}, pp.\  3320--3436. PMLR,
  2020.

\bibitem[Slotine et~al.(1991)Slotine, Li, et~al.]{slotine1991applied}
Slotine, J.-J.~E., Li, W., et~al.
\newblock \emph{Applied nonlinear control}, volume 199.
\newblock Prentice hall Englewood Cliffs, NJ, 1991.

\bibitem[Song \& Sun(2021)Song and Sun]{song2021pc}
Song, Y. and Sun, W.
\newblock Pc-mlp: Model-based reinforcement learning with policy cover guided
  exploration.
\newblock In \emph{International Conference on Machine Learning}, pp.\
  9801--9811. PMLR, 2021.

\bibitem[Sun et~al.(2019)Sun, Jiang, Krishnamurthy, Agarwal, and
  Langford]{sun2019model}
Sun, W., Jiang, N., Krishnamurthy, A., Agarwal, A., and Langford, J.
\newblock Model-based rl in contextual decision processes: Pac bounds and
  exponential improvements over model-free approaches.
\newblock In \emph{Conference on learning theory}, pp.\  2898--2933. PMLR,
  2019.

\bibitem[Sutton \& Barto(2018)Sutton and Barto]{sutton2018reinforcement}
Sutton, R.~S. and Barto, A.~G.
\newblock \emph{Reinforcement learning: An introduction}.
\newblock MIT press, 2018.

\bibitem[Wagenmaker \& Jamieson(2020)Wagenmaker and
  Jamieson]{wagenmaker2020active}
Wagenmaker, A. and Jamieson, K.
\newblock Active learning for identification of linear dynamical systems.
\newblock In \emph{Conference on Learning Theory}, pp.\  3487--3582. PMLR,
  2020.

\bibitem[Wagenmaker \& Jamieson(2022)Wagenmaker and
  Jamieson]{wagenmaker2022instance}
Wagenmaker, A. and Jamieson, K.
\newblock Instance-dependent near-optimal policy identification in linear mdps
  via online experiment design.
\newblock \emph{arXiv preprint arXiv:2207.02575}, 2022.

\bibitem[Wagenmaker \& Pacchiano(2022)Wagenmaker and
  Pacchiano]{wagenmaker2022leveraging}
Wagenmaker, A. and Pacchiano, A.
\newblock Leveraging offline data in online reinforcement learning.
\newblock \emph{arXiv preprint arXiv:2211.04974}, 2022.

\bibitem[Wagenmaker et~al.(2021)Wagenmaker, Simchowitz, and
  Jamieson]{wagenmaker2021task}
Wagenmaker, A.~J., Simchowitz, M., and Jamieson, K.
\newblock Task-optimal exploration in linear dynamical systems.
\newblock In \emph{International Conference on Machine Learning}, pp.\
  10641--10652. PMLR, 2021.

\bibitem[Williams et~al.(2017)Williams, Wagener, Goldfain, Drews, Rehg, Boots,
  and Theodorou]{williams2017information}
Williams, G., Wagener, N., Goldfain, B., Drews, P., Rehg, J.~M., Boots, B., and
  Theodorou, E.~A.
\newblock Information theoretic mpc for model-based reinforcement learning.
\newblock In \emph{2017 IEEE International Conference on Robotics and
  Automation (ICRA)}, pp.\  1714--1721. IEEE, 2017.

\bibitem[Yu et~al.(2020{\natexlab{a}})Yu, Shi, Chung, Yue, and
  Wierman]{yu2020power}
Yu, C., Shi, G., Chung, S.-J., Yue, Y., and Wierman, A.
\newblock The power of predictions in online control.
\newblock \emph{Advances in Neural Information Processing Systems},
  33:\penalty0 1994--2004, 2020{\natexlab{a}}.

\bibitem[Yu et~al.(2020{\natexlab{b}})Yu, Thomas, Yu, Ermon, Zou, Levine, Finn,
  and Ma]{yu2020mopo}
Yu, T., Thomas, G., Yu, L., Ermon, S., Zou, J.~Y., Levine, S., Finn, C., and
  Ma, T.
\newblock Mopo: Model-based offline policy optimization.
\newblock \emph{Advances in Neural Information Processing Systems},
  33:\penalty0 14129--14142, 2020{\natexlab{b}}.

\bibitem[Zahavy et~al.(2021)Zahavy, O'Donoghue, Desjardins, and
  Singh]{zahavy2021reward}
Zahavy, T., O'Donoghue, B., Desjardins, G., and Singh, S.
\newblock Reward is enough for convex mdps.
\newblock \emph{Advances in Neural Information Processing Systems},
  34:\penalty0 25746--25759, 2021.

\bibitem[Zanette \& Brunskill(2019)Zanette and Brunskill]{zanette2019tighter}
Zanette, A. and Brunskill, E.
\newblock Tighter problem-dependent regret bounds in reinforcement learning
  without domain knowledge using value function bounds.
\newblock In \emph{International Conference on Machine Learning}, pp.\
  7304--7312. PMLR, 2019.

\bibitem[Zhou et~al.(2021)Zhou, Gu, and Szepesvari]{zhou2021nearly}
Zhou, D., Gu, Q., and Szepesvari, C.
\newblock Nearly minimax optimal reinforcement learning for linear mixture
  markov decision processes.
\newblock In \emph{Conference on Learning Theory}, pp.\  4532--4576. PMLR,
  2021.

\end{thebibliography}

\newpage
\appendix

\section{Technical Tools}
\iftoggle{arxiv}{}{
\paragraph{Additional Notation.}
We let $\cS^{d-1}$ refer to the unit ball in $d$ dimensions and $\bbS_+^d$ (resp. $\bbS_{++}^d$) the set of positive semi-definite matrices (resp. positive definite matrices) in $\R^{d \times d}$. Throughout, $\cO(\cdot)$ denotes standard big-O notation, $\cOtil(\cdot)$ hides additional logarithmic factors. $\poly(\cdot)$ denotes some term that is polynomial in its arguments, with exponents absolute constants. 
$\simplex_\cX$ denotes the set of distributions over set $\cX$.
}

\begin{lemma}\label{lem:gaussian_bound}
Let $\bw_i \sim \cN(0,I_{\dimx})$ for all $i \in \{ 1,2, \ldots, n \}$. Then,
\begin{align*}
\Exp \left [ \left (\sum_{i=1}^{n} \| \bw_i \|_2 \right )^c \right ] \le n^2 \cdot \poly(\dimx). 
\end{align*}
for $c$ an absolute constant.
\end{lemma}
\begin{proof}
We first bound 
\begin{align*}
(\sum_{i=1}^{n} \| \bw_i \|_2 )^c \le n \cdot \max_{i} \| \bw_i \|_2^c \le n \cdot \sum_{i} \| \bw_i \|_2^c.
\end{align*}
The result then follows since we can bound the $\Exp[\| \bw_i \|_2^c] \le \poly(d)$ for $\bw_i \sim \cN(0,I)$ and $c$ an absolute constant.
\end{proof}

\begin{lemma}[Lemma I.4 of \cite{wagenmaker2021task}]\label{lem:gaminv_diff}
Assume $A,B \in \bbS^d_{++}$, $\| A - B \|_\op \le \epsilon$, and $\epsilon < \lammin(B)$. Then
\begin{align*}
\| A^{-1} - B^{-1} \|_\op \le \frac{\epsilon}{\lammin(B)(\lammin(B) - \epsilon)}.
\end{align*}
\end{lemma}

\subsection{Martingale Regression in General Norms}\label{sec:mdm}

For the following two results, we consider the martingale regression setting of \cite{wagenmaker2021task} (referred to as the \textsf{MDM} setting). In particular, we consider observations of the form
\begin{align}\label{eq:mdm}
y_{\sft} = \inner{\bmust}{\bz_{\sft}} + w_{\sft}, 
\end{align}
for $y_{\sft} \in \R$, unknown parameter $\bmust \in \R^{\dimmu}$, $w_{\sft} \mid \cF_{\sft-1} \sim \cN(0,\sigw^2)$, and $\bz_{\sft}$ $\cF_{\sft-1}$-measurable, for a filtration $(\cF_{\sft})_{\sft \ge 1}$. This setting therefore encompasses general stochastic processes where the observations are linear---the evolution of $\bz_{\sft}$ could be arbitrary. 

We consider the setting where we interact with \eqref{eq:mdm} for $\sfT$ steps, collecting observations $\{ (\bz_{\sft},y_{\sft}) \}_{\sft = 1}^{\sfT}$, and then form the least-squares estimate of $\bmust$:
\begin{align*}
\bmuhat = \left ( \sum_{\sft=1}^{\sfT} \bz_{\sft} \bz_{\sft}^\top \right )^{-1} \sum_{\sft=1}^{\sfT} \bz_{\sft} y_{\sft}.
\end{align*}
We also denote $\bSigma_{\sfT} := \sum_{\sft=1}^{\sfT} \bz_{\sft} \bz_{\sft}^\top$. The following results characterize the estimation error of $\bmuhat$ in the $M$-norm and 2-norm. 

\begin{proposition}[Theorem 7.2 of \cite{wagenmaker2021task}]\label{prop:M_reg_bound}
Fix any matrices $\bGamma \in \bbS_{++}^{\dimmu}, M \in \bbS_+^{\dimmu}$, with $M \neq 0$. Given a parameter $\beta \in (0,1/4)$, define the event
\begin{align*}
\cE := \{ \| \bSigma_{\sfT} - \bGamma \|_\op \le \beta \lammin(\bGamma) \}.
\end{align*}
Then, if $\cE$ holds, the following holds with probability at least $1-\delta$:
\begin{align*}
\| \bmuhat - \bmust \|_M^2 \le 5 (1 + \zeta) \cdot \sigw^2 \log \frac{6 \dimmu}{\delta} \cdot \tr(M \bGamma^{-1})
\end{align*}
where $\zeta = 26 \beta^2 \lammax(\bGamma) \tr(\bGamma^{-1})$.
\end{proposition}

\begin{proposition}[Lemma E.1 of \cite{wagenmaker2021task}]\label{prop:fro_reg_bound}
On the event
\begin{align*}
\cEop := \{ \lammin(\bSigma_{\sfT}) \ge \lamun \sfT, \bSigma_{\sfT} \preceq \sfT \bGambar_{\sfT} \},
\end{align*}
then we have that with probability at least $1-\delta$:
\begin{align*}
\| \bmuhat - \bmust \|_2 \le C \cdot \sigw \sqrt{\frac{\log 1/\delta + \dimmu + \log \det(\bGambar_{\sfT} / \lamun + I)}{\lamun \sfT}}.
\end{align*}
\end{proposition}

\subsubsection{Connection Between \eqref{eq:system} and \eqref{eq:mdm}}\label{sec:knr_mdm}
We will apply the results \Cref{prop:M_reg_bound} and \Cref{prop:fro_reg_bound} in the setting of \eqref{eq:system} in order to obtain estimation bounds on $\Ast$. As the setting of \eqref{eq:system} has vector observations, we briefly describe here how it can be mapped into the setting described above. 

Recall that \eqref{eq:system} evolves as
\begin{align*}
\bx_{h+1} = \Ast \bphi(\bx_h,\bu_h) + \bw_h, \quad h = 1,\ldots, H,
\end{align*}
for $\bx_h \in \R^{\dimx}$, $\bphi(\bx,\bu) \in \R^{\dimphi}$, and $\Ast \in \R^{\dimx \times \dimphi}$. We assume that $\bx_1$ is some fixed starting state. Assume that we have run for $T$ episodes, and collected observations $\{ (\bx_1^t,\bu_1^t,\bx_2^t,\ldots,\bx_h^t,\bu_h^t,\bx_{h+1}^t) \}_{t = 1}^T$. Now let $\bmust := \vec(\Ast)$. Furthermore, for any $t,h$, and $i \in [\dimx]$, let $\sft = (t,h,i)$ and $\bz_{\sft} = [\bm{0}_{\dimphi (i-1)}, \bphi(\bx_h^t,\bu_h^t), \bm{0}_{\dimphi ( \dimx - i )}] \in \R^{\dimx \dimphi}$ where $\bm{0}_d$ denotes the zero vector of length $d$. Then we see that
\begin{align*}
[\bx_{h+1}^t]_{i} = \inner{\bmust}{\bz_{\sft}} + [\bw_h^t]_i.
\end{align*}
Setting $y_{\sft} = [\bx_{h+1}^t]_{i}$ and $w_{\sft} = [\bw_h^t]_i$, it is clear that this follows the observation model of \eqref{eq:mdm} with $\dimmu = \dimphi \dimx$ and $\sfT = \dimx T H$. It is also straightforward to see that the measurability assumptions of the setting of \eqref{eq:mdm} are satisfied by this.

\section{Proof of Main Result}

\begin{algorithm}[h]
\begin{algorithmic}[1]
\State \textbf{inputs:} number of episodes to run $T$, cost function $(\cost_h)_{h=1}^H$, confidence $\delta$, control policies $\Picon$, exploration policies $\Piexp$
\State $\Ahat^1 \leftarrow $ anything, $\ellf \leftarrow \lceil \log_2 T/8 \rceil$
\For{$\ell = 1,2,3, \ldots, \ellf$}
	\State $T_\ell \leftarrow 2^\ell, \delta_\ell \leftarrow \delta / 8 \ell^2$
	\State Compute estimate of cost matrix: $\cH_\ell \leftarrow \cH(\Ahat^\ell)$
	\State $\Pi_\ell \leftarrow \learnpolicies(\cH_\ell,T_\ell,\delta_\ell,\regalg,\Piexp)$ (\Cref{alg:learn_policies}), with $\regalg$ the \LC algorithm \citep{kakade2020information}
	\State Rerun each policy in $\Pi_\ell$ $N_\ell = \lceil T_\ell / |\Pi_\ell| \rceil$ times, denote collected data $\frakD_\ell$
	\State Estimate system parameters
	\State \begin{align*}
	\Ahat^{\ell+1} = \argmin_A \sum_{h=1}^H \sum_{(\bx_{h+1},\bu_h,\bx_h) \in \frakD_\ell} \| \bx_{h+1} - A \bphi(\bx_{h},\bu_h) \|_2^2 
	\end{align*}
\EndFor
\State \textbf{return} $\pihat_T \leftarrow \pist(\Ahat^{\ellf+1})$
\end{algorithmic}
\caption{Optimal Exploration in Nonlinear Systems (Full Version of \Cref{alg:main})}
\label{alg:main_app}
\end{algorithm}

\begin{theorem}[Full Version of \Cref{thm:main}]\label{thm:main_app}
Assume \Cref{asm:bounded_features,asm:full_rank_cov,asm:bounded_cost,asm:smooth_phi,asm:smooth_controller,asm:smooth_thetast} hold. Then if
\begin{align}\label{eq:main_thm_app_T_burnin}
T \ge \poly \left ( \dimphi, \dimx, H, \BA, \Bphi, \Lphi, \Lzeta, \Lcost, \LKst, \sigw, \sigw^{-1}, \tfrac{1}{\lamminst}, \log \tfrac{T}{\delta} \right ) \cdot \max \left \{ 1, \tfrac{1}{\rcost(\Ast)^2}, \tfrac{1}{\rtheta(\Ast)^2} \right \},
\end{align}
with probability at least $1-\delta$, \Cref{alg:main_app} plays exploration policies $\piexp \in \Piexp$ at every episode, runs for at most $T$ episodes, and the controller $\Khat_T$ returned \Cref{alg:main_app} satisfies, with probability at least $1-\delta$:
\begin{align*}
\cJ(\Khat_T;\Ast) - \cJ(\Kst(\Ast);\Ast) \le \frac{\sigw^2}{T} \cdot \min_{\bLambda \in \bOmega} \tr \left ( \cH(\Ast) \bLamchk^{-1} \right )  \cdot C \log \frac{6 \dimx \dimphi}{\delta} + \frac{\Clot}{T^{3/2}}
\end{align*}
for $C$ a universal constant and
\begin{align*}
\Cpoly :=  \poly \left ( \dimphi, \dimx, H, \BA, \Bphi, \Lphi, \Lzeta, \Lcost, \LKst, \sigw, \sigw^{-1}, \tfrac{1}{\lamminst}, \log \tfrac{T}{\delta} \right ).
\end{align*}
\end{theorem}
\begin{proof}
Let $\cE_\ell$ denote that the good event of \Cref{lem:knr_est_bound} holds at round $\ell$ which, by \Cref{lem:knr_est_bound}, occurs with probability at least $1-6\delta_\ell$. By our setting of $\delta_\ell = \delta/12 \ell^2$, we have that the total failure probability of $\cE_\ell$ for all $\ell$ is bounded as
\begin{align*}
\sum_{\ell=1}^{\infty} 6 \cdot \frac{\delta}{12 \ell^2} \le \delta.
\end{align*}
Henceforth we assume that $\cE := \cap_\ell \cE_\ell$ holds. Let $\Ahat := \Ahat^{\ellf+1}$, and $\Ahatm := \Ahat^{\ellf}$.

Before proceeding to the main proof, we note that the conclusion that \Cref{alg:main_app} only explores with policies in $\Piexp$ follows from the definition of \learnpolicies and \LC. Note that \learnpolicies only interacts with \eqref{eq:system} through calls to \expdesign, which itself only interacts with \eqref{eq:system} through calls to $\regalg$, instantiated in \Cref{alg:main_app} by \LC. Inspection of the \LC algorithm in \cite{kakade2020information} reveals that \LC only interacts with \eqref{eq:system} by playing policies in $\Piexp$, from which the conclusion follows.

\paragraph{Bounding the Number of Episodes.}
Denote $\Toed_\ell = |\Pi_\ell|$.
Note that by construction we always have $N_\ell \Toed_\ell \ge T_\ell$.
By \Cref{lem:T_bounds}, as long as \eqref{eq:main_thm_app_T_burnin} is met, we can bound the total number of episodes collected up to and including round $\ell$ by $4T_{\ell}$ for $\ell \in \{\ellf, \ellf - 1 \}$.
We therefore, in the following, will make use of the fact that
$$ c \cdot T \le N_{\ell} K_{\ell}  \le c' \cdot T$$
for $\ell \in \{\ellf, \ellf - 1\}$ and absolute constants $c,c'$. Furthermore, it also follows from this that the total number of episodes run by \Cref{alg:main_app} is bounded by 
\begin{align*}
4 T_{\ellf} = 4 \cdot 2^{\lceil \log T/8 \rceil} \le 8 \cdot \frac{T}{8} = T,
\end{align*}
so \Cref{alg:main_app} runs for at most $T$ episodes.

\paragraph{Approximating the Controller Loss.}
Let $\rest(\Ast) := \min \{ 1, \rcost(\Ast), \rtheta(\Ast) \}$, for $\rcost(\Ast)$ as in \Cref{asm:bounded_cost} and $\rtheta(\Ast)$ as in \Cref{asm:smooth_thetast}. 
By \Cref{lem:quadratic_loss_approx}, under \Cref{asm:bounded_features,asm:smooth_phi,asm:smooth_controller,asm:bounded_cost,asm:smooth_thetast}, as long as $\Ahat \in \cB_{\fro}(\Ast;\rest(\Ast))$, we have
\begin{align*}
& \cJ(\pihat_T; \Ast) - \cJ(\Kst(\Ast); \Ast)  \le \| \vec(\Ahat - \Ast) \|_{\cH(\Ast)}^2  \\
& \qquad + \poly(\LKst, \| \Ast \|_\op, \Lphi, \Lzeta, \Lcost, \sigw^{-1}, H, \dimx)\cdot \| \Ahat - \Ast \|_\op^3.
\end{align*}
Furthermore, we can bound
\begin{align*}
\| \vec(\Ahat - \Ast) \|_{\cH(\Ast)}^2  & = \| \vec(\Ahat - \Ast) \|_{\cH(\Ahatm)}^2  + \vec(\Ahat - \Ast)^\top (\cH(\Ast) - \cH(\Ahatm)) \vec(\Ahat - \Ast) \\
& \le (\Ahat - \Ast)^\top \cH(\Ahatm) (\Ahat - \Ast) + \| \Ahatm - \Ast \|_\op^2 \| \cH(\Ast) - \cH(\Ahatm) \|_\op.
\end{align*}

\paragraph{Bounding the Hessian Estimation Error.}
On $\cE$, by \Cref{lem:knr_est_bound} and as long as \eqref{eq:main_thm_app_T_burnin} is met, we have (note that at the final epoch, the plug-in estimator $\cH(\Ahatm)$ is given as input to \expdesign):
\begin{align*}
\| \vec(\Ahat - \Ast) \|_{\cH(\Ahatm)}^2  & \le \frac{60}{N_{\ellf} \Toed_{\ellf}} \cdot  \min_{\bLambda \in \bOmega} \tr \left ( \cH(\Ahatm) \bLamchk^{-1} \right )  \cdot \sigw^2 \log \frac{6 \dimx \dimphi}{\delta} \\
 & \qquad + \poly \left ( \dimphi, \dimx, \frac{1}{\lamminst}, \BA, \Bphi, \log \frac{1}{\sigw}, H,\| \cH(\Ahatm) \|_\op, \log \frac{T_{\ellf}}{\delta} \right ) \cdot \frac{1}{N_{\ellf}^2} \\
 & \le \frac{C}{T} \cdot  \min_{\bLambda \in \bOmega} \tr \left ( \cH(\Ahatm) \bLamchk^{-1} \right )  \cdot \sigw^2 \log \frac{6 \dimx \dimphi}{\delta} \\
 & \qquad + \poly \left ( \dimphi, \dimx, \frac{1}{\lamminst}, \BA, \Bphi, \log \frac{1}{\sigw},  H, \| \cH(\Ahatm) \|_\op, \log \frac{T}{\delta} \right ) \cdot \frac{1}{T^2}
\end{align*}
where the last line uses that $N_{\ellf} \Toed_{\ellf}$ is within a constant of $T$, and that $\Toed_{\ellf}$ can be bounded by
\begin{align*}
\poly \left ( \dimphi, \dimx, \frac{1}{\lamminst}, \BA, \Bphi, \log \frac{1}{\sigw}, H, \log \frac{T}{\delta} \right ) 
\end{align*}
by \Cref{lem:knr_est_bound}.
By \Cref{lem:ce_exp_design} we can bound
\begin{align*}
\min_{\bLambda \in \bOmega} \tr \left ( \cH(\Ahatm) \bLamchk^{-1} \right ) \le \min_{\bLambda \in \bOmega}2 \tr \left ( \cH(\Ast) \bLamchk^{-1} \right ) + \frac{2\dimx \dimphi}{\lamminst} \cdot \| \cH(\Ast) - \cH(\Ahatm) \|_\op
\end{align*}
and by \Cref{lem:Hst_to_Hhat}, under \Cref{asm:bounded_features,asm:smooth_phi,asm:smooth_controller,asm:bounded_cost,asm:smooth_thetast}, and as long as $\Ahatm \in \cB_{\fro}(\Ast;\rest(\Ast))$, we can bound 
\begin{align*}
\| \cH(\Ast) - \cH(\Ahatm) \|_\op \le  \poly(\LKst, \| \Ast \|_\op, \Lphi, \Lzeta, \Lcost, \sigw^{-1}, H, \dimx)\cdot \| \Ahatm - \Ast \|_\op.
\end{align*}
Let
\begin{align*}
\Clot :=  \poly \left (\LKst, \BA, \Bphi, \Lphi, \Lzeta, \Lcost, \dimx,  \dimphi, \frac{1}{\lamminst}, \sigw, \sigw^{-1}, H, \| \cH(\Ast) \|_\op, \log \frac{T}{\delta} \right )
\end{align*}
denote some lower-order constant, whose precise polynomial dependence may change from line to line. 
On $\cE$, by \Cref{lem:knr_est_bound} we can bound 
\begin{align}\label{eq:main_op_norm_bound}
\| \Ahat - \Ast \|_\op, \| \Ahatm - \Ast \|_\op \le \Clot \cdot \frac{1}{\sqrt{T}},
\end{align}
and so assuming the burn-in \eqref{eq:main_thm_app_T_burnin} is met, we can bound $\| \Ahatm - \Ast \|_\op \le \frac{1}{2} \rest(\Ast)$ and $\| \Ahat - \Ast \|_\op \le \frac{1}{2} \rest(\Ast)$. This then implies that
\begin{align*}
\| \cH(\Ast) - \cH(\Ahatm) \|_\op \le  \Clot \cdot \frac{1}{\sqrt{T}},
\end{align*}
so in particular we can bound
\begin{align*}
& \poly \left ( \dimphi, \dimx, \frac{1}{\lamminst}, \BA, \Bphi, H, \log \frac{1}{\sigw}, \| \cH(\Ahatm) \|_\op, \log \frac{T}{\delta} \right ) \\
&  \le  \poly \left ( \dimphi, \dimx, \frac{1}{\lamminst}, \BA, \Bphi, H, \log \frac{1}{\sigw}, \| \cH(\Ast) \|_\op, \log \frac{T}{\delta} \right ).
\end{align*}

We have therefore shown that
\begin{align*}
\cJ(\pihat_T; \Ast) - \cJ(\Kst(\Ast); \Ast) & \le \frac{C}{T} \cdot  \min_{\bLambda \in \bOmega} \tr \left ( \cH(\Ast) \bLamchk^{-1} \right )  \cdot \sigw^2 \log \frac{6 \dimx \dimphi}{\delta} \\
& \qquad + \Clot \cdot \left ( \| \Ahat - \Ast \|_\op^3 + \| \Ahatm - \Ast \|_\op^3 + \frac{1}{T^2} + \frac{1}{T} \cdot \| \Ahatm - \Ast \|_\op \right ) \\
& \le \frac{C}{T} \cdot  \min_{\bLambda \in \bOmega} \tr \left ( \cH(\Ast) \bLamchk^{-1} \right )  \cdot \sigw^2 \log \frac{6 \dimx \dimphi}{\delta} + \frac{\Clot}{T^{3/2}},
\end{align*}
The final result follows from using \Cref{lem:hessian_norm_bound} to bound
\begin{align*}
\| \cH(\Ast) \|_\op \le \poly(\| \Ast \|_\op, \Bphi, \Lphi, \Lzeta, \Lcost, \LKst, \sigw^{-1}, H, \dimx).
\end{align*}
\end{proof}

\begin{proof}[Proof of \Cref{thm:main}]
The proof of \Cref{thm:main} is identical to that of \Cref{thm:main_app}, the only difference being that we replace \Cref{asm:smooth_thetast} with \Cref{asm:unique_min}. However, by \Cref{prop:theta_global_min}, the conditions of \Cref{asm:smooth_thetast} are met when \Cref{asm:unique_min} holds.
\end{proof}

\subsection{Supporting Lemmas}

\begin{lemma}\label{lem:knr_satisfies_ed_asm}
Under \Cref{asm:full_rank_cov,asm:bounded_features}, the system \eqref{eq:system} satisfies \Cref{asm:mat_psi,asm:full_rank_cov_app} with
\begin{align*}
\bpsi(\traj) \leftarrow I_{\dimx} \otimes \sum_{h=1}^H \bphi(\bx_h^{\traj},\bu_h^{\traj}) \bphi(\bx_h^{\traj},\bu_h^{\traj})^\top, \quad D =\dimx H \Bphi^2,
\end{align*}
$\dimpsi \leftarrow \dimphi \dimx$, and where $\bx_h^{\traj}$ (resp. $\bu_h^{\traj}$) denotes the state (resp. input) at step $h$ of trajectory $\traj$. Furthermore, it satisfies \Cref{asm:regret_min_app} with $\regalg$ instantiated with the \LC algorithm of \cite{kakade2020information} and
\begin{align*}
\CR = C \cdot H \sqrt{\dimphi(\dimphi + \dimx + \BA)} \cdot \log(1 + \Bphi H/\sigw), \quad \pR = 3/2, \quad \alpha = 1/2
\end{align*}
for a universal constant $C$. 
\end{lemma}
\begin{proof}
That \Cref{asm:full_rank_cov_app} is satisfied is immediate under \Cref{asm:full_rank_cov}. It is clear that $\bpsi(\traj) \in \cS_+^{\dimphi \dimx}$. To obtain a bound on $D$, we only need to bound the trace of $\bpsi(\traj)$:
\begin{align*}
\tr(\bpsi(\traj)) & = \sum_{h=1}^H \tr(I_{\dimx}) \cdot \tr(\bphi(\bx_h^{\traj},\bu_h^{\traj})) \bphi(\bx_h^{\traj},\bu_h^{\traj}))^\top) \\
& = \dimx \sum_{h=1}^H \| \bphi(\bx_h^{\traj},\bu_h^{\traj})) \|_2^2 \\
& \le \dimx H \Bphi^2
\end{align*}
where the inequality holds under \Cref{asm:bounded_features}.

To show that \Cref{asm:regret_min_app} is satisfied in this setting, we have by \Cref{thm:knr_regret}
that with probability at least $1-\delta$, \LC has regret bounded as (using that $\cmax \le 1$ in the setting of \Cref{asm:regret_min_app}):
\begin{align*}
\cR_T \le C \cdot H \sqrt{\dimphi \cdot (\dimphi + \dimx + \BA + \log \frac{1}{\delta} ) \cdot T } \cdot \log \left ( 1 + \Bphi HT /\sigw \right )
\end{align*}
for $C$ a universal constant. We can therefore take $\alpha = 1/2$, $\pR = 3/2$, and 
\begin{align*}
\CR = C' \cdot H \sqrt{\dimphi(\dimphi + \dimx + \BA)} \cdot \log(1 + \Bphi H/\sigw).
\end{align*}
\end{proof}

\begin{lemma}\label{lem:T_bounds}
Let $\Tbar_\ell$ denote the total number of episodes collected by \Cref{alg:main_app} at round $\ell$. For
\begin{align}\label{eq:T_bound_burnin}
T_\ell \ge  \poly \left ( \dimphi, \dimx, \frac{1}{\lamminst}, \BA, \Bphi, \log \frac{1}{\sigw}, H, \log \frac{T_\ell}{\delta} \right ), 
\end{align}
on the success event of \Cref{lem:knr_est_bound}, we have $2 T_\ell \ge \Tbar_\ell$ and
\begin{align*}
2T_\ell \ge \sum_{i=1}^{\ell-1} \Tbar_i .
\end{align*}
\end{lemma}
\begin{proof}
Recall that $T_\ell = 2^\ell$ and
$N_\ell = \lceil T_\ell/\Toed_\ell \rceil$ for $\Toed_\ell = |\Pi_\ell|$. 
By \Cref{lem:knr_est_bound}, $\Tbar_\ell$ can be bounded as
\begin{align*}
\Tbar_\ell \le N_\ell \Toed_\ell + (16 + 2 \log \Toed_\ell) \Toed_\ell
\end{align*}
and $\Toed_\ell$ can be bounded as
\begin{align}\label{eq:T_Kout_bound}
\Toed_\ell \le \poly \left ( \dimphi, \dimx, \frac{1}{\lamminst}, \BA, \Bphi, \log \frac{1}{\sigw}, H, \log \frac{T_\ell}{\delta} \right ) .
\end{align}

Note that we can bound
\begin{align*}
\Tbar_i & \le N_i \Toed_i + (16 + 2 \log \Toed_i) \Toed_i \\
& \le T_i + (17 + 2 \log \Toed_i) \Toed_i \\
& \le T_i + \poly \left ( \dimphi, \dimx, \frac{1}{\lamminst}, \BA, \Bphi, \log \frac{1}{\sigw}, H, \log \frac{T_\ell}{\delta} \right ) .
\end{align*}
From this it is immediately obvious that $2 T_\ell \ge \Tbar_\ell$ as long as \eqref{eq:T_bound_burnin} is satisfied.

To show the second conclusion note that, by our choice of $T_i = 2^i$, we have that $T_\ell \ge \sum_{i=1}^{\ell - 1} T_i$, so it therefore remains to show that
\begin{align*}
T_\ell \ge \sum_{i=1}^{\ell-1}  \poly \left ( \dimphi, \dimx, \frac{1}{\lamminst}, \BA, \Bphi, \log \frac{1}{\sigw}, H, \log \frac{T_i}{\delta} \right ) .
\end{align*}
However, we can bound 
\begin{align*}
\sum_{i=1}^{\ell-1}  \poly \left ( \dimphi, \dimx, \frac{1}{\lamminst}, \BA, \Bphi, \log \frac{1}{\sigw}, H, \log \frac{T_i}{\delta} \right )  \le   \poly \left ( \dimphi, \dimx, \frac{1}{\lamminst}, \BA, \Bphi, \log \frac{1}{\sigw}, H, \log \frac{T_\ell}{\delta} \right )  \cdot \log T_\ell,
\end{align*}
so a sufficient condition is
\begin{align*}
T_\ell \ge  \poly \left ( \dimphi, \dimx, \frac{1}{\lamminst}, \BA, \Bphi, \log \frac{1}{\sigw},  H, \log \frac{T_\ell}{\delta} \right )  \cdot \log T_\ell
\end{align*}
which we see is met when \eqref{eq:T_bound_burnin} holds. 
\end{proof}

\begin{lemma}\label{lem:knr_est_bound}
Consider running \Cref{alg:learn_policies} with weight matrix $\cH$, parameter $\Ntil$, and confidence $\delta$, and rerunning each policy in $\Piout$ $N \le \Ntil$ times. Then, under \Cref{asm:full_rank_cov,asm:bounded_features}, with probability at least $1-6\delta$:
\begin{align*}
\| \vec(\Ahat - \Ast) \|_\cH^2 & \le \frac{60}{N\Kout} \cdot  \min_{\bLambda \in \bOmega} \tr \left ( \cH \bLamchk^{-1} \right )  \cdot \sigw^2 \log \frac{6 \dimx \dimphi}{\delta}  \\
& \qquad + \poly \left ( \dimphi, \dimx, \frac{1}{\lamminst}, \BA, \Bphi, \log \frac{1}{\sigw}, H, \| \cH \|_\op, \log \frac{\Ntil}{\delta} \right )  \cdot \frac{1}{N^2}
\end{align*}
where $\Ahat$ denotes the least-squares estimate of $\Ast$ obtained on the data generated by rerunning $\Piout$. In addition, we have
\begin{align*}
\| \Ahat - \Ast \|_\fro \le   \poly \left ( \dimphi, \dimx, \frac{1}{\lamminst}, \BA, \Bphi, \sigw, H, \log \frac{\Ntil}{\delta} \right )  \cdot \frac{1}{\sqrt{N}}.
\end{align*}
Furthermore, we have
\begin{align*}
\Kout \le  \poly \left ( \dimphi, \dimx, \frac{1}{\lamminst}, \BA, \Bphi, \log \frac{1}{\sigw}, H, \log \frac{\Ntil}{\delta} \right ) 
\end{align*}
and the total number of episodes collected by this procedure is bounded by $N \Kout + (16 + 2 \log(\Kout)) \cdot \Kout$, for $\Kout = |\Piout|$.
\end{lemma}
\begin{proof}
By \Cref{lem:knr_satisfies_ed_asm}, the assumptions of \Cref{lem:fw_exp_good} and \Cref{lem:fw_exploration_good_event} are met, so we can therefore apply these results in our setting. 
By \Cref{lem:fw_exp_good}, the event $\cEexp$ occurs with probability at least $1-\delta$. Throughout the remainder of the proof we union bound over the success event of \Cref{lem:fw_exploration_good_event} and $\cEexp$, which together occur with probability at least $1-4\delta$. 

Let 
\begin{align*}
\bLamtil := \sum_{t=1}^{N\Tout} \bpsi(\traj^t) = I_{\dimx} \otimes \sum_{t=1}^{N \Tout} \sum_{h=1}^H \bphi(\bx_h^t,\bu_h^t) \bphi(\bx_h^t,\bu_h^t)^\top
\end{align*}
denote the features returned by rerunning every policy in $\Piout$ $N$ times.
By \Cref{lem:fw_exploration_good_event}, we then have that: 
\begin{align*}
\left \| \bLamtil - N \cdot \tsum_{\pi \in \Piout} \bLamchk_{\pi} \right \|_\op \le \underbrace{\frac{\sqrt{ \Kout} \cdot \sqrt{8 \dimphi \dimx \log(1 + 8 \sqrt{N\Kout}) + 8 \log 1/\delta}}{\sqrt{N} \cdot 6272  \dimphi \dimx \log \frac{68\Ntil}{\delta}}}_{=: \beta} \cdot \lammin \left ( N \cdot \tsum_{\pi \in \Piout} \bLamchk_{\pi} \right ).
\end{align*}
Applying \Cref{prop:M_reg_bound} with $\cE$ the event that the above conclusion holds
and $\bGamma := N \cdot \tsum_{\pi \in \Piout} \bLamchk_{\pi}$, we obtain that, with probability at least $1-\delta$ (using the mapping to the martingale regression setting described in \Cref{sec:knr_mdm}):
\begin{align*}
\| \vec(\Ahat - \Ast) \|_\cH^2 \le 5 (1 + \zeta) \cdot \sigw^2 \log \frac{6 \dimx \dimphi}{\delta} \cdot \tr(\cH \bGamma^{-1})
\end{align*}
for $\zeta = 26 \beta^2 \lammax(\bGamma) \tr(\bGamma^{-1})$ and $\beta$ as defined above. Since $\| \bphi(\bx,\bu) \|_2 \le \Bphi$ under \Cref{asm:bounded_features}, we have $\| \bLamchk_\pi \|_2 \le \Bphi^2$, so we can upper bound $\lammax(\bGamma) \le N \Kout \Bphi^2$. 
By \Cref{lem:fw_exploration_good_event}, we can also bound (using that $D = \dimx H \Bphi^2$ by \Cref{lem:knr_satisfies_ed_asm}):
\begin{align*}
\tr(\bGamma^{-1}) \le  \frac{1}{N \cdot 6272 \dimx H \Bphi^2 \log \frac{68\Ntil}{\delta}}.
\end{align*}
Combining these and using that
\begin{align}\label{eq:H_reg_Kout_bound}
\Kout \le \poly \left ( \dimphi, \dimx, \frac{1}{\lamminst}, \BA, \Bphi, \log \frac{1}{\sigw}, H, \log \frac{\Ntil}{\delta} \right )
\end{align}
as shown in \Cref{lem:Kout_bound} (and using our bounds on $\CR$ and $\pR$ in \Cref{lem:knr_satisfies_ed_asm}), we can therefore bound
\begin{align*}
\zeta \le  \poly \left ( \dimphi, \dimx, \frac{1}{\lamminst}, \BA, \Bphi, \log \frac{1}{\sigw}, H, \log \frac{\Ntil}{\delta} \right ) \cdot \frac{1}{N}.
\end{align*}
Using that $\tr(\cH \bGamma^{-1}) \le \| \cH \|_\op \cdot \tr(\bGamma^{-1})$, and the bound on $\tr(\bGamma^{-1})$ given above, it follows that
\begin{align*}
\| \vec(\Ahat - \Ast) \|_\cH^2 \le 5 \sigw^2 \log \frac{6 \dimx \dimphi}{\delta} \cdot \tr(\cH \bGamma^{-1}) + \poly \left ( \dimphi, \dimx, \frac{1}{\lamminst}, \BA, \Bphi, \log \frac{1}{\sigw}, H, \| \cH \|_\op, \log \frac{\Ntil}{\delta} \right )  \cdot \frac{1}{N^2}.
\end{align*}
Finally, by \Cref{lem:fw_exploration_good_event}, we can bound
\begin{align*}
\tr(\cH \bGamma^{-1}) \le \frac{12}{N \Kout} \cdot \min_{\bLambda \in \bOmega} \tr (\cH \bLamchk^{-1}).
\end{align*}
By \Cref{lem:Kout_bound}, we can bound the total number of episodes collected by \Cref{alg:learn_policies} by $(16 + 2 \log(\Kout)) \cdot \Kout$.

\paragraph{Bound on Frobenius Norm Error.}
By \Cref{lem:fw_exploration_good_event}, we can lower bound
\begin{align*}
\lammin(\bLamtil) \ge N \cdot 6272 \dimx \dimphi H \Bphi^2 \log \frac{68\Ntil}{\delta}.
\end{align*}
Furthermore, since $\| \bphi(\bx,\bu) \|_2 \le \Bphi$, we always have $\| \bLamtil \|_\op \le N \Kout \Bphi^2$, which implies $\bLamtil \preceq N\Kout \Bphi^2 \cdot I$. By \Cref{prop:fro_reg_bound}, we then have that with probability at least $1-\delta$ (again using the mapping to the martingale regression setting described in \Cref{sec:knr_mdm}):
\begin{align*}
\| \Ahat - \Ast \|_\fro & \le C \cdot \sigw \sqrt{\frac{\log 1/\delta + \dimx \dimphi + \log\det( \frac{\Kout}{6272 \dimx \dimphi H  \log \frac{68\Ntil}{\delta}} \cdot I + I)}{N \cdot 6272 \dimx \dimphi H \Bphi^2 \log \frac{68\Ntil}{\delta}}} \\
& \le \poly \left ( \dimphi, \dimx, \frac{1}{\lamminst}, \BA, \Bphi, \sigw, H,  \log \frac{\Ntil}{\delta} \right ) \cdot \frac{1}{\sqrt{N}}.
\end{align*}
\end{proof}

\begin{lemma}\label{lem:ce_exp_design}
Under \Cref{asm:full_rank_cov}, for any $\cH,\cH'$, we can bound
\begin{align*}
\min_{\bLamchk \in \bOmega} \tr(\cH \bLamchk^{-1}) \le \min_{\bLamchk \in \bOmega} 2\tr(\cH' \bLamchk^{-1}) + \frac{2 \dimx \dimphi}{\lamminst} \cdot \| \cH - \cH' \|_\op.
\end{align*}
\end{lemma}
\begin{proof}
We have
\begin{align*}
\min_{\bLamchk \in \bOmega} \tr(\cH \bLamchk^{-1}) & = \min_{\bLamchk \in \bOmega} \tr(\cH' \bLamchk^{-1}) + \tr((\cH - \cH') \bLamchk^{-1}) \\
& \le \min_{\bLamchk \in \bOmega} \tr(\cH' \bLamchk^{-1}) + \| \cH - \cH' \|_\op \cdot \tr( \bLamchk^{-1})
\end{align*}
Under \Cref{asm:full_rank_cov}, we know that there exists some $\bLamchk' \in \bOmega$ such that $\lammin(\bLamchk') \ge \lamminst$. We can then bound
\begin{align*}
& \min_{\bLamchk \in \bOmega} \tr(\cH' \bLamchk^{-1}) + \| \cH - \cH' \|_\op \cdot \tr( \bLamchk^{-1})\\
& \le \min_{\bLamchk \in \bOmega} \tr(\cH' (\frac{1}{2}\bLamchk + \frac{1}{2} \bLamchk')^{-1}) + \| \cH - \cH' \|_\op \cdot \tr( (\frac{1}{2} \bLamchk + \frac{1}{2} \bLamchk')^{-1}) \\
& \le \min_{\bLamchk \in \bOmega} 2 \tr(\cH' \bLamchk^{-1}) + 2\| \cH - \cH' \|_\op \cdot \frac{\dimx\dimphi}{\lamminst}
\end{align*}
which proves the result.
\end{proof}


\section{Experiment Design in Arbitrary Dynamical Systems}\label{sec:nl_exp_design}

In this section we generalize somewhat the setting of \Cref{sec:exp_design}. In particular, our goal will now be to collect some set of trajectories $\frakD = \{ \traj_t \}_{t=1}^T$, which minimize
\begin{align*}
\Phi \left ( \frac{1}{T} \sum_{t=1}^T \bpsi(\traj_t) \right ) 
\end{align*}
for some general feature mapping $\bpsi  :  \cT \rightarrow \R^{\dimpsib}$, $\Phi :  \R^{\dimpsib} \rightarrow \R$, and $\cT = (\cX \times \cU)^H \times \cX$ the space of possible state-input trajectories, $\traj = (\bx_0, \bu_0, \bx_1, \ldots,\bx_{H-1}, \bu_{H-1}, \bx_H) \in \cT$.
We will assume that $\bpsi$ can be decomposed additively as
\begin{align*}
\bpsi(\traj) = \sum_{h=1}^H \bpsi_h(\bx_h,\bu_h).
\end{align*}
In \Cref{sec:exp_design} we considered the special case where $\bpsi_h(\bx,\bu) = \bphi(\bx,\bu) \bphi(\bx,\bu)^\top$; in this section $\bpsi$ could instead be any arbitrary mapping.

As before, we will be interested in defining optimal exploration with respect to some set of exploration policies, $\Piexp$. Let 
\begin{align*}
\bOmpsi := \{ \Exp_{\pi \sim \omega}[\Exp_{\pi}[\bpsi(\traj)]] \ : \ \omega \in \simplex_\Pi \}
\end{align*}
denote the space of expected value of $\bpsi(\traj)$ for mixtures of policies in $\Piexp$. To distinguish elements $\bLambda \in \bOmega$ from elements in $\bOmpsi$, we will let $\bGamma = \Exp_{\pi \sim \omega}[\Exp_\pi[\bpsi(\traj)]]$ refer to elements of $\bOmpsi$, and in particular define $\bGamma_\pi := \Exp_\pi[\bpsi(\traj)]$ (where it is assumed that the expectation is collected over trajectories on \eqref{eq:system_general}). We will usually denote unnormalized sums of features, e.g. $\sum_{t = 1}^T \bpsi(\traj_t)$, with $\bSigma$. We also define $\bOmhatpsi$ to be the space of all possible combinations of $\bpsi(\traj)$:
\begin{align*}
\bOmhatpsi:= \{ \Exp_{\traj \sim \omega}[\bpsi(\traj)] \ : \ \omega \in \simplex_\cT \}.
\end{align*}
We generalize \Cref{asm:regular_phi} and \Cref{asm:regret_min} as follows.

\begin{asm}[Regularity of $\Phi$]\label{asm:regular_phi_app}
We make the following assumptions:
\begin{enumerate}
\item $\Phi$ is convex, differentiable, and $\beta$-smooth in the norm $\| \cdot \|$:
\begin{align*}
\| \nabla_{\bGamma} \Phi(\bGamma) - \nabla_{\bGamma'} \Phi(\bGamma') \|_* \le \beta \cdot \| \bGamma - \bGamma' \|, \quad \forall \bGamma,\bGamma' \in \bOmhatpsi
\end{align*} 
for $\| \cdot \|_*$ the dual norm of $\| \cdot \|$.
\item There exists some $M < \infty$ satisfying 
\begin{align*}
\sup_{\bGamma \in \bOmhatpsi} \sup_{\traj \in \cT} | \inner{ \nabla_{\bGamma} \Phi(\bGamma)}{\bpsi(\traj)} | \le M.
\end{align*}
\end{enumerate}
\end{asm}

\begin{asm}[Regret Minimization Oracle]\label{asm:regret_min_app}
Let $\cost_h(\traj) = \inner{ Q_h}{\bpsi_h(\traj)}$ for some $Q_h \in \R^{\dimpsib}$, and $\cost(\traj) = \sum_{h=1}^H \cost_h(\traj)$ the total cost of trajectory $\traj$. We assume we have access to some learner $\regalg$ which, in the setting when $|\cost(\traj)| \le 1$ for all $\traj \in \cT$,
is able to achieve low regret on $\{ \cost_h(\cdot,\cdot) \}_{h=1}^H$ with respect to policy class $\Piexp$. That is, with probability at least $1-\delta$:
\begin{align*}
 \sum_{t=1}^T \Exp_{f,\pi_t} [\cost(\traj_t) ] - T \cdot \inf_{\pi \in \Piexp} \Exp_{f,\pi} [ \cost(\traj)  ] \le \CR \cdot \log^{\pR} \frac{T}{\delta} \cdot T^\alpha
\end{align*}
for some $\CR > 0$, $\pR > 0$, and $\alpha \in (0,1)$, and where $\pi_t$ is the policy $\regalg$ plays at episode $t$. 
\end{asm}

\begin{algorithm}[h]
\begin{algorithmic}[1]
\State \textbf{input}: objective $\Phi$, number of episodes $T$ (OR number of iterates $N$, episodes per iterate $K$), confidence $\delta$, regret minimization algorithm $\regalg$, exploration policies $\Piexp$
\State Play any policy $\piexp \in \Piexp$ for $K$ episodes, collect trajectories $\frakD_0 = \{ \traj_k^0 \}_{k=1}^K$, set $\bGamma_0 \leftarrow K^{-1} \sum_{k=1}^K \bpsi(\traj_k^0)$
\For{$n = 1,2,\ldots,N$}
	\State Set $\gamma_n \leftarrow \frac{1}{n+1}$
	\State Run $\regalg$ on cost
	\begin{align*}
	\cost_h^n(\traj) \leftarrow \frac{1}{M} \inner{\Xi_n}{\bpsi_h(\bx_h,\bu_h)} \quad \text{for} \quad \Xi_n \leftarrow \nabla_{\bGamma} \Phi(\bGamma)|_{\bGamma = \bGamma_{n-1}}
	\end{align*}
	for $K$ episodes, collect trajectories $\frakD_n = \{ \traj_k^n \}_{k=1}^K$, denote policies run as $\Pi_n$
	\State $\bGamma_{n} \leftarrow (1-\gamma_n) \bGamma_{n-1} + \gamma_n K^{-1} \sum_{k=1}^K \bpsi(\traj_k^n)$
\EndFor
\State \textbf{return} $(N+1) K \bGamma_{N}$, $\cup_{n=0}^N \Pi_n$, $\cup_{n=0}^N \frakD_n$
\end{algorithmic}
\caption{Dynamic Optimal Experiment Design (\expdesign)}
\label{alg:regret_fw}
\end{algorithm}

We define \expdesign as in \Cref{alg:regret_fw}. We then have the following generalization of \Cref{lem:fwregret}.

\begin{theorem}[Full Version of \Cref{lem:fwregret}]\label{lem:fwregret_app}
Let \Cref{asm:regular_phi} hold, and assume that we have access to a learner $\regalg$ satisfying \Cref{asm:regret_min}. Fix $N,K > 0$. Then, with probability at least $1-\delta$, \expdesign runs for at most $(N+1)K$ episodes, and collects a dataset satisfying $\frakD = \{ \{ \traj_k^n \}_{k=1}^K \}_{n=0}^N$ satisfying
\begin{align*}
\Phi \left (\frac{1}{K(N+1)} \sum_{n=0}^N \sum_{k = 1}^K \bpsi(\traj_k^n) \right ) -  \min_{\bGamma \in \bOmpsi} \Phi( \bGamma) & \le \frac{\beta R^2 ( \log N + 1)}{2 (N + 1)} +  M \cdot \bigg ( \CR\log^{\pR} \frac{2NK}{\delta} \cdot K^{\alpha - 1} \\
& \qquad +  \sqrt{\frac{8 \log(4N/\delta)}{K}} \bigg )
\end{align*}
where $R = \sup_{\bGamma,\bGamma' \in \bOmhatpsi} \| \bGamma - \bGamma' \|$.
\end{theorem}

In this work we are particularly interested in the case where $\bpsi(\traj) \in \cS^{\dimpsi}_+$. We encapsulate this in the following assumption.
\begin{asm}[Matrix Experiment Design]\label{asm:mat_psi}
We assume that $\bpsi(\traj) \in \cS^{\dimpsi}_+$ and that, for all $\traj \in \cT$, $\tr(\bpsi(\traj)) \le D$ for some $D > 0$. 
\end{asm}
The following corollary instantiates \Cref{lem:fwregret} under \Cref{asm:mat_psi} with objective $\Phi(\bGamma) = \tr \left ( \cH (\bGamma + \bGamma_0)^{-1} \right )$, the objective considered in \Cref{alg:main}. 

\begin{corollary}\label{lem:weighted_Aopt}
Consider the objective
\begin{align*}
\Phi(\bGamma) = \tr \left ( \cH \cdot (\bGamma + \bGamma_0)^{-1} \right )
\end{align*}
and assume that $\cH \succeq 0$ and \Cref{asm:regret_min_app} holds with $\alpha = 1/2$ and \Cref{asm:mat_psi} holds. Fix $N,K$, let $T := (N+1)K$, and consider running \Cref{alg:regret_fw} on this objective and with these choices of $N$ and $K$. Then \Cref{alg:regret_fw} will run for at most $T$ episodes, 
and, with probability at least $1-\delta$, will return data satisfying
\begin{align*}
\tr \left ( \cH \left ( \sum_{t=1}^T \bpsi(\traj_t) + T \bGamma_0 \right )^{-1} \right )&  \le \frac{1}{T} \cdot  \min_{\bGamma \in \bOmpsi} \tr \left ( \cH (\bGamma + \bGamma_0)^{-1} \right ) + \frac{8D^4 \| \cH \|_\op \| \bGamma_0^{-1} \|_\op^3}{T (N+1)}  \\
& \qquad +  \frac{8D \| \cH \|_\op \| \bGamma_0^{-1} \|_\op^2 (  \log^{1/2} \frac{4T}{\delta} + \CR \log^{\pR} \frac{2T}{\delta})   }{T \sqrt{K}}   
\end{align*}
\end{corollary}

\subsection{Proof of \Cref{lem:fwregret} and \Cref{lem:fwregret_app}}
\begin{algorithm}[h]
\begin{algorithmic}[1]
\State \textbf{input}: function to optimize $f$, number of iterations to run $N$, starting iterate $\bx_1$
\For{$t = 1,2,\ldots,N$}
	\State Set $\gamma_n \leftarrow \frac{1}{n+1}$
	\State Choose $\by_n$ to be any point such that
	\begin{align*}
	\nabla f(\bz_n)^\top \by_n \le \min_{\by \in \cZ} \nabla f(\bz_n)^\top \by + \epsilon_n
	\end{align*}
	\State $\bz_{n+1} \leftarrow (1- \gamma_n) \bz_n + \gamma_n \by_n$
\EndFor
\State \textbf{return} $\bx_{N+1}$
\end{algorithmic}
\caption{Approximate Frank-Wolfe}
\label{alg:approx_fw}
\end{algorithm}

\begin{lemma}[Lemma C.1 of \cite{wagenmaker2022instance}]\label{lem:approx_fw}
Consider running \Cref{alg:approx_fw} with some convex function $f$ that is $\beta$-smooth with respect to some norm $\| \cdot \|$, assume that $\by_n \in \cY$ for some $\cY$ and all $n$, and let $R := \sup_{\bz,\by \in \cZ \cup \cY} \| \bz - \by \|$. Then for $N \ge 2$, we have 
\begin{align*}
f(\bz_{N+1}) - \min_{\bz \in \cZ} f(\bz) \le \frac{\beta R^2 (\log N + 1)}{2(N+1)} + \frac{1}{N+1} \sum_{n=1}^{N}  \epsilon_n.
\end{align*}
\end{lemma}

\begin{lemma}[Lemma C.2 of \cite{wagenmaker2022instance}]\label{lem:fw_final_iterate}
When running \Cref{alg:approx_fw}, we have
\begin{align*}
\bz_{N+1} = \frac{1}{N+1} \left ( \sum_{n=1}^N \by_n + \bz_1 \right ).
\end{align*}
\end{lemma}

\begin{proof}[Proof of \Cref{lem:fwregret}]
By our assumption on $\regalg$, \Cref{asm:regret_min}, we have that, at round $n$, with probability at least $1-\delta/2N$,
\begin{align*}
&  \sum_{k=1}^K \Exp_{\pi_k} [\cost^n(\traj_k) ]  - K \cdot \inf_{\pi \in \Piexp} \Exp_\pi \left [ \cost^n(\traj) \right ] \le \CR \log^{\pR} \frac{2NK}{\delta}  \cdot K^\alpha 
\end{align*}
where we have used that, under \Cref{asm:regular_phi_app} and by the definition of $\cost_h^n(\traj)$, $|\cost^n(\traj)| \le 1$ for all $\traj \in \cT$. This implies that
\begin{align*}
 \frac{1}{K} \sum_{k=1}^K \Exp_{\pi_k}[\inner{\Xi_n}{\bpsi(\traj_k)}  ]  \le M \cdot \inf_{\pi \in \Piexp} \Exp_\pi \left [ \cost^n(\traj) \right ] + M\CR \log^{\pR} \frac{2NK}{\delta} \cdot K^{\alpha - 1}.
\end{align*}
Furthermore, by Azuma-Hoeffding and under \Cref{asm:regular_phi_app}, we have that, with probability at least $1-\delta/2N$, 
\begin{align*}
\left | \frac{1}{K} \sum_{k=1}^K \inner{\Xi_n}{\bpsi(\traj_k)} - \frac{1}{K} \sum_{k=1}^K \Exp_{\pi_k} [\inner{\Xi_n}{\bpsi(\traj_k)}] \right | \le \sqrt{\frac{8 M^2 \log(4N/\delta)}{K}}.
\end{align*}
This implies that
\begin{align}\label{eq:approx_fw_guarantee}
\frac{1}{K} \sum_{k=1}^K \inner{\Xi_n}{\bpsi(\traj_k)} \le M \cdot  \inf_{\pi \in \Piexp} \Exp_\pi \left [ \cost^n(\traj) \right ] + M  \cdot \left ( \CR \log^{\pR} \frac{K}{\delta} \cdot K^{\alpha - 1} +  \sqrt{\frac{8 \log(4N/\delta)}{K}} \right ).
\end{align}

Note that
\begin{align*}
\inner{\Xi_n}{\bpsi(\traj_k)} = \inner{\nabla_{\bGamma} \Phi(\bGamma)|_{\bGamma = \bGamma_n }}{\bpsi(\traj)},
\end{align*}
and that for any $\bGamma \in \bOmpsi$, we have
\begin{align*}
\inner{\nabla_{\bGamma} \Phi(\bGamma)|_{\bGamma = \bGamma_n }}{\bGamma} = \Exp_{\pi \sim \omega}[\Exp_\pi[\inner{\nabla_{\bGamma} \Phi(\bGamma)|_{\bGamma = \bGamma_n }}{\bpsi(\traj)}]]
\end{align*}
for some $\omega$. This implies that
\begin{align*}
\inf_{\bGamma \in \bOmpsi} \inner{\nabla_{\bGamma} \Phi(\bGamma)|_{\bGamma = \bGamma_n }}{\bGamma} & = \inf_{\omega \in \simplex_\Pi} \Exp_{\pi \sim \omega}[\Exp_\pi[\inner{\nabla_{\bGamma} \Phi(\bGamma)|_{\bGamma = \bGamma_n }}{\bpsi(\traj)}]] \\
& = \inf_{\pi \in \Piexp} \Exp_\pi [ \inner{\Xi_n}{\bpsi(\traj)} ] \\
& = M \cdot \inf_{\pi \in \Piexp} \Exp_\pi [ \cost^n(\traj) ] .
\end{align*}
By \eqref{eq:approx_fw_guarantee} above, we have that
\begin{align*}
\frac{1}{K} \sum_{k=1}^K  \bpsi(\traj_k)
\end{align*}
is an approximate minimizer of $M \cdot \sup_{\pi \in \Piexp} \Exp_\pi [ \cost^n(\traj) ]$, with approximation tolerance $M ( \CR \log^{\pR} \frac{2NK}{\delta} \cdot K^{\alpha - 1} +  \sqrt{\frac{8 \log(4N/\delta)}{K}})$. We can therefore apply \Cref{lem:approx_fw} with
\begin{align*}
\epsilon_n = M \cdot \left ( \CR \log^{\pR} \frac{K}{\delta} \cdot K^{\alpha - 1} + \sqrt{\frac{8 \log(4N/\delta)}{K}} \right )
\end{align*}
to get that
\begin{align*}
\Phi(\bGamma_{N+1}) -  \min_{\bGamma \in \bOmpsi} \Phi( \bGamma) \le \frac{\beta R^2 ( \log N + 1)}{2 (N + 1)} +  M  \cdot \left ( \CR \log^{\pR} \frac{2NK}{\delta} \cdot K^{\alpha - 1} +  \sqrt{\frac{8 \log(4N/\delta)}{K}} \right ).
\end{align*}
The result then follows since $\bGamma_{N+1} = \frac{1}{K(N+1)} \sum_{n=0}^N \sum_{k = 1}^K \bpsi(\traj_k^n)$ by \Cref{lem:fw_final_iterate}.

\end{proof}

\begin{proof}[Proof of \Cref{lem:weighted_Aopt}]
By \Cref{lem:fwregret_app}, for any setting of $N$ and $K$, we have that with probability at least $1-\delta$:
\begin{align*}
& \tr \left ( \cH \left (\frac{1}{K(N+1)} \sum_{n=0}^N \sum_{k = 1}^K \bpsi(\traj_k^n) + \bGamma_0 \right )^{-1} \right ) -  \min_{\bGamma \in \bOmega} \tr \left ( \cH (\bGamma + \bGamma_0)^{-1} \right ) \\
& \qquad \le  \frac{\beta R^2 \log N }{N + 1} +  \frac{M\CR \log^{\pR} \frac{2NK}{\delta} }{K^{1-\alpha}} + M \sqrt{\frac{8 \log \frac{4N}{\delta}}{K}}
\end{align*}
which implies 
\begin{align*}
 & \tr \left ( \cH \left ( \sum_{n=0}^N \sum_{k = 1}^K \bpsi(\traj_k^n) + T \bGamma_0 \right )^{-1} \right ) - \frac{  \min_{\bGamma \in \bOmega} \tr \left ( \cH (\bGamma + \bGamma_0)^{-1} \right )}{T} \\
 & \qquad \le   \frac{\beta R^2 \log N }{T(N+1)} +  \frac{M\CR \log^{\pR} \frac{2NK}{\delta} }{T \sqrt{K} } + M \frac{1}{T} \sqrt{\frac{8 \log \frac{4N}{\delta}}{K}}
\end{align*}
This gives
\begin{align*}
\tr \left ( \cH \left ( \sum_{n=0}^N \sum_{k = 1}^K \bpsi(\traj_k^n) + T \bGamma_0 \right )^{-1} \right ) & \le \frac{  \min_{\bGamma \in \bOmega} \tr \left ( \cH (\bGamma + \bGamma_0)^{-1} \right )}{T} \\
& \qquad + \frac{\beta R^2 \log T}{T(N+1)} + \frac{M(3\log^{1/2} \frac{4T}{\delta} + \CR \log^{\pR} \frac{2T}{\delta}) }{T \sqrt{K} }  .
\end{align*}

It then remains to bound $R,\beta,$ and $M$. By Lemma D.6 of \cite{wagenmaker2022instance}, we have that
\begin{align*}
\nabla_{\bGamma} \Phi(\bGamma)[\bGamtil] = - \tr \left (\cH (\bGamma + \bGamma_0)^{-1} \bGamtil (\bGamma + \bGamma_0)^{-1} \right ). 
\end{align*}
We can then compute the second derivative as, using Lemma D.6 of \cite{wagenmaker2022instance}:
\begin{align*}
 \nabla_{\bGamma}^2 \Phi(\bGamma)[\bGamtil,\bGambar]  & = \frac{\rmd}{\rmd t} \left [  - \tr \left (\cH (\bGamma + \bGamma_0 + t \bGambar)^{-1} \bGamtil (\bGamma + \bGamma_0 + t \bGambar)^{-1} \right ) \right ] \\
& = \tr \left ( \cH ( \bGamma + \bGamma_0)^{-1} \bGambar ( \bGamma + \bGamma_0)^{-1} \bGamtil ( \bGamma + \bGamma_0)^{-1} \right ) \\
& \qquad + \tr \left ( \cH ( \bGamma + \bGamma_0)^{-1} \bGamtil ( \bGamma + \bGamma_0)^{-1} \bGambar ( \bGamma + \bGamma_0)^{-1} \right ).
\end{align*}
Recall that $M$ is any bound on 
\begin{align*}
\sup_{\bGamma \in \bOmhatpsi} \sup_{\traj \in \cT} | \inner{\nabla_{\bGamma} \Phi(\bGamma)}{\bpsi(\traj)}| .
\end{align*}
By the above computation of the gradient, we can bound this as
\begin{align}\label{eq:M_bound_Aopt}
\begin{split}
 \sup_{\bGamma \in \bOmhatpsi} \sup_{\traj \in \cT}  | \inner{\nabla_{\bGamma} \Phi(\bGamma)}{\bpsi(\traj)}| & \le \sup_{\bGamma \in \bOmhatpsi} \sup_{\traj \in \cT} \left | \tr \left ( \cH (\bGamma + \bGamma_0)^{-1} \bpsi(\traj) (\bGamma + \bGamma_0)^{-1} \right ) \right | \\
& \le \| \cH \|_\op \| \bGamma_0^{-1} \|_\op^2 \cdot \sup_{\traj \in \cT}  \tr(\bpsi(\traj)) \\
& \le  D \| \cH \|_\op \| \bGamma_0^{-1} \|_\op^2 .
\end{split}
\end{align}
To bound $\beta$, by the Mean Value Theorem it suffices to bound the operator norm of $ \nabla_{\bGamma}^2 \Phi(\bGamma)$. Using the expression above, we can bound this as 
\begin{align*}
\sup_{\bGamtil,\bGambar \in \bOmhatpsi} | \nabla_{\bGamma}^2 \Phi(\bGamma)[\bpsi(\traj_1),\bpsi(\traj_2)] | & \le 2 \| \cH \|_\op \| \bGamma_0^{-1} \|_\op^3 \cdot \sup_{\bGamtil,\bGambar \in \bOmhatpsi} \tr(\bGamtil \bGambar) \\
& \le 2D^2 \| \cH \|_\op \| \bGamma_0^{-1} \|_\op^3. 
\end{align*}
Finally, it's straightforward to bound $R \le 2D$. Putting all of this together gives the result.
\end{proof}

\subsection{Collecting Full-Rank Data}\label{sec:full_rank}

\begin{algorithm}[h]
\begin{algorithmic}[1]
\State \textbf{input}: scale $N$, confidence $\delta$, regret minimization algorithm $\regalg$, exploration policies $\Piexp$
\For{$j = 1,2,3,\ldots$}
	\State $N_j \leftarrow \lceil 2^{j/3} \rceil - 1, K_j \leftarrow \lceil 2^{2j/3} \rceil,T_j \leftarrow (N_j+1) K_j, \lambda_j \leftarrow T_j^{-1/18}, \delta_j \leftarrow \frac{\delta}{4j^2}$
	\State $\bSigma_j, \Pi_j \leftarrow \expdesign(\Phi,N_j,K_j,\delta_j,\regalg,\Piexp)$ for $\Phi(\bGamma) = \tr ((\bGamma + \lambda_j \cdot I)^{-1})$
	\If{$\lammin(\bSigma_j) \ge 12544 D \dimpsi \log \frac{2N(2 + 32  T_j)}{\delta}$}\label{line:lammin_if}
		\State \textbf{break}
	\EndIf
\EndFor
\State \textbf{return} $\Pi_j$
\end{algorithmic}
\caption{Minimum Eigenvalue Maximization (\mineigalg)}
\label{alg:lammin}
\end{algorithm}

In this section, we consider the setting where $\bpsi(\traj) \in \cS_+^{\dimpsi}$, and our goal is to collect $\{ \traj_t \}_{t=1}^T$ such that $\bpsi(\frac{1}{T} \sum_{t=1}^T \bpsi(\traj_t)) > 0$. For this to be achievable, we need the following assumption, a generalization of \Cref{asm:full_rank_cov}.

\begin{asm}[Full-Rank Data]\label{asm:full_rank_cov_app}
Consider $\bpsi(\traj)$ such that $\bpsi(\traj) \in \bbS_+^{\dimpsi}$.
Then we have $\sup_{\bGamma \in \bOmpsi} \lammin(\bGamma) \ge \lamminst$ for some $\lamminst > 0$.
\end{asm}

Throughout this section we also assume that \Cref{asm:regret_min_app} is satisfied with $\alpha = 1/2$ (though all results generalize in a straightforward way for $\alpha \neq 1/2$).
We have the following result.

\begin{lemma}\label{lem:min_eig_policies}
Under \Cref{asm:regret_min_app,asm:mat_psi,asm:full_rank_cov_app}, running \Cref{alg:lammin} we have that with probability at least $1-\delta$, it will terminate after collecting at most
\begin{align*}
\poly \left ( \dimpsi, \frac{1}{\lamminst}, D, \CR, \log^{\pR} \frac{N}{\delta} \right )
\end{align*}
episodes, and return policy set $\Pi$ such that
\begin{align*}
\lammin \left ( \tsum_{\pi \in \Pi} \bGamma_\pi \right ) \ge 6272 D \dimpsi \log \frac{68N}{\delta}.
\end{align*}
Furthermore, if we rerun each policy in $\Pi$ once, the resulting features $\bSigma$ will satisfy, with probability at least $1-\delta/N$:
\begin{align*}
\lammin \left ( \bSigma \right ) \ge 6272 D \dimpsi \log \frac{68N}{\delta} .
\end{align*}
\end{lemma}
\begin{proof}
By \Cref{lem:min_eig_suff} and our choice of $N_j$ and $K_j$ in \Cref{alg:lammin}, we have that if $\lambda_j \le \frac{\lamminst}{4 \dimpsi}$ and
\begin{align}\label{eq:lammin_proof_kj_cond}
T_j^{1/3} \ge \widetilde{\Omega} \left ( \left ( D \lambda_j^{-2} (D^3 \lambda_j^{-1}  + \CR \cdot \log^{\pR} \frac{1}{\delta_j} ) \right ) \cdot \frac{\lamminst }{\dimpsi} \right ),
\end{align}
then $\lammin(\bSigma_j) \ge \frac{\lamminst}{4 \dimpsi} \cdot T_j$ with probability at least $1-\delta_j$. It follows that, with probability at least $1-\delta_j$, the if statement on \Cref{line:lammin_if} will be true once $\lambda_j \le \frac{\lamminst}{4 \dimpsi}$, \eqref{eq:lammin_proof_kj_cond} holds, and
\begin{align}\label{eq:lammin_proof_kj_cond2}
\frac{\lamminst}{4 \dimpsi} \cdot T_j \ge 12544 D \dimpsi \log \frac{2N(2 + 32 T_j)}{\delta}.
\end{align}
By our choice of $\lambda_j = T_j^{-1/18}$, a sufficient condition to ensure $\lambda_j \le \frac{\lamminst}{4 \dimpsi}$, \eqref{eq:lammin_proof_kj_cond}, and \eqref{eq:lammin_proof_kj_cond2} is
\begin{align*}
T_j \ge \widetilde{\Omega} \left ( \max \left \{ \left ( \frac{\dimpsi}{\lamminst} \right )^{18}, \left ( \frac{D^4 \lamminst }{\dimpsi} \right )^{6}, \left ( D \CR \cdot \log^{\pR} \frac{1}{\delta_j}\cdot \frac{\lamminst }{\dimpsi} \right )^{9/2}, \frac{D \dimpsi^2}{\lamminst} \cdot \log \frac{N T_j}{\delta} \right \} \right ).
\end{align*}
Since $T_j = \lceil 2^{j/3} \rceil \lceil 2^{2j/3} \rceil \in [2^j, 4 \cdot 2^j]$, it follows that the if statement on \Cref{line:lammin_if} will be met after running for at most
\begin{align}\label{eq:lammin_pi_bound}
 \cOtil \left ( \max \left \{ \left ( \frac{\dimpsi}{\lamminst} \right )^{18}, \left ( \frac{D^4 \lamminst }{\dimpsi} \right )^{6}, \left ( D \CR \cdot \log^{\pR} \frac{1}{\delta_j}\cdot \frac{\lamminst }{\dimpsi} \right )^{9/2}, \frac{D \dimpsi^2}{\lamminst} \cdot \log \frac{N}{\delta} \right \} \right )
\end{align}
episodes. 

By \Cref{lem:lammin_concentration}, if $\lammin(\bSigma_j) \ge 12544 D \dimpsi \log \frac{2N(2 + 32  T_j)}{\delta}$ and we rerun all policies in $\Pi_j$, then we will collect data $\bSigtil$ such that $\lammin(\bSigtil) \ge \frac{1}{2} \lammin(\bSigma_j)$, with probability at least $1-\delta/2N$. As the if statement on \Cref{line:lammin_if} will only be true once this is met, it follows that, with probability at least $1-\delta/2N$, rerunning all policies in $\Pi_j$ once, we will collect data $\bSigma$ which satisfies
\begin{align*}
\lammin(\bSigma) \ge \frac{1}{2} \lammin(\bSigma_j) \ge  6272 D \dimpsi \log \frac{2N(2 + 32  T_j)}{\delta} \ge 6272 D \dimpsi \log \frac{68N}{\delta}.
\end{align*}
The lower bound on $\lammin( \sum_{\pi \in \Pi} \bGamma_\pi)$ follows analogously from \Cref{lem:lammin_concentration}.

The result then follows noting that the failure probability of running \expdesign is at most
\begin{align*}
\sum_{j=1}^\infty \frac{\delta}{4 j^2} \le \delta / 2.
\end{align*}
\end{proof}

\subsubsection{Supporting Lemmas}

\begin{lemma}\label{lem:min_eig_suff}
Under \Cref{asm:regret_min_app,asm:mat_psi,asm:full_rank_cov_app}, consider running \expdesign on the objective
\begin{align*}
\Phi(\bGamma) = \tr (( \bGamma + \lambda \cdot I)^{-1})
\end{align*}
with $N = \lceil 2^{i/3} \rceil - 1$ and $K = \lceil 2^{2i/3} \rceil$, 
for some $\lambda > 0$ and $i$. Let $T := (N+1) K$. Then if $\lambda \le \frac{\lamminst}{4\dimpsi}$ and 
\begin{align}\label{eq:min_eig_K_cond}
T^{1/3} \ge \widetilde{\Omega} \left ( \left ( D \lambda^{-2} (D^3 \lambda^{-1}  + \CR \cdot \log^{\pR} \frac{1}{\delta})  \right ) \cdot \frac{\lamminst }{\dimpsi} \right ),
\end{align}
with probability at least $1-\delta$,
\begin{align*}
\lammin \left ( \sum_{t=1}^T \bpsi(\traj_t) \right ) \ge \frac{\lamminst}{4\dimpsi} \cdot T.
\end{align*}
\end{lemma}
\begin{proof}
Applying \Cref{lem:weighted_Aopt} with $\cH = I$ and $\bGamma_0 = \lambda \cdot I$, we have that, with probability at least $1-\delta$:
\begin{align*}
\tr  \left ( \left ( \sum_{t=1}^T \bpsi(\traj_t) + T \lambda \cdot I \right )^{-1} \right ) & \le \frac{1}{T} \cdot \min_{\bGamma \in \bOmpsi} \tr((\bGamma + \lambda \cdot I)^{-1}) + \frac{8D^4 \lambda^{-3}}{T (N+1)} + \frac{8 D \lambda^{-2} ( \log^{1/2} \frac{T}{\delta} + \CR \log^{\pR} \frac{T}{\delta} ) }{T \sqrt{K}} \\
& \le \frac{1}{T} \cdot \min_{\bGamma \in \bOmpsi} \tr((\bGamma + \lambda \cdot I)^{-1}) + \frac{24D^4 \lambda^{-3}}{T^{4/3}} + \frac{24 D \lambda^{-2} ( \log^{1/2} \frac{T}{\delta} + \CR \log^{\pR} \frac{T}{\delta} ) }{T^{4/3}}
\end{align*}
where the second inequality follows since
\begin{align*}
T = \lceil 2^{i/3} \rceil \lceil 2^{2i/3} \rceil \le 4 \cdot 2^i 
\end{align*}
which implies $N+1 = \lceil 2^{i/3} \rceil \ge T^{1/3}/4^{1/3}$ and $K = \lceil 2^{2i/3} \rceil \ge T^{2/3} / 4^{2/3}$. 
If $T$ satisfies \eqref{eq:min_eig_K_cond}, then we can bound 
\begin{align*}
\frac{24D^4 \lambda^{-3}}{T^{4/3}} + \frac{24 D \lambda^{-2} ( \log^{1/2} \frac{T}{\delta} + \CR \log^{\pR} \frac{T}{\delta} ) }{T^{4/3}} \le \frac{1}{T} \cdot \frac{\dimpsi}{\lamminst }.
\end{align*}
Furthermore, under \Cref{asm:full_rank_cov} there exists some $\bGamma \in \bOmpsi$ such that $\bGamma \succeq \lamminst \cdot I$, 
so we can upper bound
\begin{align*}
\min_{\bGamma \in \bOmpsi} \tr((\bGamma + \lambda \cdot I)^{-1}) \le \frac{\dimpsi}{\lamminst }
\end{align*}
and we can lower bound
\begin{align*}
\tr  \left ( \left ( \sum_{t=1}^T \bpsi(\traj_t) + T \lambda \cdot I \right )^{-1} \right ) \ge \frac{1}{\lammin ( \sum_{t=1}^T \bpsi(\traj_t) ) + T \lambda}.
\end{align*}
Thus,
\begin{align*}
& \frac{1}{\lammin ( \sum_{t=1}^T \bpsi(\traj_t) ) + T \lambda} \le \frac{1}{T} \cdot  \frac{2\dimpsi}{\lamminst}  \implies \lammin \left ( \sum_{t=1}^T \bpsi(\traj_t) \right ) \ge \frac{T \lamminst}{2\dimpsi} - T \lambda.
\end{align*}
It follows that if $\lambda \le \frac{\lamminst}{4\dimpsi}$, then we have
\begin{align*}
\lammin \left ( \sum_{t=1}^T \bpsi(\traj_t) \right ) \ge T \cdot \frac{\lamminst}{4 \dimpsi}
\end{align*}
which proves the result. 
\end{proof}

\begin{lemma}\label{lem:lammin_concentration}
Consider running some policies $(\pi_\tau)_{\tau = 1}^T$, for $\pi_\tau$ $\cF_{\tau-1}$-measurable, and collecting covariance $\bSigma_T = \sum_{t=1}^T \bpsi(\traj_t)$. Then under \Cref{asm:mat_psi}, as long as
\begin{align*}
\lammin(\bSigma_T) \ge 12544 D \dimpsi \log \frac{2 + 32 T}{\delta}
\end{align*}
with probability at least $1-\delta$, if we rerun each $(\pi_\tau)_{\tau = 1}^T$, we will collect features $\bSigtil_T$ such that
\begin{align*}
\lammin(\bSigtil_T) \ge \frac{1}{2} \lammin(\bSigma_T). 
\end{align*}
Furthermore,
\begin{align*}
\lammin \left ( \sum_{\tau = 1}^T \bGamma_{\pi_\tau} \right ) \ge \frac{1}{2} \lammin(\bSigma_T).
\end{align*}
\end{lemma}
\begin{proof}
 This follows from applying Lemma D.7 of \cite{wagenmaker2022instance} to the matrix $\frac{1}{D} \bSigma_T$. Note that while \cite{wagenmaker2022instance} considers the setting of linear MDPs, the proof of Lemma D.7 of \cite{wagenmaker2022instance} does not make use of the linear MDP assumption, and the proof therefore extends immediately to our setting. Furthermore, though it is not explicitly stated, the lower bound on $\lammin  ( \sum_{\tau = 1}^T \bGamma_{\pi_\tau}  )$ is also proved in Lemma D.7 of \cite{wagenmaker2022instance}.
\end{proof}

\subsection{Rerunning Policies}

\begin{algorithm}[h]
\begin{algorithmic}[1]
\State \textbf{input}: $\cH$, iterates bound $\Ntil$, confidence $\delta$, regret minimization algorithm $\regalg$, exploration policies $\Piexp$
\For{$i=1,2,3,\ldots$}
\State $N_i \leftarrow \lceil 2^{i/3} \rceil - 1, K_i \leftarrow \lceil 2^{2i/3} \rceil, T_i \leftarrow (N_i + 1)K_i, \delta_i \leftarrow \delta / 4 i^2$
\State $\Pimineig^i \leftarrow \mineigalg(\Ntil T_i,\delta_i,\regalg,\Piexp)$
\State Run policies in $\Pimineig^i$ $\lceil T_i / | \Pimineig^i | \rceil$ times, set $\bGamma_0^i$ to collected features \label{line:collect_more_cov}
\State $\Phi(\bGamma) \leftarrow \tr( \cH (\bGamma + T_i^{-1} \bGamma_0^i)^{-1})$
\State $\bLamfw^i, \Pifw^i \leftarrow \expdesign(\Phi,N_i,K_i,\delta,\regalg,\Piexp)$
\State \textbf{if}
\begin{align}
& \max_{j = 1,\ldots, i} | \Pimineig^j | \le T_i \label{eq:learn_policy_cond1} \\
&  \frac{16 D^4 \| \cH \|_\op \| (T_i^{-1} \bGamma_0)^{-1} \|_\op^3}{T_i (N_i + 1)} + \frac{16D \| \cH \|_\op \| (T_i^{-1} \bGamma_0)^{-1} \|_\op^2 (  \log^{1/2} \frac{4T_i}{\delta} +  \CR \log^{\pR} \frac{2T_i}{\delta})   }{T_i\sqrt{K_i}} \le \tr \left ( \cH \left (\bLamfw^i +  \bGamma_0^i \right )^{-1} \right ) \label{eq:learn_policy_cond2}  \\
& \tr(\cH) \cdot D \sqrt{2T_i} \sqrt{8 \dimpsi \log(1 + 8\sqrt{2T_i}) +8 \log 1/\delta} \cdot \frac{2}{\lammin \left ( \bLamfw^i + \bGamma_0^i \right )^2} \le \tr \left ( \cH \left (  \bLamfw^i + \bGamma_0^i \right )^{-1} \right ) \label{eq:learn_policy_cond3} \\
& D \sqrt{2T_i} \sqrt{8 \dimpsi \log(1 + 8\sqrt{2T_i}) +8 \log 1/\delta} \le \frac{1}{2} \lammin \left ( \bLamfw^i + \bGamma_0^i \right ) \label{eq:learn_policy_cond4}
\end{align}
\State \textbf{then}
\State \quad $\bGamout \leftarrow \bLamfw^i + \bGamma_0^i, \Piout \leftarrow \Pimineig^i \cup (\cup_{j=1}^{\lceil T_i / | \Pimineig | \rceil} \Pifw^i )$ 
\State \quad \textbf{return} $\bGamout, \Piout$
\EndFor
\end{algorithmic}
\caption{Learn Minimizing Exploration Policies (\learnpolicies)}
\label{alg:learn_policies}
\end{algorithm}

In this section, we build on the analysis of the \expdesign algorithm to show that, not only do the features collected by \expdesign approximately minimize $\Phi$, but that, under certain conditions, if we rerun the policies that \expdesign ran to collect this data, we will collect a new set of features which also approximately minimizes $\Phi$.

In particular, we specialize this argument to objectives of the form $\Phi(\bGamma) = \tr(\cH \bGamma^{-1})$. \learnpolicies (\Cref{alg:learn_policies}) proceeds by first calling \mineigalg to collect full-rank data, using this data as a regularizer of $\Phi(\bGamma)$, and the running \expdesign on this objective. After meeting a certain termination criteria, it terminates, and returns the policies it has run over its operation.

\begin{lemma}\label{lem:fw_exp_good}
Let $\cEexp$ denote the event that, for all $i=1,2,3,\ldots$, the success event of \mineigalg and \expdesign occur, and
\begin{align*}
\lammin(\bGamma_0^i) \ge \lceil T_i / | \Pimineig^i | \rceil \cdot 6272 D \dimpsi \log \frac{68\Ntil}{\delta}.
\end{align*}
Then if \Cref{asm:regret_min_app,asm:mat_psi,asm:full_rank_cov_app} hold, $\Pr[\cEexp] \ge 1-\delta$.
\end{lemma}

\begin{lemma}\label{lem:fw_exploration_good_event}
Consider rerunning each policy in $\Piout$ $N \le \Ntil$ times, and let $\bGamtil$ denote the obtained features. Then, if \Cref{asm:regret_min_app,asm:mat_psi,asm:full_rank_cov_app} hold, with probability at least $1-3\delta$, on the event $\cEexp$:
\begin{align}\label{eq:fw_exp_conclusion1}
\left \| \bGamtil - N \cdot \tsum_{\pi \in \Piout} \bGamma_{\pi} \right \|_\op  \le \frac{\sqrt{ \Kout} \cdot \sqrt{8 \dimpsi \log(1 + 8 \sqrt{N \Kout}) + 8 \log 1/\delta}}{\sqrt{N} \cdot 6272  \dimpsi \log \frac{68\Ntil}{\delta}} \cdot \lammin \left ( N \cdot \tsum_{\pi \in \Piout} \bGamma_{\pi} \right ), 
\end{align}
\begin{align}\label{eq:fw_exp_conclusion2}
N \cdot 6272 D \dimpsi \log \frac{68\Ntil}{\delta}  \le \min \left \{ \lammin \left ( N \cdot \tsum_{\pi \in \Piout} \bGamma_{\pi} \right ), \lammin(\bGamtil) \right \} ,
\end{align}
and
\begin{align}\label{eq:fw_exp_conclusion3}
\tr \left ( \cH \left (  \tsum_{\pi \in \Piout} \bGamma_\pi \right )^{-1} \right ) & \le \frac{12}{\Kout} \cdot  \min_{\bGamma \in \bOmega} \tr \left ( \cH \bGamma^{-1} \right ) 
\end{align}
for $\Tout := |\Piout|$.
\end{lemma}

\begin{lemma}\label{lem:Kout_bound}
On the event $\cEexp$, under \Cref{asm:regret_min_app,asm:mat_psi,asm:full_rank_cov_app}, we can bound
\begin{align*}
\Kout \le \poly \left ( \dimpsi, \frac{1}{\lamminst}, D, \CR, \log^{\pR} \frac{\Ntil}{\delta} \right ).
\end{align*}
Furthermore, the total number of episodes collected by \Cref{alg:learn_policies} is bounded by $(16 + 2 \log(\Kout)) \cdot \Kout$. 
\end{lemma}

\subsubsection{Supporting Lemmas and Proofs}
\begin{lemma}\label{lem:objective_bounds}
Under \Cref{asm:mat_psi}, for any $\bGamma = \Exp_{\traj \sim \omega}[\bpsi(\traj)]$ and $\cH \succeq 0$ we can bound
\begin{align*}
\tr(\cH \bGamma^{-1}) \ge D^{-1} \cdot \tr(\cH). 
\end{align*}
\end{lemma}
\begin{proof}
By Von Neumann's Trace Inequality we can lower bound
\begin{align*}
\tr(\cH \bGamma^{-1}) \ge \lammin ( \bGamma^{-1} ) \cdot \tr(\cH) = \| \bGamma \|_\op^{-1} \cdot \tr(\cH).
\end{align*}
By our assumption that $\tr( \bpsi(\traj) )  \le D$, we can bound $\| \bGamma \|_\op \le D$, which proves the result.
\end{proof}

\begin{lemma}\label{lem:cov_concentration}
Assume $\tr (\bpsi(\traj) )\le D$ for all $\traj$.
Let $\bGamma_K$ denote the time-normalized features obtained by playing policies $\{ \pi_k \}_{k=1}^K$, where $\pi_k$ is $\cF_{k-1}$-measurable. Then, with probability at least $1-\delta$,
\begin{align*}
\left \| \frac{1}{K} \sum_{k=1}^K \bGamma_{\pi_k} - \bGamma_K \right \|_\op \le D \sqrt{\frac{8 \dimpsi \log(1 + 8\sqrt{K}) +8 \log 1/\delta}{K}}.
\end{align*}
\end{lemma}
\begin{proof}
This follows from an argument identical to the proof of Lemma C.4 of \cite{wagenmaker2022instance}.
While \cite{wagenmaker2022instance} considers the setting of linear MDPs, we note that the proof of Lemma C.4 of \cite{wagenmaker2022instance} nowhere relies on the linear MDP assumption. The result stated here then follows identically as Lemma C.4 of \cite{wagenmaker2022instance}, after normalizing $\bpsi(\traj)$ by $D$. 
\end{proof}

\begin{proof}[Proof of \Cref{lem:fw_exp_good}]
By \Cref{lem:min_eig_policies}, the failure probability of running \mineigalg at round $i$ is $\delta_i = \delta / 8 i^2$, and by \Cref{lem:weighted_Aopt} the failure probability of \expdesign at round $i$ is also bounded by $\delta_i = \delta / 8 i^2$. It follows that the total failure probability of running \mineigalg and \expdesign is bounded by
\begin{align*}
\sum_{i=1}^\infty \frac{2 \cdot \delta}{8 i^2} \le \frac{\delta}{2}.
\end{align*}
Furthermore, by \Cref{lem:min_eig_policies}, we have that rerunning all policies in $\Pimineig^i$, we will obtain features $\bGamma$ satisfying, with probability at least $1 - \delta_i / \Ntil T_i$:
\begin{align*}
\lammin(\bGamma) \ge 6272 D\dimpsi \log \frac{68\Ntil}{\delta_i}.
\end{align*}
Repeating this $\lceil T_i / | \Pimineig^i | \rceil$ times and union bounding, we have that 
\begin{align*}
\lammin(\bGamma_0^i) \ge \lceil T_i / | \Pimineig^i | \rceil \cdot 6272 D \dimpsi \log \frac{68\Ntil}{\delta_i}
\end{align*}
with probability at least $1 - \delta/\Ntil T_i \cdot \lceil T_i / | \Pimineig^i | \rceil \ge 1 - \delta/\Ntil$. 
\end{proof}

\begin{proof}[Proof of \Cref{lem:fw_exploration_good_event}]
Let $\Pimineig,\bGamma_0,\Pifw,\bLamfw$ denote the policies and features obtained on the round at which \Cref{alg:learn_policies} terminates. Let $\Kfw =  |\Pifw|$ denote the number of episodes of \expdesign on the terminating round, and $\Nfw,\Kfwb$ the corresponding values of $N_i$ and $K_i$. Throughout the proof we make use of the fact that at termination of \Cref{alg:learn_policies}, all of \eqref{eq:learn_policy_cond1}-\eqref{eq:learn_policy_cond4} are met. 

\paragraph{Proof of \eqref{eq:fw_exp_conclusion1} and \eqref{eq:fw_exp_conclusion2}.} 
By \Cref{lem:cov_concentration}  we have that, with probability at least $1-\delta$:
\begin{align*}
\left \| \bGamtil - N \cdot \tsum_{\pi \in \Piout} \bGamma_{\pi} \right \|_\op \le D \sqrt{N \Kout} \cdot \sqrt{8 \dimpsi \log(1 + 8 \sqrt{N \Kout}) + 8 \log 1/\delta}.
\end{align*}
On $\cEexp$, by \Cref{lem:min_eig_policies}, we can bound
\begin{align*}
|\Pimineig| \le \poly \left ( \dimpsi, \frac{1}{\lamminst}, D, \CR, \log^{\pR} \frac{\Ntil \Kout}{\delta} \right )
\end{align*}
and, furthermore, we can lower bound
\begin{align*}
\lammin \left ( \tsum_{\pi \in \Pimineig} \bGamma_\pi \right ) \ge 6272 D \dimpsi \log \frac{68 \Ntil}{\delta}.
\end{align*}
Since $\Pimineig \subseteq \Piout$, it follows that
\begin{align*}
\lammin \left ( N \cdot \tsum_{\pi \in \Piout} \bGamma_{\pi} \right ) \ge N \cdot 6272 D\dimpsi \log \frac{68 \Ntil}{\delta}.
\end{align*}
Combining these, we therefore have that, with probability at least $1-\delta$:
\begin{align*}
\left \| \bGamtil - N \cdot \tsum_{\pi \in \Piout} \bGamma_{\pi} \right \|_\op \le \frac{\sqrt{N \Kout} \cdot \sqrt{8 \dimpsi \log(1 + 8 \sqrt{N \Kout}) + 8 \log 1/\delta}}{N \cdot 6272  \dimpsi \log \frac{68 \Ntil}{\delta}} \cdot \lammin \left ( N \cdot \tsum_{\pi \in \Piout} \bGamma_{\pi} \right ).
\end{align*}
In addition, also by \Cref{lem:min_eig_policies}, we have that with probability at least $1- \delta/\Ntil \Kout \cdot N \ge 1 - \delta$, that
\begin{align*}
\lammin(\bGamtil) \ge N \cdot 6272 D \dimpsi \log \frac{68 \Ntil}{\delta}.
\end{align*}

\paragraph{Proof of \eqref{eq:fw_exp_conclusion3}.}
By \Cref{lem:weighted_Aopt}, on $\cEexp$ we have that:
\begin{align*}
\Phi(\bLamfw) & = \tr \left ( \cH \left (\bLamfw +  \bGamma_0 \right )^{-1} \right )  \le \frac{  \min_{\bGamma \in \bOmega} \tr \left ( \cH (\bGamma + \Kfw^{-1} \bGamma_0)^{-1} \right )}{\Kfw} \\
& \qquad \qquad +   \frac{8D^4 \| \cH \|_\op \| (\Kfw^{-1} \bGamma_0)^{-1} \|_\op^3}{\Tfw (\Nfw + 1)} + \frac{8D \| \cH \|_\op \| (\Kfw^{-1} \bGamma_0)^{-1} \|_\op^2 (  \log^{1/2} \frac{4\Kfw}{\delta} +  \CR \log^{\pR} \frac{2\Kfw}{\delta})   }{\Kfw \sqrt{\Kfwb}}  .
\end{align*}
Since $\Kfw$ satisfies \eqref{eq:learn_policy_cond2}, we can bound
\begin{align*}
 \frac{8D^4 \| \cH \|_\op \| (\Kfw^{-1} \bGamma_0)^{-1} \|_\op^3}{\Tfw (\Nfw + 1)} + \frac{8D \| \cH \|_\op \| (\Kfw^{-1} \bGamma_0)^{-1} \|_\op^2 (  \log^{1/2} \frac{4\Kfw}{\delta} +  \CR \log^{\pR} \frac{2\Kfw}{\delta})   }{\Kfw \sqrt{\Kfwb}}  \le \frac{1}{2} \Phi(\bLamfw).
\end{align*}
It follows that
\begin{align}
& \Phi(\bLamfw) \le \frac{  \min_{\bGamma \in \bOmega} \tr \left ( \cH (\bGamma+ \Kfw^{-1} \bGamma_0)^{-1} \right )}{\Kfw} + \frac{1}{2} \Phi(\bLamfw) \nonumber \\
\implies & \Phi(\bLamfw) \le 2 \cdot \frac{  \min_{\bGamma \in \bOmega} \tr \left ( \cH (\bGamma + \Kfw^{-1} \bGamma_0)^{-1} \right )}{\Kfw}. \label{eq:exp_design_pf_eq1}
\end{align}
By \Cref{lem:cov_concentration}, we have that, with probability at least $1-\delta$:
\begin{align*}
\left \| \tsum_{\pi \in \Piout} \bGamma_{\pi} -  (\bLamfw + \bGamma_0) \right \|_\op \le D \sqrt{\Kout} \sqrt{8 \dimpsi \log(1 + 8\sqrt{\Kout}) +8 \log 1/\delta}.
\end{align*}
Since \eqref{eq:learn_policy_cond4} is satisfied and $| \Pimineig^i | \le T_i $, we have
\begin{align*}
D \sqrt{\Kout} \sqrt{8 \dimpsi \log(1 + 8\sqrt{\Kout}) +8 \log 1/\delta} \le \frac{1}{2} \lammin \left ( \bLamfw + \bGamma_0 \right ) .
\end{align*}
By \Cref{lem:gaminv_diff} it follows that
\begin{align*}
\left \| \left ( \tsum_{\pi \in \Piout} \bGamma_{\pi} \right )^{-1} - \left ( \bLamfw + \bGamma_0 \right )^{-1} \right \|_\op \le  D \sqrt{\Kout} \sqrt{8 \dimpsi \log(1 + 8\sqrt{\Kout}) +8 \log 1/\delta} \cdot \frac{2}{\lammin \left ( \bLamfw + \bGamma_0 \right )^2}.
\end{align*}
This implies that
\begin{align*}
\tr \left ( \cH \left (  \tsum_{\pi \in \Piout} \bGamma_{\pi} \right )^{-1} \right ) & \le \tr \left ( \cH \left (  \bLamfw + \bGamma_0 \right )^{-1} \right ) \\
& + \tr(\cH) \cdot D \sqrt{\Kout} \sqrt{8 \dimpsi \log(1 + 8\sqrt{\Kout}) +8 \log 1/\delta} \cdot \frac{2}{\lammin \left ( \bLamfw + \bGamma_0 \right )^2} .
\end{align*}
Now if
\begin{align}\label{eq:exp_design_pf_eq2}
\tr(\cH) \cdot D \sqrt{\Kout} \sqrt{8 \dimpsi \log(1 + 8\sqrt{\Kout}) +8 \log 1/\delta} \cdot \frac{2}{\lammin \left ( \bLamfw + \bGamma_0 \right )^2} \le \tr \left ( \cH \left (  \bLamfw + \bGamma_0 \right )^{-1} \right ),
\end{align}
we can bound this all by
\begin{align*}
& \le 2\tr \left ( \cH \left (  \bLamfw + \bGamma_0 \right )^{-1} \right )  \le \frac{4}{\Kfw} \cdot  \min_{\bGamma \in \bOmega} \tr \left ( \cH (\bGamma + \Kfw^{-1} \bGamma_0)^{-1} \right )
\end{align*}
where the last inequality follows from \eqref{eq:exp_design_pf_eq1}. However, note that  \eqref{eq:exp_design_pf_eq2} since \eqref{eq:learn_policy_cond3} holds. Finally, note that
$$\Kout = \Kfw + \lceil \Kfw / |\Pimineig| \rceil |\Pimineig| \le 2\Kfw + |\Pimineig| \le 3\Kfw $$
where the last inequality follows since \eqref{eq:learn_policy_cond1} holds. We can therefore upper bound $\frac{4}{\Kfw} \le \frac{12}{\Kout}$. Putting this together proves the result.

\end{proof}

\begin{proof}[Proof of \Cref{lem:Kout_bound}]
To bound $\Kout$, it suffices to show that \eqref{eq:learn_policy_cond1}-\eqref{eq:learn_policy_cond4} are satisfied for sufficiently large $T_i$. 

On $\cEexp$, by \Cref{lem:min_eig_policies}, we can bound
\begin{align*}
|\Pimineig^i| \le \poly \left ( \dimpsi, \frac{1}{\lamminst}, D, \CR, \log^{\pR} \frac{\Ntil T_i}{\delta} \right ),
\end{align*}
so to ensure \eqref{eq:learn_policy_cond1} is met it suffices that 
\begin{align*}
T_i \ge  \poly \left ( \dimpsi, \frac{1}{\lamminst}, D, \CR, \log^{\pR} \frac{\Ntil}{\delta} \right ).
\end{align*}

On $\cEexp$, we have
\begin{align*}
\lammin(\bLamfw^i + \bGamma_0^i) \ge \lammin(\bGamma_0^i) \ge  \lceil T_i / | \Pimineig^i | \rceil \cdot 6272 D \dimpsi \log \frac{68 \Ntil}{\delta}.
\end{align*}
Which also implies
\begin{align*}
\| (T_i^{-1} \bGamma_0^i)^{-1} \|_\op = \frac{T_i}{\lammin(\bGamma_0^i)} \le \frac{T_i}{\lceil T_i / | \Pimineig^i | \rceil} \cdot \frac{1}{6272 D \dimpsi \log \frac{68 \Ntil}{\delta}} \le \frac{| \Pimineig^i |}{6272 D \dimpsi \log \frac{68 \Ntil}{\delta}}
\end{align*}
Furthermore, by \Cref{lem:objective_bounds} we can lower bound
\begin{align*}
\tr \left ( \cH \left (\bLamfw^i +  \bGamma_0^i \right )^{-1} \right ) \ge  \frac{\tr(\cH)}{D(T_i + \lceil T_i / | \Pimineig^i | \rceil |\Pimineig^i|)} \ge \frac{\tr(\cH)}{3 D T_i}.
\end{align*}
Combining these and using that $N_i = \cO(T_i^{1/3})$ and $K_i = \cO(T_i^{2/3})$, it is easy to see that \eqref{eq:learn_policy_cond2}-\eqref{eq:learn_policy_cond4} will be met once 
\begin{align*}
T_i \ge  \poly \left ( \dimpsi, \frac{1}{\lamminst}, D, \CR, \log^{\pR} \frac{\Ntil}{\delta} \right ).
\end{align*}
The bound on $\Kout$ then follows since $T_i = \lceil 2^{i/3} \rceil \lceil 2^{2i/3} \rceil \in [2^i, 4 \cdot 2^i]$, so it can be at most a constant larger than the sufficient condition before terminating. 

Let $\ist$ denote the round that \Cref{alg:learn_policies} terminates on. Note that at round $i$, \mineigalg runs for at most $|\Pimineig^i|$, \expdesign runs for at most $T_i$ episodes, and we run for an additional $\lceil T_i / | \Pimineig^i | \rceil \cdot | \Pimineig^i |$ episodes on \Cref{line:collect_more_cov}. In total, then, the number of episodes \Cref{alg:learn_policies} runs for is bounded by
\begin{align*}
\sum_{i = 1}^{\ist} (T_i +  |\Pimineig^i| + \lceil T_i / | \Pimineig^i | \rceil \cdot | \Pimineig^i |) & \le2 \sum_{i=1}^{\ist} ( T_i + |\Pimineig^i|) \\
& \le 16 T_{\ist} + 2 \sum_{i=1}^{\ist} |\Pimineig^i|
\end{align*}
where the last inequality follows since $T_i \in [2^i, 4 \cdot 2^i]$. Now note that, since \Cref{alg:learn_policies} only terminates once \eqref{eq:learn_policy_cond1} is met, we will have $\max_{j = 1,\ldots, \ist} | \Pimineig^j | \le T_{\ist}$. This implies that $2 \sum_{i=1}^{\ist} |\Pimineig^i| \le 2 \ist T_{\ist} \le 2 \log (T_{\ist}) \cdot T_{\ist}$. Bounding $T_{\ist} \le \Kout$ gives the result.

\end{proof}


\section{Smooth Nonlinear Systems}\label{sec:smooth_system}

In this section we restrict to the nonlinear regulator system of \eqref{eq:system}. Our goal will be to show that, under our assumptions, the nonlinear regulator system exhibits certain smooth behavior. As we have assumed
\begin{align*}
\Picon = \{ \Ktheta \ : \ \btheta \in \R^{\dimtheta} \},
\end{align*}
it will be convenient to define $\cJ(\btheta;A) := \cJ(\Ktheta;A)$ and $\bthetast(A) = \bthetast$. 
For the remainder of this section, we will typically use $\btheta$ in place of $\Ktheta$. 
In addition, when considering radius terms such as $\rtheta(\Ast)$ and $\rcost(\Ast)$, to simplify results we assume that $\rtheta(\Ast) \le 1$ and $\rcost(\Ast) \le 1$. Note that this does not change the validity of the result since, for example, if a result holds with $A \in \cB_{\fro}(\Ast;r)$ for some $r > \rtheta(\Ast)$, it also holds for $A \in \cB_{\fro}(\Ast;\rtheta(\Ast))$.
Throughout this section, we let $\nabla_x f(x)[\Delta]$ refer to the directional gradient of $f(x)$ in direction $\Delta$.

We first have the following result, which shows that under our assumptions, the controller loss is differentiable.

\begin{lemma}\label{lem:knr_loss_diff}
Under \Cref{asm:smooth_phi,asm:smooth_controller,asm:bounded_features,asm:bounded_cost}, for any $A$ satisfying $A \in \cB_{\fro}(\Ast;\rtheta(\Ast))$, the controller loss $\cJ(\btheta;A)$ is four-times differentiable in $\btheta$ and $A$. Furthermore, we can bound
\begin{align*}
& \| \nabla_A^{(i)} \nabla_{\btheta}^{(j)} \cJ(\btheta;A) \|_\op \le \poly(\| A \|_\op, \Bphi, \Lphi, \Lzeta, \Lcost, \sigw^{-1}, H, \dimx)
\end{align*}
for $i,j \in \{0,1,2,3,4\}$ satisfying $1 \le i + j \le 3$.
\end{lemma}

In this section, we generalize \Cref{asm:unique_min} to the following.
\begin{asm}\label{asm:smooth_thetast}
We assume there exists some $\rtheta(\Ast) > 0$ such that, for all $A \in \cB_{\fro}(\Ast;\rtheta(\Ast))$, $\bthetast(A)$ satisfies:
\begin{itemize}
\item $\nabla_{\btheta} \cJ(\btheta;A)|_{\btheta = \bthetast(A)} = 0$, 
\item $\bthetast(A)$ is three-times differentiable in $A$, and we can bound $\| \nabla^{(i)}_A \bthetast(A) \|_\op \le \LKst$ for some $\LKst > 0$ and $i \in \{1,2,3 \}$. 
\end{itemize}
\end{asm}
The first condition requires that $\bthetast(A)$ corresponds to a stationary point of the loss. This will be met, for example, by choosing $\bthetast(A)$ to be a minima (local or global) of $\cJ(\btheta;A)$. It is not obvious, however, that the first and second condition can be simultaneously satisfied. In the following we show that, assuming $\nabla_{\btheta}^2 \cJ(\btheta;\Ast)|_{\btheta = \bthetast(\Ast)}$ is full-rank (which will be the case, for example, when $\bthetast(\Ast)$ is a strict local minimum of $\cJ(\pi^{\btheta};\Ast)$), there always exists some $\bthetast(A)$ satisfying both conditions of \Cref{asm:smooth_thetast}, with $\LKst$ scaling polynomially in problem parameters, and $\rtheta(\Ast)$ scaling inverse polynomially in problem parameters. 
Note that this definition of $\pist(A)$ is general enough to capture settings where the global minimum of $\cJ(\pi;A)$ cannot be efficiently computed---it suffices to take $\pist(A)$ a local minimum of the loss.

\begin{prop}\label{prop:smooth_theta}
Assume that \Cref{asm:smooth_phi,asm:smooth_controller,asm:bounded_features,asm:bounded_cost} hold and that $\lammin(\nabla_{\btheta}^2 \cJ(\btheta;\Ast)|_{\btheta = \bthetast(\Ast)}) > 0$. Let $\rtheta(\Ast) > 0$ be some value satisfying
\begin{align*}
\rtheta(\Ast) = \min \left \{ \rcost(\Ast), \poly \left ( \frac{1}{\lammin(\nabla_{\btheta}^2 \cJ(\btheta;\Ast)|_{\btheta = \bthetast(\Ast)})},  \| \Ast \|_\op, \Bphi, \Lphi, \Lzeta, \Lcost, \sigw^{-1}, H, \dimx \right )^{-1} \right \}.
\end{align*}
Then there exists some function $\bthetast(A)$ such that, for all $A \in \cB_{\fro}(\Ast;\rtheta(\Ast))$:
\begin{itemize}
\item $\nabla_{\btheta} \cJ(\btheta;A)|_{\btheta = \bthetast(A)} = 0$, 
\item $\bthetast(A)$ is three-times differentiable in $A$,
\end{itemize}
and it suffices that we take
\begin{align*}
\LKst = \poly \left ( \frac{1}{\lammin(\nabla_{\btheta}^2 \cJ(\btheta;\Ast)|_{\btheta = \bthetast(\Ast)})},  \| \Ast \|_\op, \Bphi, \Lphi, \Lzeta, \Lcost, \sigw^{-1}, H, \dimx \right ).
\end{align*}
\end{prop}

While \Cref{prop:smooth_theta} shows that there exists some $\bthetast(A)$ satisfying \Cref{asm:smooth_thetast}, it does not directly give a recipe for constructing such a map. The following result shows that under a mild additional assumption, the minimizer of the loss satisfies \Cref{asm:smooth_thetast}.

\begin{proposition}\label{prop:theta_global_min}
Let 
\begin{align*}
\bthetast(A) := \argmin_{\btheta \in \R^{\dimtheta}} \cJ(\btheta;A).
\end{align*}
Then under  \Cref{asm:smooth_phi,asm:smooth_controller,asm:bounded_features,asm:bounded_cost,asm:unique_min}, there exists some $\rtheta(\Ast) > 0$ and $\LKst < \infty$ such that $\bthetast(A)$ satisfies \Cref{asm:smooth_thetast}.
\end{proposition}

The scaling of $\LKst$ in \Cref{prop:theta_global_min} can be shown to match that of \Cref{prop:smooth_theta},
but in general $\rtheta(\Ast)$ could be smaller than the value of $\rtheta(\Ast)$ given in \Cref{prop:smooth_theta}. In particular, in the setting of \Cref{prop:theta_global_min}, we can only show that $\rtheta(\Ast)$ scales with $\min_{\btheta \not \in \cB_{2}(\bthetast(\Ast);r)} \cJ(\btheta;\Ast) - \cJ(\bthetast(\Ast);\Ast)$ for some $r > 0$ which scales inverse polynomially in problem parameters. While we can show that $\cJ(\btheta;\Ast) - \cJ(\bthetast(\Ast);\Ast)$ scales inverse polynomially in problem parameters, including in $\lammin(\nabla_{\btheta}^2 \cJ(\btheta;\Ast)|_{\btheta = \bthetast(\Ast)} )$, for $\btheta$ approximately a distance of $r$ from $\bthetast(\Ast)$, it is possible $\cJ(\btheta;\Ast)$ has some local minimizer $\btheta'$ arbitrarily far away from $\bthetast(\Ast)$, such that $\cJ(\btheta';\Ast)$ and $\cJ(\bthetast(\Ast);\Ast)$ are arbitrarily close, in which case $\Delst$, and therefore $\rtheta(\Ast)$, could be arbitrarily small. The failure mode here is that, while $\bthetast(\Ast)$ may be the global minimum of $\cJ(\btheta;\Ast)$, for $A$ arbitrarily close to $\Ast$, the global minimum of $\cJ(\btheta;A)$ could instead be near $\btheta'$, which would render the map $\bthetast(A)$ discontinuous. 

By making further assumptions on $\cJ(\btheta;\Ast)$ which exclude this case, we can obtain a value of $\rtheta(\Ast)$ scaling similarly to in \Cref{prop:smooth_theta}.
For example, in the following, we show that under the assumption that $\cJ(\btheta;A)$ is convex, this holds.

\begin{prop}\label{prop:thetast_smooth_convex}
Assume that there exists some $r_{\mathrm{conv}}(\Ast)>0$ such that, for all $A \in \cB_{\fro}(\Ast;r_{\mathrm{conv}}(\Ast))$, $\cJ(\btheta;A)$ is convex in $\btheta$, and set
\begin{align*}
\bthetast(A) = \argmin_{\btheta \in \R^{\dimtheta}} \cJ(\btheta;A).
\end{align*}
Then we have that $\bthetast$ satisfies \Cref{asm:smooth_thetast} with 
\begin{align*}
\rtheta(\Ast) & = \min \bigg \{ r_{\mathrm{conv}}(\Ast), \rcost(\Ast), \\
& \poly \left ( \frac{1}{\lammin(\nabla_{\btheta}^2 \cJ(\btheta;\Ast)|_{\btheta = \bthetast(\Ast)})}, \| \Ast \|_\op, \Bphi, \Lphi, \Lzeta, \Lcost, \sigw^{-1}, H, \dimx \right )^{-1} \bigg \}
\end{align*}
and it suffices that we take
\begin{align*}
\LKst = \poly \left ( \frac{1}{\lammin(\nabla_{\btheta}^2 \cJ(\btheta;\Ast)|_{\btheta = \bthetast(\Ast)})}, \| \Ast \|_\op, \Bphi, \Lphi, \Lzeta, \Lcost, \sigw^{-1}, H, \dimx \right ).
\end{align*}
\end{prop}
Note that, if $\cJ(\btheta;A)$ is $\mu$-strongly convex in $\btheta$ for all $A$ near $\Ast$, we can lower bound $\lammin(\nabla_{\btheta}^2 \cJ(\btheta;\Ast)|_{\btheta = \bthetast(\Ast)}) \ge \mu$.

\paragraph{Approximating the Controller Loss.}
In order to efficiently direct our exploration, it is convenient to derive a quadratic approximation to the controller loss. The following result shows that, under our assumptions, this is indeed possible. 

\begin{lemma}[Formal Version of \Cref{prop:quadratic_loss_approx}]\label{lem:quadratic_loss_approx}
Under  \Cref{asm:smooth_phi,asm:smooth_controller,asm:smooth_controller,asm:bounded_cost,asm:bounded_features,asm:smooth_thetast}, for $\Ahat \in \cB_{\fro}(\Ast; \min \{ \rcost(\Ast), \rtheta(\Ast) \} )$, we have
\begin{align*}
& \cJ(\bthetast(\Ahat); \Ast) - \cJ(\bthetast(\Ast); \Ast)  \le \vec(\Ahat - \Ast)^\top \cH(\Ast) \vec(\Ahat - \Ast) + M[\Ahat - \Ast,\Ahat-\Ast,\Ahat-\Ast] .
\end{align*}
for some tensor $M$ such that
\begin{align*}
\| M[\Ahat - \Ast,\Ahat-\Ast,\Ahat-\Ast] \|_\op \le \poly(\LKst, \| \Ast \|_\op, \Bphi, \Lphi, \Lzeta, \Lcost, \sigw^{-1}, H, \dimx)\cdot \| \Ahat - \Ast \|_\op^3. 
\end{align*}
\end{lemma}

In practice we do not know $\cH(\Ast)$ and must estimate it. The following result shows that the distance between $\cH(\Ast)$ and $\cH(\Ahat)$ can be bounded.

\begin{lemma}\label{lem:Hst_to_Hhat}
Under \Cref{asm:smooth_phi,asm:smooth_controller,asm:bounded_cost,asm:bounded_features,asm:smooth_thetast}, and if $\Ahat \in \cB_{\fro}(\Ast; \min \{ \rcost(\Ast), \rtheta(\Ast) \})$, we can bound
\begin{align*}
\| \cH(\Ast) - \cH(\Ahat) \|_\op \le  \poly(\LKst, \| \Ast \|_\op, \Bphi, \Lphi, \Lzeta, \Lcost, \sigw^{-1}, H, \dimx)\cdot \| \Ahat - \Ast \|_\op.
\end{align*}
\end{lemma}

\subsection{Proof of Smoothness of Nonlinear System}
We let $\fw(\cdot)$ denote the density of the noise (which, by assumption, is simply an isotropic Gaussian density). We let $\f{A}{\btheta}(\cdot)$ denote the density over trajectories induced by playing controller $\btheta$ on system $A$. We will overload notation somewhat and let $\f{A}{\btheta}(\bx_{h+1} \mid \traj_
{1:h})$ denote the density over $\bx_{h+1}$ induced by playing controller $\btheta$ given trajectory $\traj_{1:h}$. Note that $\f{A}{\btheta}(\bx_{h+1} \mid \traj_{1:h}) = \fw(\bx_{h+1} - A \bphi(\bx_h^{\traj}, \Kzeta_h(\traj_{1:h})))$ and
\begin{align*}
\f{A}{\btheta}(\traj) = \prod_{h=1}^H \f{A}{\btheta}(\bx_{h+1} \mid \traj_{1:h}).
\end{align*}
Throughout this section we let $\bx_h^{\traj}$ (resp. $\bu_h^{\traj}$) denote the state (resp. input) at step $h$ of trajectory $\traj$. Under our regularity assumptions (\Cref{asm:smooth_phi,asm:smooth_controller,asm:bounded_features,asm:bounded_cost}) and since the noise is Gaussian, we can swap derivatives and integrals, which we make use of throughout the following proofs.

\newcommand{\DelA}{\Delta^A}
\newcommand{\Deltheta}{\Delta^{\btheta}}
\newcommand{\bmt}{\bm{t}}
\newcommand{\bms}{\bm{s}}
\newcommand{\bmzer}{\bm{0}}
\newcommand{\bline}{\big |}

\begin{proof}[Proof of \Cref{lem:knr_loss_diff}]
Let $\cost(\traj)$ denote the cost of trajectory $\traj$. Then we have
\begin{align*}
\cJ(\btheta;A) = \int \cost(\traj)\f{A}{\btheta}(\traj)   \rmd \traj.
\end{align*}
Let $A_{\bmt} := A + t_1 \DelA_1 + t_2 \DelA_2 + t_3 \DelA_3$ and $\btheta_{\bms} := \btheta + s_1 \Deltheta_1 + s_2 \Deltheta_2 + s_3 \Deltheta_3$, for some $\DelA_i$ and $\Deltheta_j$, which we assume satisfy $\| \DelA_i \|_\op, \| \Deltheta_j \|_\op \le 1$. Rather than differentiating $\cJ(\btheta;A)$ with respect to $\btheta$ or $A$, we will differentiate $\cJ(\btheta_{\bms};A_{\bmt})$ with respect to some $x_1,x_2,x_3, x_4 \in \{ t_1,t_2,t_3,s_1,s_2,s_3 \}$. Note that, for example,
\begin{align*}
\frac{\rmd}{\rmd t_1} \cJ(\btheta_{\bms};A_{\bmt})|_{\bmt = \bms = 0} = \nabla_A \cJ(\btheta,A)[\DelA_1],
\end{align*}
i.e. the directional gradient of $\cJ(\btheta,A)$ with respect to $A$ in direction $\DelA_1$, and that this similarly holds for gradients with respect to other $t_i,s_j$, or higher-order derivatives. Thus, if we can show that $\cJ(\btheta_{\bms};A_{\bmt})$ is differentiable with respect to any $x_1,x_2,x_3, x_4 \in \{ t_1,t_2,t_3,s_1,s_2,s_3 \}$, and this holds for any choice of $\DelA_i,\Deltheta_j$, then we have that $\cJ(\btheta,A)$ is four-times differentiable with respect to $\btheta$ and $A$. Furthermore, we can bound the operator norm of $\nabla_A \cJ(\btheta,A)$, by bounding the value of $\frac{\rmd}{\rmd t_1} \cJ(\btheta_{\bms};A_{\bmt})|_{\bmt = \bms = 0}$ for all $\DelA_1$ satisfying $\| \DelA_1 \|_\op \le 1$ (and we can similarly bound the operator norm of the higher order derivatives of $\cJ(\btheta,A)$).

\paragraph{$\cJ(\btheta;A)$ is Differentiable.}
Let $x_1,x_2,x_3,x_4 \in \{ t_1,t_2,t_3,s_1,s_2,s_3 \}$. We have
\begin{align*}
\frac{\rmd}{\rmd x_1} \cJ(\btheta_{\bms};A_{\bmt})  & = \frac{\rmd}{\rmd x_1}  \int  \cost(\traj) \f{A_{\bmt}}{\btheta_{\bms}}(\traj) \rmd \traj  \\
& =  \int \frac{\f{A_{\bmt}}{\btheta_{\bms}}(\traj)}{\f{A_{\bmt}}{\btheta_{\bms}}(\traj)} \frac{\rmd}{\rmd x_1}  \f{A_{\bmt}}{\btheta_{\bms}}(\traj) \cost(\traj) \rmd \traj  \\
& = \int  \frac{\rmd}{\rmd x_1}  \log \f{A_{\bmt}}{\btheta_{\bms}}(\traj) \cdot \cost(\traj)\f{A_{\bmt}}{\btheta_{\bms}}(\traj)  \rmd \traj.
\end{align*}
Differentiating this gives
\begin{align*}
\frac{\rmd}{\rmd x_2} \frac{\rmd}{\rmd x_1} \cJ(\btheta_{\bms};A_{\bmt}) & = \frac{\rmd}{\rmd x_2} \int \frac{\rmd}{\rmd x_1}  \log \f{A_{\bmt}}{\btheta_{\bms}}(\traj) \cdot \cost(\traj)\f{A_{\bmt}}{\btheta_{\bms}}(\traj)  \rmd \traj \\
& =  \int \left ( \frac{\rmd}{\rmd x_1}  \log \f{A_{\bmt}}{\btheta_{\bms}}(\traj) \right ) \left ( \frac{\rmd}{\rmd x_2}  \log \f{A_{\bmt}}{\btheta_{\bms}}(\traj) \right )  \cdot \cost(\traj)\f{A_{\bmt}}{\btheta_{\bms}}(\traj)  \rmd \traj \\
& \qquad + \int \frac{\rmd}{\rmd x_2} \frac{\rmd}{\rmd x_1}  \log \f{A_{\bmt}}{\btheta_{\bms}}(\traj) \cdot \cost(\traj)\f{A_{\bmt}}{\btheta_{\bms}}(\traj)  \rmd \traj,
\end{align*}
and 
\begin{align*}
\frac{\rmd}{\rmd x_3} \frac{\rmd}{\rmd x_2} \frac{\rmd}{\rmd x_1} &\cJ(\btheta_{\bms};A_{\bmt}) = \int \left ( \frac{\rmd}{\rmd x_1}  \log \f{A_{\bmt}}{\btheta_{\bms}}(\traj) \right ) \left ( \frac{\rmd}{\rmd x_2}  \log \f{A_{\bmt}}{\btheta_{\bms}}(\traj) \right ) \left ( \frac{\rmd}{\rmd x_3}  \log \f{A_{\bmt}}{\btheta_{\bms}}(\traj) \right )  \cdot \cost(\traj)\f{A_{\bmt}}{\btheta_{\bms}}(\traj)  \rmd \traj \\
& + \int \left ( \frac{\rmd}{\rmd x_3} \frac{\rmd}{\rmd x_1}  \log \f{A_{\bmt}}{\btheta_{\bms}}(\traj) \right ) \left ( \frac{\rmd}{\rmd x_2}  \log \f{A_{\bmt}}{\btheta_{\bms}}(\traj) \right )  \cdot \cost(\traj)\f{A_{\bmt}}{\btheta_{\bms}}(\traj)  \rmd \traj \\
& + \int \left ( \frac{\rmd}{\rmd x_1}  \log \f{A_{\bmt}}{\btheta_{\bms}}(\traj) \right ) \left (  \frac{\rmd}{\rmd x_3} \frac{\rmd}{\rmd x_2}  \log \f{A_{\bmt}}{\btheta_{\bms}}(\traj) \right )  \cdot \cost(\traj)\f{A_{\bmt}}{\btheta_{\bms}}(\traj)  \rmd \traj \\
& + \int \frac{\rmd}{\rmd x_3} \frac{\rmd}{\rmd x_2} \frac{\rmd}{\rmd x_1}  \log \f{A_{\bmt}}{\btheta_{\bms}}(\traj) \cdot \cost(\traj)\f{A_{\bmt}}{\btheta_{\bms}}(\traj)  \rmd \traj \\
& + \int \left ( \frac{\rmd}{\rmd x_2} \frac{\rmd}{\rmd x_1}  \log \f{A_{\bmt}}{\btheta_{\bms}}(\traj) \right ) \left ( \frac{\rmd}{\rmd x_3}  \log \f{A_{\bmt}}{\btheta_{\bms}}(\traj) \right ) \cdot \cost(\traj)\f{A_{\bmt}}{\btheta_{\bms}}(\traj)  \rmd \traj.
\end{align*}
The fourth derivative of $\cJ(\btheta;A)$ can be similarly calculated by differentiating $\frac{\rmd}{\rmd x_3} \frac{\rmd}{\rmd x_2} \frac{\rmd}{\rmd x_1} \cJ(\btheta_{\bms};A_{\bmt})$; we omit it for brevity.
We have 
\begin{align*}
\log \f{A_{\bmt}}{\btheta_{\bms}}(\traj) & = \log \prod_{h=1}^H \f{A_{\bmt}}{\btheta_{\bms}}(\bx_{h+1}^{\traj} \mid \traj_{1:h}) \\
&  = \sum_{h=1}^H \log \fw(\bx_{h+1}^{\traj} - A_{\bmt} \bphi(\bx_{h}^{\traj}, \pi^{\btheta_{\bms}}(\traj_{1:h}))) \\
& = \sum_{h=1}^H - \frac{1}{2 \sigw^2} \| \bx_{h+1}^{\traj} - A_{\bmt} \bphi(\bx_{h}^{\traj}, \pi^{\btheta_{\bms}}(\traj_{1:h})) \|_2^2 + C
\end{align*}
for some $C$ which does not depend on $\bmt$ or $\bms$. Given that $\bphi(\bx,\bu)$ is four-times differentiable in $\bu$ and $\pi^{\btheta_{\bms}}_{h}(\traj_{1:h})$ is four-times differentiable in $\bx$ (which hold by \Cref{asm:smooth_phi} and \Cref{asm:smooth_controller}), it is clear that $\log \f{A_{\bmt}}{\btheta_{\bms}}(\traj)$ is four-times differentiable in $t_i$ or $s_i$, regardless of the choice of $\DelA_i$ or $\Deltheta_i$. This proves the first result.

\paragraph{Norm Bounds on Gradient.}
Note that 
\begin{align*}
\frac{\rmd}{\rmd t_i} \log \f{A_{\bmt}}{\btheta_{\bms}}(\traj) |_{\bmt = \bms = 0} & = \sum_{h=1}^H \frac{1}{\sigw^2} ( \bx_{h+1}^{\traj} - A\bphi(\bx_{h}^{\traj}, \Kzeta_{h}(\traj_{1:h})))^\top \cdot \DelA_i \bphi(\bx_{h}^{\traj}, \Kzeta_{h}(\traj_{1:h})), \\
\frac{\rmd}{\rmd s_i} \log \f{A_{\bmt}}{\btheta_{\bms} }(\traj)|_{\bmt = \bms =0} & = \sum_{h=1}^H \frac{1}{\sigw^2} (\bx_{h+1}^{\traj} - A\bphi(\bx_{h}^{\traj}, \Kzeta_{h}(\traj_{1:h})))^\top \cdot A \nabla_{\bu} \bphi(\bx_{h}^{\traj},\Kzeta_{h}(\traj_{1:h})) \cdot \nabla_{\btheta} \Kzeta_{h}(\traj_{1:h}) \cdot \Deltheta_i .
\end{align*}
Furthermore, differentiating these expressions further with respect to $t_j$ or $s_j$ will simply yield higher-order derivates of $\bphi(\bx,\bu)$ and $\Kzeta_h(\traj_{1:h})$. 
Using the norm bounds on the gradient of $\bphi(\bx,\bu)$ and $\Kzeta_h(\traj_{1:h})$ given in \Cref{asm:smooth_phi} and \Cref{asm:smooth_controller}, and the norm bound of $\bphi(\bx,\bu)$ given in \Cref{asm:bounded_features}, we can then bound
\begin{align*}
\| \nabla_A^{(i)} \nabla_{\btheta}^{(j)} \log \f{A}{\btheta}(\traj) \|_\op \le \poly(\| A \|_\op, \Bphi, \Lphi, \Lzeta, \sigw^{-1}) \cdot \sum_{h=1}^H ( 1 + \| \bx_{h+1}^{\traj} - A\bphi(\bx_{h}^{\traj}, \Kzeta_{h}(\traj_{1:h})) \|_2)
\end{align*}
for $i, j \in \{0, 1,2,3,4 \}$ satisfying $1 \le i + j \le 4$ (where we have used the fact noted above that, to bound the operator norm of $\nabla_A^{(i)} \nabla_{\btheta}^{(j)} \log \f{A}{\btheta}(\traj)$, it suffices to bound the directional gradient in every direction).
 It follows that we can bound
\begin{align*}
& \| \nabla_A^{(i)} \nabla_{\btheta}^{(j)} \cJ(\btheta;A) \|_\op \\
& \qquad \le \poly(\| A \|_\op,\Bphi, \Lphi, \Lzeta, \sigw^{-1}) \cdot \int  \left ( \sum_{h=1}^H ( 1 + \| \bx_{h+1}^{\traj} - A\bphi(\bx_{h}^{\traj}, \Kzeta_{h}(\traj_{1:h})) \|_2) \right)^4 \cdot \cost(\traj) \f{A}{\btheta}(\traj) \rmd \traj \\
& \qquad \overset{(a)}{\le} \poly(\| A \|_\op,\Bphi, \Lphi, \Lzeta, \sigw^{-1}) \cdot  \sqrt{\int \cost(\traj)^2 \f{A}{\btheta}(\traj) \rmd \traj} \\
& \qquad \qquad \cdot \sqrt{\int \left ( \sum_{h=1}^H ( 1 + \| \bx_{h+1}^{\traj} - A\bphi(\bx_{h}^{\traj}, \Kzeta_{h}(\traj_{1:h})) \|_2) \right)^8 \f{A}{\btheta}(\traj) \rmd \traj } \\
& \qquad \overset{(b)}{\le} \poly(\| A \|_\op,\Bphi, \Lphi, \Lzeta, \Lcost, \sigw^{-1}, H, \dimx)
\end{align*}
where $(a)$ follows from Cauchy-Schwarz, and $(b)$ follows from \Cref{lem:gaussian_bound} and \Cref{asm:bounded_cost}, since we have assumed $A \in \cB_{\fro}(\Ast;\rcost(\Ast))$. 
\end{proof}

\begin{proof}[Proof of \Cref{prop:smooth_theta}]
\textbf{Existence and Differentiability of $\bthetast$.}
By \Cref{lem:knr_loss_diff} we have that $\cJ(\btheta;A)$ is four-times differentiable in its arguments.
By the Implicit Function Theorem, since $\lammin(\nabla_{\btheta}^2 \cJ(\btheta;\Ast)|_{\btheta = \bthetast(\Ast)}) > 0$ by assumption, we have that there exists some $\rtheta'(\Ast) > 0$ and unique function $\bthetast(A)$ defined on $\cB_{\fro}(\Ast;\rtheta'(\Ast))$ such that $\nabla_{\btheta} \cJ(\btheta;A)|_{\btheta = \bthetast(A)} = 0$, and $\bthetast(A)$ is three-times differentiable (note that, while the Implicit Function Theorem is typically stated to give that the resulting function is only one-time differentiable, it can be extended to $k$-times differentiable, assuming the implicit equation is $k$-times differentiable \citep{dieudonne2011foundations}).

By \Cref{lem:knr_loss_diff} and the continuity of eigenvalues, it follows that for $A$ close enough to $\Ast$, we have $\lammin(\nabla_{\btheta}^2 \cJ(\btheta;A)|_{\btheta = \bthetast(A)}) \ge \frac{1}{2} \lammin(\nabla_{\btheta}^2 \cJ(\btheta;\Ast)|_{\btheta = \bthetast(\Ast)}) > 0$. We can therefore apply the Implicit Function Theorem as above to any $A$ satisfying this, to get that there exists some unique $\bthetatilst(A')$ defined for all $A'$ near $A$ such that $\nabla_{\btheta} \cJ(\btheta;A')|_{\btheta = \bthetatilst(A')} = 0$ and $\bthetatilst(A')$ is differentiable. By the uniqueness of $\bthetast(A)$ on $\cB_{\fro}(\Ast;\rtheta'(\Ast))$, it follows that any $\bthetatilst(A')$ defined in this way must be identical to $\bthetast(A)$ on $\cB_{\fro}(\Ast;\rtheta'(\Ast))$ (assuming the regions on which they are defined overlaps). We can therefore define $\bthetast(A)$ to simply be the extension of $\bthetast(A)$ to all such $\bthetatilst(A)$, defined for all $A$ near $\Ast$ such that $\lammin(\nabla_{\btheta}^2 \cJ(\btheta;A)|_{\btheta = \bthetast(A)}) \ge \frac{1}{2} \lammin(\nabla_{\btheta}^2 \cJ(\btheta;\Ast)|_{\btheta = \bthetast(\Ast)})$, and will have that $\bthetast(A)$ is three-times differentiable and satisfies $\nabla_{\btheta} \cJ(\btheta;A)|_{\btheta = \bthetast(A)} = 0$ for all such $A$. 

We then choose $\rtheta(\Ast)$ to be defined such that, for all $A \in \cB_{\fro}(\Ast;\rtheta(\Ast))$, we have $\lammin(\nabla_{\btheta}^2 \cJ(\btheta;A)|_{\btheta = \bthetast(A)}) \ge \frac{1}{2} \lammin(\nabla_{\btheta}^2 \cJ(\btheta;\Ast)|_{\btheta = \bthetast(\Ast)})$. By \Cref{lem:knr_loss_diff}, we know that $\nabla_{\btheta}^2 \cJ(\btheta;A)$ is continuous and furthermore we know that eigenvalues are continuous. Using the gradient bounds given in \Cref{lem:knr_loss_diff} to bound the Lipschitz constant of $\nabla_{\btheta}^2 \cJ(\btheta;A)$, it follows that we can take
\begin{align*}
\rtheta(\Ast) =  \poly \left ( \frac{1}{\lammin(\nabla_{\btheta}^2 \cJ(\btheta;\Ast)|_{\btheta = \bthetast(\Ast)})},  \| \Ast \|_\op, \Bphi, \Lphi, \Lzeta, \Lcost, \sigw^{-1}, H, \dimx \right )^{-1}.
\end{align*}

\paragraph{Bounding Norm of Gradients.}
Fix $A \in \cB_{\fro}(\Ast;\rtheta(\Ast))$. 
We know that $\bthetast(A)$ satisfies
\begin{align*}
\nabla_{\btheta} \cJ(\btheta;A)|_{\btheta = \bthetast(A)} = 0.
\end{align*}
We wish to differentiate $\bthetast(A)$ with respect to $A$, and bound the magnitude of up to the third derivative. Similar to the proof of \Cref{lem:knr_loss_diff},
we let $A_{\bmt} := A + t_1 \DelA_1 + t_2 \DelA_2 + t_3 \DelA_3$ for some $\DelA_i$ satisfying $\| \DelA_i \|_\op \le 1$. As noted in the proof of \Cref{lem:knr_loss_diff}, we have
\begin{align*}
\frac{\rmd}{\rmd t_i} \bthetast(A_{\bmt})|_{\bmt = 0} = \nabla_A \bthetast(A)[\DelA_i]
\end{align*}
(and similarly for higher-order derivatives). Thus, to show the result, it suffices to show that $\bthetast(A_{\bmt})$ is differentiable in $t_1,t_2,t_3$ for all $\DelA_i$, and to bound the magnitude of this derivative for all $\DelA_i$ with $\| \DelA_i \|_\op \le 1$.
We have
\begin{align}
& \frac{\rmd}{\rmd t_1} \nabla_{\btheta} \cJ(\btheta;A_{\bmt})|_{\btheta = \bthetast(A_{\bmt})} \bline_{\bmt=0} = 0 \nonumber \\
& \implies \underbrace{\nabla_{A'} \nabla_{\btheta} \cJ(\btheta;A')|_{\btheta = \bthetast(A),A'=A}[\DelA_1]}_{=: G_1(A,\DelA_1)} + \nabla_{\btheta}^2 \cJ(\btheta;A)|_{\btheta = \bthetast(A)} \cdot \nabla_A \bthetast(A)[\DelA_1] = 0 \label{eq:thetast_grad1}
\end{align}
which implies
\begin{align*}
\nabla_A \bthetast(A)[\DelA_1] = - \big ( \nabla_{\btheta}^2 \cJ(\btheta;A)|_{\btheta = \bthetast(A)}  \big )^{-1} \cdot G_1(A,\DelA_1)
\end{align*}
which is well-defined since we have assumed that $\nabla_{\btheta}^2 \cJ(\btheta;A)|_{\btheta = \bthetast(A)}$ is full-rank, and $\cJ$ is differentiable in both its arguments by \Cref{lem:knr_loss_diff}.
To compute the second derivative of $\bthetast$, we differentiate through \eqref{eq:thetast_grad1} which gives
\begin{align*}
& \frac{\rmd}{\rmd t_2} \left ( G_1(A_{\bmt},\DelA_1) + \nabla_{\btheta}^2 \cJ(\btheta;A_{\bmt})|_{\btheta = \bthetast(A_{\bmt})} \cdot \nabla_A \bthetast(A_{\bmt})[\DelA_1] \right ) \bline_{\bmt = 0} = 0 \\
& \implies \underbrace{\frac{\rmd}{\rmd t_2} \left ( G_1(A_{\bmt},\DelA_1) + \nabla_{\btheta}^2 \cJ(\btheta;A_{\bmt})|_{\btheta = \bthetast(A_{\bmt})} \cdot \nabla_A \bthetast(A)[\DelA_1] \right ) \bline_{\bmt = 0}}_{=: G_2(A,\DelA_1,\DelA_2)} \\
& \qquad \qquad + \nabla_{\btheta}^2 \cJ(\btheta;A)|_{\btheta = \bthetast(A)} \cdot \nabla_A^2 \bthetast(A)[\DelA_1,\DelA_2] = 0.
\end{align*}
Note that $G_2(A,\DelA_1,\DelA_2)$ involves at most a third-order derivative of $\cJ(\btheta;A)$ and first-order derivative of $\bthetast(A)$, both of which we know exist by \Cref{lem:knr_loss_diff} and what we showed above.
This then further implies
\begin{align*}
\nabla_A^2 \bthetast(A)[\DelA_1,\DelA_2] & = - \big ( \nabla_{\btheta}^2 \cJ(\btheta;A)|_{\btheta = \bthetast(A)} \big )^{-1} \cdot G_2(A,\DelA_1,\DelA_2),
\end{align*}
which is well-defined since we have assumed that $\nabla_{\btheta}^2 \cJ(\btheta;A)|_{\btheta = \bthetast(A)}$ is full-rank.
Finally, we compute
\begin{align*}
& \frac{\rmd}{\rmd t_3} \left ( G_2(A_{\bmt}, \DelA_1, \DelA_2) + \nabla_{\btheta}^2 \cJ(\btheta;A_{\bmt})|_{\btheta = \bthetast(A_{\bmt})} \cdot \nabla_A^2 \bthetast(A_{\bmt})[\DelA_1,\DelA_2] \right )\bline_{\bmt = 0} = 0 \\
& \implies \underbrace{\frac{\rmd}{\rmd t_3} \left ( G_2(A_{\bmt}, \DelA_1, \DelA_2) + \nabla_{\btheta}^2 \cJ(\btheta;A_{\bmt})|_{\btheta = \bthetast(A_{\bmt})} \cdot \nabla_A^2 \bthetast(A)[\DelA_1,\DelA_2] \right )\bline_{\bmt = 0}}_{=: G_3(A,\DelA_1,\DelA_2,\DelA_3)} \\
& \qquad \qquad + \nabla_{\btheta}^2 \cJ(\btheta;A)|_{\btheta = \bthetast(A)} \cdot \nabla_A^3 \bthetast(A)[\DelA_1,\DelA_2,\DelA_3] = 0.
\end{align*}
Note that $G_3(A,\DelA_1,\DelA_2,\DelA_3)$ involves at most a fourth-order derivative of $\cJ(\btheta;A)$ and second-order derivative of $\bthetast(A)$, both of which we know exist by \Cref{lem:knr_loss_diff} and what we showed above. We therefore have
\begin{align*}
\nabla_A^3 \bthetast(A)[\DelA_1,\DelA_2,\DelA_3] = - \big ( \nabla_{\btheta}^2 \cJ(\btheta;A)|_{\btheta = \bthetast(A)} \big )^{-1} \cdot G_3(A,\DelA_1,\DelA_2,\DelA_3)
\end{align*}
which is well-defined since we have assumed that $\nabla_{\btheta}^2 \cJ(\btheta;A)|_{\btheta = \bthetast(A)}$ is full-rank. As each of these expressions is defined for all choice of $\DelA_i$, the differentiability of $\bthetast(A)$ follows.

Note that the above expressions for $\nabla_A \bthetast(A)[\DelA_1], \nabla_A^2 \bthetast(A)[\DelA_1,\DelA_2],$ and $\nabla_A^3 \bthetast(A)[\DelA_1,\DelA_2,\DelA_3]$ all depend on at most a fourth derivative of $\cJ(\btheta;A)$, as well as $( \nabla_{\btheta}^2 \cJ(\btheta;A)|_{\btheta = \bthetast(A)} )^{-1}$. The norm bounds are then a direct consequence of \Cref{lem:knr_loss_diff}.

\end{proof}

\begin{proof}[Proof of \Cref{prop:theta_global_min}]
By \Cref{lem:knr_loss_diff} we have that $\cJ(\btheta;A)$ is four-times differentiable in its arguments.
Since we have assumed $\nabla_{\btheta}^2 \cJ(\btheta;\Ast)|_{\btheta = \bthetast(\Ast)} \succ 0$, by the Implicit Function Theorem \citep{dieudonne2011foundations}, it follows that there exists some $\rtheta' > 0$ and mapping $\bthetatil(A)$ such that, for all $A \in \cB_{\fro}(\Ast;\rtheta')$, $\nabla_{\btheta} \cJ(\btheta;A)|_{\btheta = \bthetatil(A)} = 0$, and $\bthetatil(A)$ is three-times differentiable.

Our goal is now to show that $\bthetatil(A) = \bthetast(A)$ for $A$ close enough to $\Ast$.
By the continuity of eigenvalues, $\cJ(\btheta;A)$, and $\bthetatil(A)$, we have that there exists some $r$ and $\rtheta''$ such that, for all $\btheta \in \cB_{\fro}(\bthetast(\Ast);r)$ and $A \in \cB_{\fro}(\Ast;\rtheta'')$, we have $\nabla_{\btheta}^2 \cJ(\btheta;A) \succ 0$ and, furthermore, $\bthetatil(A) \in \cB_{2}(\bthetast(\Ast);r/2)$ for all $A \in \cB_{\fro}(\Ast;\rtheta'')$. This implies that $\bthetatil(A)$ is strict local minimum of $\cJ(\btheta;A)$ and, in particular, that
\begin{align*}
\cJ(\btheta;A) > \cJ(\bthetatil(A);A), \quad \forall \btheta \in \cB_{2}(\bthetast(\Ast);r), \btheta \neq \bthetatil(A).
\end{align*}
Let $\Delst := \min_{\btheta \not \in \cB_{2}(\bthetast(\Ast);r)} \cJ(\btheta;\Ast) - \cJ(\bthetast(\Ast);\Ast)$ and note that, since we have assumed the global minimum of $\cJ(\btheta;\Ast)$ is unique, we have $\Delst > 0$. 

Fix some $A \in \cB_{\fro}(\Ast;\rtheta'')$ and assume that $\bthetatil(A)$ is not the global minimum of $\cJ(\btheta;A)$. This implies that $\bthetast(A)$, the global minimum of $\cJ(\btheta;A)$, is outside of $\cB_{2}(\bthetast(\Ast);r)$. Furthermore, by the continuity of $\cJ(\btheta;A)$, we have, for some $L, L' > 0$,
\begin{align*}
\cJ(\bthetast(A);\Ast) & \le \cJ(\bthetast(A);A) + L \| A - \Ast \|_{\fro}\\
& \le \cJ(\bthetatil(A);A) + L \| A - \Ast \|_{\fro} \\
& \le \cJ(\bthetatil(A);\Ast) + 2L \| A - \Ast \|_{\fro} \\
& \le \cJ(\bthetast(\Ast);\Ast) + 2L \| A - \Ast \|_{\fro} + L \| \bthetatil(A) - \bthetast(\Ast) \|_{2}  \\
& \le \cJ(\bthetast(\Ast);\Ast) + L' \| A - \Ast \|_{\fro} .
\end{align*}
This implies that
\begin{align*}
L' \rtheta'' \ge L' \| A - \Ast \|_{\fro} \ge \cJ(\bthetast(A);\Ast) - \cJ(\bthetast(\Ast);\Ast) \ge \Delst.
\end{align*}
However, for $\rtheta''$ small enough, this is a contradiction. Thus, it follows that $\bthetatil(A)$ is the global minimum of $\cJ(\btheta;A)$, so $\bthetatil(A) = \bthetast(A)$. 

The result then follows 
since we already have that $\bthetatil(A)$ is three-times differentiable and satisfies $\nabla_{\btheta} \cJ(\btheta;A)|_{\btheta = \bthetatil(A)} = 0$, and
by taking $\rtheta(\Ast)$ to be the minimum of $\rtheta'$ and $\rtheta''$. The boundedness of $\LKst$ follows as in the proof of \Cref{prop:smooth_theta}.

\end{proof}

\begin{proof}[Proof of \Cref{prop:thetast_smooth_convex}]
Note that, by convexity and the KKT conditions, the solutions to $\argmin_{\btheta \in \R^{\dimtheta}} \cJ(\btheta;A)$ are described by
\begin{align*}
\nabla_{\btheta} \cJ(\btheta;A)|_{\btheta = \bthetast(A)} = 0.
\end{align*}
Thus, an equivalent definition for $\bthetast(A)$ is that it satisfies $\nabla_{\btheta} \cJ(\btheta;A)|_{\btheta = \bthetast(A)} = 0$. \Cref{asm:smooth_thetast} can then be shown to hold by an argument analogous to \Cref{prop:smooth_theta}.
\end{proof}

\begin{proof}[Proof of \Cref{lem:quadratic_loss_approx}]
Let $A(t) = t \Ahat + (1-t) \Ast$ and $g(t) := \cJ(\bthetast(A(t));\Ast)$. 
By \Cref{lem:knr_loss_diff} and under \Cref{asm:smooth_controller}, we have that both $\cJ(\btheta;A)$ and $\bthetast(A)$ are three-times differentiable for all $A = t \Ahat + (1-t) \Ast, t \in [0,1]$, so it follows that $g(t)$ is three-times differentiable in $t$. 
We can therefore apply Taylor's Theorem to expand $g(1)$ about the point $t = 0$ to get:
\begin{align*}
g(1) & = g(0) + \nabla_{\btheta} \cJ(\btheta;\Ast)|_{\btheta = \bthetast(\Ast)} \cdot \nabla_A \bthetast(A)|_{A = \Ast} [\Ahat - \Ast] \\
& \qquad + \nabla_A \bthetast(A)|_{A = \Ast}^\top \nabla_{\btheta}^2 \cJ(\btheta;\Ast)|_{\btheta = \bthetast(\Ast)} \nabla_A \bthetast(A)|_{A = \Ast}[\Ahat - \Ast, \Ahat - \Ast] \\
& \qquad + \nabla_{\btheta} \cJ(\btheta;\Ast)|_{\btheta = \bthetast(\Ast)} \cdot \nabla_A^2 \bthetast(A)|_{A = \Ast}[\Ahat - \Ast,\Ahat-\Ast] \\
& \qquad + \nabla_A^3 \cJ(\bthetast(A);\Ast)|_{A = A'}[\Ahat - \Ast,\Ahat-\Ast,\Ahat-\Ast]
\end{align*}
where $A' = A(t')$ for some $t' \in [0,1]$. Under \Cref{asm:smooth_thetast}, we have that $ \nabla_{\btheta} \cJ(\btheta;\Ast)|_{\btheta = \bthetast(\Ast)} = 0$, which implies that, plugging in the definition of $g(1)$ and $g(0)$,
\begin{align*}
\cJ(\bthetast(\Ahat); \Ast) & = \cJ(\bthetast(\Ast); \Ast) + \nabla_A \bthetast(A)|_{A = \Ast}^\top \nabla_{\btheta}^2 \cJ(\btheta;\Ast)|_{\btheta = \bthetast(\Ast)} \nabla_A \bthetast(A)|_{A = \Ast}[\Ahat - \Ast, \Ahat - \Ast] \\
& \qquad + \nabla_A^3 \cJ(\bthetast(A);\Ast)|_{A = A'}[\Ahat - \Ast,\Ahat-\Ast,\Ahat-\Ast].
\end{align*}
We can bound
\begin{align*}
| \nabla_A^3 \cJ(\bthetast(A);\Ast)|_{A = A'}[\Ahat - \Ast,\Ahat-\Ast,\Ahat-\Ast] | & \le \| \nabla_A^3 \cJ(\bthetast(A);\Ast)|_{A = A'} \|_\op \cdot \| \Ahat - \Ast \|_\op^3 .
\end{align*}
The expression for $\nabla_A^3 \cJ(\bthetast(A);\Ast)$ contains up to the third derivative of both $\cJ(\btheta;\Ast)$ and $\bthetast(A)$. By \Cref{lem:knr_loss_diff} and under \Cref{asm:smooth_thetast}, since $A' \in \cB_{\fro}(\Ast; \min \{ \rcost(\Ast), \rtheta(\Ast) \} )$ by construction, we can then bound
\begin{align*}
\| \nabla_A^3 \cJ(\bthetast(A);\Ast)|_{A = A'} \|_\op \le \poly(\LKst, \| \Ast \|_\op, \Bphi, \Lphi, \Lzeta, \Lcost, \sigw^{-1}, H, \dimx). 
\end{align*}
The result follows by the definition of $\cH(\Ast)$. 
\end{proof}

\begin{proof}[Proof of \Cref{lem:Hst_to_Hhat}]
Recall that $\cH(\Ahat) = \nabla_{A}^2 \cJ(\bthetast(A);\Ahat)|_{A = \Ahat}$. To prove this, we will use that this is differentiable by \Cref{lem:knr_loss_diff}, and will apply Taylor's Theorem. 

First, note that by Taylor's Theorem we have
\begin{align*}
\nabla_{A}^2 \cJ(\bthetast(A);\Ahat)|_{A = \Ahat} = \nabla_{A}^2 \cJ(\bthetast(A);\Ahat)|_{A = \Ast} + \nabla_{A}^3 \cJ(\bthetast(A);\Ahat)|_{A = A'}[\Ahat - \Ast]
\end{align*}
for $A' = t \Ahat + (1-t) \Ast$ for some $t \in [0,1]$. The third derivative of $\cJ(\bthetast(A);\Ahat)$ will involve up to the third derivative of both $\cJ(\btheta;\Ahat)$ and $\bthetast(A)$, so using \Cref{lem:knr_loss_diff} and \Cref{asm:smooth_thetast}, since $A' \in \cB_{\fro}(\Ast;\min \{ \rcost(\Ast), \rtheta(\Ast) \})$ by assumption, we can bound
\begin{align*}
\| \nabla_{A}^3 \cJ(\bthetast(A);\Ahat)|_{A = A'}[\Ahat - \Ast] \|_{\op} \le \poly(\LKst, \| \Ast \|_\op,  \Bphi, \Lphi, \Lzeta, \Lcost, \sigw^{-1}, H, \dimx)\cdot \| \Ahat - \Ast \|_\op.
\end{align*}
Next, we wish to relate $\nabla_{A}^2 \cJ(\bthetast(A);\Ahat)|_{A = \Ast}$ to $\nabla_{A}^2 \cJ(\bthetast(A);\Ast)|_{A = \Ast} = \cH(\Ast)$. Again applying Taylor's Theorem, we have
\begin{align*}
\nabla_{A}^2 \cJ(\bthetast(A);\Ahat)|_{A = \Ast} = \nabla_{A}^2 \cJ(\bthetast(A);\Ast)|_{A = \Ast} + \nabla_{A'} \nabla_{A}^2 \cJ(\bthetast(A);A')|_{A = \Ast,A'=A''}[\Ahat - \Ast]
\end{align*}
for $A'' = t \Ahat + (1-t) \Ast$ for some $t \in [0,1]$. By \Cref{lem:knr_loss_diff} and \Cref{asm:smooth_thetast}, we can bound
\begin{align*}
& \|  \nabla_{A'} \nabla_{A}^2 \cJ(\bthetast(A);A')|_{A = \Ast,A'=A''}[\Ahat - \Ast] \|_\op \\
& \qquad \le  \poly(\LKst, \| \Ast \|_\op, \Lphi, \Lzeta, \Lcost, \sigw^{-1}, H, \dimx)\cdot \| \Ahat - \Ast \|_\op.
\end{align*}
The result follows.
\end{proof}

\begin{lemma}\label{lem:hessian_norm_bound}
Under \Cref{asm:smooth_phi,asm:smooth_controller,asm:bounded_features,asm:bounded_cost,asm:smooth_thetast}, for all $A \in \cB_{\fro}(\Ast;\min \{ \rcost(\Ast), \rtheta(\Ast) \})$, we can bound 
\begin{align*}
\| \cH(A) \|_\op \le \poly(\| \Ast \|_\op, \Bphi, \Lphi, \Lzeta, \Lcost, \LKst, \sigw^{-1}, H, \dimx)
\end{align*}
\end{lemma}
\begin{proof}
Recall that $\cH(\Ahat) = \nabla_{A}^2 \cJ(\bthetast(A);\Ahat)|_{A = \Ahat}$. The bound then follows from \Cref{lem:knr_loss_diff} and \Cref{asm:smooth_thetast}.
\end{proof}

\section{High-Probability Regret Bounds in Nonlinear Systems}
In this section, we modify the proof the main result of \cite{kakade2020information} slightly to show a high probability regret bound for \LC. For the sake of brevity, we omit details that are identical to the proof given in \cite{kakade2020information}. We will need the following assumption.

\begin{asm}[Bounded Cost]\label{asm:abs_bounded_cost}
We assume that, for all trajectories $\traj$, we have $\cost(\traj) \le \cmax$.
\end{asm}

We adopt the notation used in this work, modifying somewhat the notation from \cite{kakade2020information}. In particular, we let $\cJ(\pi;A)$ denote the expected cost of playing policy $\pi$ under system $A$, and we set 
\begin{align*}
\bSigma_t = \sum_{s=1}^t \sum_{h=1}^H \bphi(\bx_h^t,\bu_h^t) \bphi(\bx_h^t,\bu_h^t)^\top + \lambda I
\end{align*}
denote the covariates obtained by the first $t$ episodes of \LC (plus a regularizer). We let $\pi^t$ denote the policy played at episode $t$ of \LC. For a policy set $\Pi$, we define regret as
\begin{align*}
\cR_T(\Pi) := \sum_{t=1}^T \cJ(\pi^t;\Ast) - T \cdot \min_{\pi \in \Pi} \cJ(\pi;\Ast).
\end{align*}
We will also denote $\pist := \argmin_{\pi \in \Pi} \cJ(\pi;\Ast)$. 

In addition to these notational changes, we modify \LC slightly to use the parameter
\begin{align*}
\beta^t := \sqrt{\lambda} \BA + \sqrt{8 \dimx \log 5 + 8 \log ( T \det (\bSigma_t) \det(\bSigma_0)^{-1} / \delta)}
\end{align*}
in the construction of the confidence set, $\ball^t$.

Besides the aforementioned changes, in the following proofs we adopt the same notation as \cite{kakade2020information}.
We have the following result.

\begin{theorem}\label{thm:knr_regret}
Under \Cref{asm:bounded_features,asm:abs_bounded_cost} and with any policy class $\Pi$, with probability at least $1-\delta$, \LC has regret bounded as
\begin{align*}
\cR_T(\Pi) \le C \cdot \cmax H \sqrt{\dimphi \cdot (\dimphi + \dimx + \BA + \log \frac{1}{\delta} ) \cdot T } \cdot \log \left ( 1 + \Bphi HT / \sigw \right )
\end{align*}
for a universal constant $C$.
\end{theorem}
\begin{proof}[Proof of \Cref{thm:knr_regret}]
By \Cref{lem:knr_regret_good_event}, we have that the event $\cE_1$ holds with probability at least $1-\delta$. We therefore assume $\cE_1$ holds for the remainder of the proof. 

By the definition of the confidence set in \LC, on $\cE_1$ we have that $\Ast$ is the in confidence set for all $t \le T$. It follows that on $\cE_1$,
\begin{align}
\cR_T & = \sum_{t=1}^T \left [ \cJ(\pi^t;\Ast) - \cJ(\pist;\Ast) \right ] \nonumber \\
& \overset{(a)}{\le} \sum_{t=1}^T \left [ \cJ(\pi^t;\Ast) - \cJ(\pist;\Ahat^t) \right ] \nonumber \\
& \overset{(b)}{\le} \sum_{t=1}^T \cmax \cdot  \Exp_{\Ast,\pi^t} \left [ \sum_{h=1}^H \min \left \{ \frac{1}{\sigw} \| (\Ast - \Ahat^t) \cdot \bphi(\bx_h,\bu_h) \|_2, 1 \right \} \right ] \label{eq:knr_regret_eq1}
\end{align}
where $(a)$ follows from the optimistic property of \LC when $\Ast \in \ball^t$, and $(b)$ follows from \Cref{lem:knr_instant_regret}. On $\cE_1$, we have
\begin{align*}
\| (\Ast - \Ahat^t) \bphi(\bx_h,\bu_h) \|_2 & \le \| (\Ast - \Ahat^t) \bSigma_t^{1/2} \|_2 \| \bSigma_t^{-1/2} \bphi(\bx_h,\bu_h) \|_2 \\
& \le \left ( \| (\Ast - \Abar^t) \bSigma_t^{1/2} \|_2 + \| (\Abar^t - \Ahat^t) \bSigma_t^{1/2} \|_2 \right ) \cdot \| \bSigma_t^{-1/2} \bphi(\bx_h,\bu_h) \|_2 \\
& \le 2 \beta^t \| \bphi(\bx_h,\bu_h) \|_{\bSigma_t^{-1}}
\end{align*}
where the last inequality follows from the definition of $\ball^t$ since $\Ahat^t \in \ball^t$ by construction, and by the definition of $\cE_1$. This gives
\begin{align*}
\eqref{eq:knr_regret_eq1} \le \sum_{t=1}^T \cmax \cdot \Exp_{\Ast,\pi^t} \left [ \sum_{h=1}^H \min \left \{ \frac{2 \beta^t}{\sigw} \|  \bphi(\bx_h,\bu_h) \|_{\bSigma_t^{-1}}, 1 \right \} \right ] .
\end{align*}
By \Cref{lem:elliptic_concentration}, with probability $1-\delta$ we can bound this as
\begin{align*}
& \le \underbrace{\frac{2\cmax \beta^T}{\sigw} \cdot \sum_{t=1}^T \sum_{h=1}^H  \min \left \{  \|  \bphi(\bx_h^t,\bu_h^t) \|_{\bSigma_t^{-1}}, 1 \right \}}_{(a)} + 4 \cmax H \sqrt{T \log 1/\delta}. 
\end{align*}
By Cauchy-Schwarz, we can bound $(a)$ as
\begin{align*}
(a) & \le \frac{2\cmax \beta^T}{\sigw}  \cdot \sqrt{T} \sqrt{ \sum_{t=1}^T  \sum_{h=1}^H \min \left \{  \|  \bphi(\bx_h^t,\bu_h^t) \|_{\bSigma_t^{-1}}^2, 1 \right \}}.
\end{align*}
We have
\begin{align*}
\sum_{t=1}^T  \sum_{h=1}^H \min \left \{  \|  \bphi(\bx_h^t,\bu_h^t) \|_{\bSigma_t^{-1}}^2, 1 \right \} & H \le \sum_{t=1}^T   \min \left \{ \sum_{h=1}^H \|  \bphi(\bx_h^t,\bu_h^t) \|_{\bSigma_t^{-1}}^2, 1 \right \} \\
& \le 2 H \log (\det ( \bSigma_T) \det (\bSigma_0)^{-1})
\end{align*}
where the last inequality uses Lemma B.6 of \cite{kakade2020information}. Putting all of this together, we have shown that with probability at least $1-2\delta$, we have
\begin{align*}
\cR_T & \le \frac{2\cmax \beta^T}{\sigw}  \cdot \sqrt{T} \cdot \sqrt{2 H \log (\det ( \bSigma_T) \det (\bSigma_0)^{-1})} + 4 \cmax H \sqrt{T \log 1/\delta}.
\end{align*}

It remains to bound $\beta^T$ and $\log (\det ( \bSigma_T) \det (\bSigma_0)^{-1})$. We have $\bSigma_0 = \lambda I$, so $\det(\bSigma_0) = \lambda^{\dimphi}$. Furthermore, if $\| \bphi(\bx,\bu) \|_2 \le \Bphi$, then we can bound $\det(\bSigma_T) \le (\lambda + \Bphi^2 TH)^{\dimphi}$. Putting this together we have
\begin{align*}
\log (\det ( \bSigma_T) \det (\bSigma_0)^{-1}) & \le \dimphi \cdot \log ( 1 + \Bphi^2 TH / \lambda).
\end{align*}
Recalling that
\begin{align*}
\beta^T = \sqrt{\lambda} \BA + \sigw \sqrt{8 \dimx \log 5 + 8 \log ( T \det (\bSigma_T) \det(\bSigma_0)^{-1} / \delta)}
\end{align*}
we can similarly bound
\begin{align*}
\beta^T/\sigw & \le \sqrt{\lambda} \BA/\sigw + \sqrt{8 \dimx \log 5 + 8  \dimphi \cdot \log ( 1 + \Bphi^2 TH / \lambda) + 8 \log (T/\delta)} \\
& \le \sqrt{\lambda} \BA/\sigw + c \sqrt{\dimx + \dimphi \log ( 1 + \Bphi T H/\lambda) + \log 1/\delta}.
\end{align*}
Choosing $\lambda = \sigw^2$ completes the proof.
\end{proof}

\subsection{Supporting Lemmas}

\begin{lemma}\label{lem:knr_instant_regret}
Under \Cref{asm:abs_bounded_cost}, we can bound
\begin{align*}
 \cJ(\pi;\Ast) - \cJ(\pi;A) \le \cmax \cdot  \Exp_{\Ast,\pi} \left [ \sum_{h=1}^H \min \left \{ \frac{1}{\sigw} \| (\Ast - A) \bphi(\bx_h,\bu_h) \|_2, 1 \right \} \right ] .
\end{align*}
\end{lemma}
\begin{proof}
Following the proof of Lemma B.3 of \cite{kakade2020information}, and adopting the same notation, we have
\begin{align*}
 \cJ(\pi;\Ast) - \cJ(\pi;A) \le \sum_{h=1}^H\Exp_{\Ast,\pi} \left [ \sqrt{A_h} \min \left \{ \frac{1}{\sigw} \| (\Ast - A) \bphi(\bx_h,\bu_h) \|_2, 1 \right \} \right ] .
\end{align*}
Under \Cref{asm:abs_bounded_cost} we have $A_h \le \cmax^2$. Plugging this in gives the result.
\end{proof}

\begin{lemma}\label{lem:knr_regret_good_event}
Let $\beta^t := \sqrt{\lambda} \BA + \sigw \sqrt{8 \dimx \log 5 + 8 \log ( T \det (\bSigma_t) \det(\bSigma_0)^{-1} / \delta)}$ and let
$\cE_1$ denote the event
\begin{align*}
\cE_1 := \left \{ \forall t \le T \ : \ \left \| \left ( \Abar^t - \Ast \right ) \bSigma_t^{1/2} \right \|_\op \le \beta^t \right \}.
\end{align*}
Then running \LC we have $\Pr_{\Ast}[\cE_1] \ge 1 - \delta$. 
\end{lemma}
\begin{proof}
The proof of Lemma B.5 of \cite{kakade2020information} shows that with probability at least $1-\delta$,
\begin{align*}
\left \| \left ( \Abar^t - \Ast \right ) \bSigma_t^{1/2} \right \|_\op \le \sqrt{\lambda} \| \Ast \|_\op + \sigw \sqrt{8 \dimx \log 5 + 8 \log ( \det (\bSigma_t) \det(\bSigma_0)^{-1} / \delta)}.
\end{align*}
The result then follows from this, since $\| \Ast \|_\op \le \BA$, and a union bound.
\end{proof}

\begin{lemma}\label{lem:elliptic_concentration}
With probability $1-\delta$, we have
\begin{align*}
\sum_{t=1}^T \Exp_{\Ast,\pi^t} \left [ \sum_{h=1}^H \min \left \{\frac{2 \beta^t}{\sigw} \|  \bphi(\bx_h,\bu_h) \|_{\bSigma_t^{-1}}, 1 \right \} \right ] & \le \frac{2\beta^T}{\sigw}  \sum_{t=1}^T \sum_{h=1}^H \min \left \{   \|  \bphi(\bx_h^t,\bu_h^t) \|_{\bSigma_t^{-1}}, 1 \right \} \\
& \qquad + 4 H \sqrt{T \log 1/\delta}.
\end{align*}
\end{lemma}
\begin{proof}
This is an immediate consequence of Azuma-Hoeffding, since $\sum_{h=1}^H \min \left \{ \frac{2 \beta^t}{\sigw} \|  \bphi(\bx_h,\bu_h) \|_{\bSigma_t^{-1}}, 1 \right \} \le H$ almost surely, and from upper bounding
\begin{align*}
 \min \left \{ \frac{2\beta^t}{\sigw} \|  \bphi(\bx_h^t,\bu_h^t) \|_{\bSigma_t^{-1}}, 1 \right \}  & \le \frac{2\beta^T}{\sigw} \min \left \{   \|  \bphi(\bx_h^t,\bu_h^t) \|_{\bSigma_t^{-1}}, 1 \right \} .
\end{align*}
\end{proof}


\section{Lower Bounds on Learning in Nonlinear Systems}\label{sec:lb_proofs}
In this section, we assume that $\bthetast$ and $\Kst$ correspond to the global minimizer:
\begin{align}\label{eq:lb_Kst}
\bthetast(A) := \argmin_{\btheta \in \R^{\dimzeta}} \cJ(\btheta;A), \quad \Kst(A) := \argmin_{\pi \in \Picon} \cJ(\pi;A).
\end{align}

Here we formally state the additional assumptions needed in \Cref{sec:lb}, and provide a formal version of \Cref{thm:lb_informal}.

\begin{asm}\label{asm:lb_strong_convex}
There exists some $\rmu(\Ast) > 0$ such that, for all $A \in \cB_{\fro}(\Ast,\rmu(\Ast))$, $\Kst(A)$ is unique and, furthermore, there exists some $\mu > 0$ such that
\begin{align*}
\cJ(\btheta;A) \ge \cJ(\bthetast(A);A) + \tfrac{\mu}{2} \| \btheta - \thetast(A) \|_2^2 .
\end{align*}
\end{asm}
\Cref{asm:lb_strong_convex} will be satisfied in cases where $\cJ(\Ktheta;A)$ is strongly convex in $\btheta$, but may hold even when this is not the case.
Intuitively, it requires that our controller class is not overparameterized---moving $\btheta$ away from its optimal value will cause the loss to increase. 
We will additionally make the following regularity assumptions on policies in $\Piexp$ and their induced covariates set, $\bOmega$.

\begin{asm}\label{asm:lb_min_eig}
There exists some $\lamun > 0$ such that, for each $\bLambda \in \bOmega$, we have $\lammin(\bLambda) \ge \lamun$.
\end{asm}

\Cref{asm:lb_min_eig} requires that \emph{every} exploration policy we consider excites all directions in $\bphi$ space (in contrast, \Cref{asm:full_rank_cov} only assumes there \emph{exists} some distribution over policies in $\Piexp$ which excite all directions). We remark that this assumption is relatively mild if \Cref{asm:full_rank_cov} holds. As we show in \Cref{sec:full_rank}, under \Cref{asm:full_rank_cov}, a mixture over policies, $\omega$, satisfying $\lammin(\Exp_{\pi \sim \omega}[\bLambda_\pi])  > 0$ can be learned using only a number of samples scaling polynomially in problem parameters. Given $\omega$, a policy class $\Piexp$ satisfying \Cref{asm:lb_min_eig} can be obtained by simply mixing $\omega$ with every other exploration policy.
We are now ready to state our main lower bound.

\begin{theorem}[Formal Version of \Cref{thm:lb_informal}]\label{thm:lb}
Under \Cref{asm:bounded_features,asm:smooth_phi,asm:smooth_controller,asm:bounded_cost,asm:smooth_thetast,asm:lb_strong_convex,asm:lb_min_eig} and if $\pist$ is defined as in \eqref{eq:lb_Kst}, as long as $T \ge \Clb$, for any $\omegaexp \in \simplex_{\Piexp}$, we have
\begin{align*}
\min_{\Khat} \max_{A \in \cB_T} \Exp_{\frakD_T \sim A,\omegaexp} [\cJ(\Khat(\frakD_T);A) - \cJ(\Kst(A);A)] \ge \frac{\sigw^2}{3T} \cdot \min_{\bLambda \in \bOmega} \tr(\cH(\Ast) \bLamchk^{-1}) - \frac{\Clb}{T^{5/4}} 
\end{align*}
for $\cB_T := \{ A \ : \ \| A - \Ast \|_\fro^2 \le 5 \dimx \dimphi / (\lamun \dimx T H)^{5/6} \}$, $\Exp_{\frakD_T \sim A,\omegaexp}[\cdot] = \Exp_{\pi \sim \omegaexp}[\Exp_{\frakD_T \sim A, \pi}[\cdot]]$ denotes the expectation over trajectories generated by running policies $\pi$ drawn according to $\omegaexp$ on system $A$ for $T$ episodes, $\pihat$ any mapping from observations to policies in $\Picon$,
and 
\begin{align*}
\Clb := \poly \left ( \dimphi, \dimx , H, \| \Ast \|_\op, \Bphi, \Lphi, \Lzeta, \Lcost, \LKst, \sigw, \sigw^{-1}, \tfrac{1}{\lamun}, \tfrac{1}{\mu}, \tfrac{1}{\rcost(\Ast)}, \tfrac{1}{\rtheta(\Ast)}, \tfrac{1}{\rmu(\Ast)} \right ).
\end{align*}
\end{theorem}

\begin{proof}[Proof of \Cref{thm:lb}]
This proof follows immediately from \Cref{lem:lb_intermediate}, by lower bounding the right-hand side of \eqref{eq:lb_intermediate_bound} by the min over all policies in $\Pi$.
\end{proof}

\begin{lemma}\label{lem:lb_intermediate}
Under \Cref{asm:bounded_features,asm:smooth_phi,asm:smooth_controller,asm:bounded_cost,asm:smooth_thetast,asm:lb_strong_convex,asm:lb_min_eig}
and if $\bthetast$ is defined as in \eqref{eq:lb_Kst}, as long as 
\begin{align*}
T \ge \poly(\| \Ast \|_\op, \Lphi, \Lzeta, \LKst, \Lcost, \sigw, \sigw^{-1}, \Bphi, H, \dimx, \dimphi, \lamun^{-1}, \mu^{-1}, \rcost(\Ast)^{-1}, \rtheta(\Ast)^{-1},\rmu(\Ast)^{-1} ),
\end{align*}
for any $\omegaexp \in \simplex_\Pi$, we have
\begin{align}\label{eq:lb_intermediate_bound}
\min_{\bthetahat} \max_{A \in \cB_T} \Exp_{\frakD_T \sim A,\omegaexp} [\cJ(\bthetahat(\frakD_T);A) - \cJ(\bthetast(A);A)] \ge \frac{\sigw^2}{3T} \cdot  \tr(\cH(\Ast) \Exp_{\pi \sim \omegaexp}[\bLamchk_{\pi}]^{-1}) - \frac{\Clb}{T^{5/4}} 
\end{align}
for $\cB_T := \{ A \ : \ \| A - \Ast \|_\fro^2 \le 5 \dimx \dimphi / (\lamun T H)^{5/6} \}$, where $\Exp_{\frakD_T \sim A,\omegaexp}[\cdot] = \Exp_{\pi \sim \omegaexp}[\Exp_{\frakD_T \sim A, \pi}[\cdot]]$ denotes the expectation over trajectories generated by running policies $\pi$ drawn according to $\omegaexp$ on system $A$ for $T$ episodes, and 
\begin{align*}
\Clb := \poly(\| \Ast \|_\op, \Lphi, \Lzeta, \LKst, \Lcost, \sigw, \sigw^{-1}, \Bphi, H, \dimx, \dimphi, \lamun^{-1}, \mu^{-1}, \rcost(\Ast)^{-1}, \rtheta(\Ast)^{-1},\rmu(Ast)^{-1} ).
\end{align*}
\end{lemma}
\begin{proof}
This result is a direct consequence of Theorem 6.1 of \cite{wagenmaker2021task}---to obtain the result we must only verify that the assumptions of this result are met. We verify each assumption below.

\paragraph{Verifying Assumption 3 of \cite{wagenmaker2021task}.}
Part 1 of Assumption 3 of \cite{wagenmaker2021task} is met by \Cref{asm:lb_strong_convex} within diameter $\rmu(\Ast)$. Furthermore, under \Cref{asm:bounded_features,asm:smooth_phi,asm:smooth_controller,asm:bounded_cost,asm:smooth_thetast} and by \Cref{lem:knr_loss_diff}, the additional parts of Assumption 3 of \cite{wagenmaker2021task} are also met with diameter $\min \{ \rcost(\Ast), \rtheta(\Ast) \}$ and smoothness constant $\poly(\| \Ast \|_\op, \Lphi, \Lzeta, \LKst, \Lcost, \sigw^{-1}, H, \dimx)$. 

\paragraph{Verifying Assumption 4 and Assumption 5 of \cite{wagenmaker2021task}.}
Assumption 4 of \cite{wagenmaker2021task} is immediately met by \Cref{asm:lb_min_eig}. Furthermore, Assumption 5 is met by \Cref{lem:lb_smooth_cov} with $c_{\mathrm{cov}} = 1, L_{\mathrm{cov}}(\theta_\star, \gamma^2) = \frac{H^2\Bphi^3}{\sigw^2} \cdot \poly(\dimx)$, and $C_{\mathrm{cov}} = 0$.

Given that these assumptions are met, the result follows noting that, if we run for $T$ episodes, then the effective horizon is $\dimx TH$ (using the mapping from the setting of \eqref{eq:system} to the martingale regression setting described in \Cref{sec:knr_mdm}). Note that the final bound scales with $\frac{1}{T}$ instead of $\frac{1}{\dimx T H}$ as we are able to bring the $\dimx H$ factor into the  $\bLamchk_{\piexp}$ term, since $\bLambda_{\piexp}$ is not normalized by $\dimx H$. 
\end{proof}

\begin{lemma}\label{lem:lb_smooth_cov}
Under \Cref{asm:bounded_features}, for any policy distribution $\omega \in \simplex_{\Piexp}$ and $A,A'$, we have
\begin{align*}
\Exp_{\pi \sim \omega}[\bLambda_{A,\pi}] \preceq \Exp_{\pi \sim \omega}[\bLambda_{A',\pi}] + \frac{H^2 \Bphi^3}{\sigw^2} \cdot \poly(\dimx) \cdot \| A - A' \|_\fro \cdot I.
\end{align*}
\end{lemma}
\begin{proof}
We will prove that the desired bound follows for a particular $\pi \in \Piexp$, which immediately implies that it holds for $\omega \in \simplex_{\Piexp}$. 
By definition we have
\begin{align*}
\bLambda_{A,\pi} = \int \left ( \sum_{h=1}^H \bphi(\bx_h^{\traj},\bu_h^{\traj}) \bphi(\bx_h^{\traj},\bu_h^{\traj})^\top \right ) \cdot f_{A,\pi}(\traj) \rmd \traj
\end{align*}
and 
\begin{align*}
f_{A,\pi}(\traj) & = \prod_{h=1}^H f_{A}(\bx_{h+1}^{\traj} \mid \bx_{h}^{\traj}, \bu_{h}^{\traj}) \pi_h(\bu_h^{\traj} \mid \bx_h^{\traj}) .
\end{align*}

Fix some $\bv \in \cS^{\dimphi - 1}$, and note that, given the expression above, we have
\begin{align*}
\bv^\top \bLambda_{A,\pi} \bv = \int \sum_{h=1}^H (\bv^\top \bphi(\bx_h^{\traj},\bu_h^{\traj}))^2 \cdot f_{A,\pi}(\traj) \rmd \traj.
\end{align*}
It follows that
\begin{align*}
\nabla_{A} \bv^\top \bLambda_{A,\pi} \bv & = \int \sum_{h=1}^H (\bv^\top \bphi(\bx_h^{\traj},\bu_h^{\traj}))^2 \cdot \nabla_A f_{A,\pi}(\traj) \rmd \traj \\
& = \int \sum_{h=1}^H (\bv^\top \bphi(\bx_h^{\traj},\bu_h^{\traj}))^2 \cdot f_{A,\pi}(\traj) \nabla_A \log f_{A,\pi}(\traj) \rmd \traj.
\end{align*}
As in the proof of \Cref{lem:knr_loss_diff}, we have, for any $\Delta$,
\begin{align*}
\nabla_A \log f_{A,\pi}(\traj)[\Delta] & = \sum_{h=1}^H \frac{1}{\sigw^2} (\bx_{h+1}^{\traj} - A \bphi(\bx_h^{\traj},\bu_h^{\traj})) \cdot \Delta \bphi(\bx_h^{\traj},\bu_h^{\traj}),
\end{align*}
so we can bound
\begin{align*}
\| \nabla_A \log f_{A,\pi}(\traj)  \|_\op \le \frac{\Bphi}{\sigw^2} \sum_{h=1}^H \| \bx_{h+1}^{\traj} - A \bphi(\bx_h^{\traj},\bu_h^{\traj}) \|_2.
\end{align*}
Furthermore, we can also bound $(\bv^\top \bphi(\bx_h^{\traj},\bu_h^{\traj}))^2 \le \Bphi^2$. We therefore have
\begin{align*}
\| \nabla_{A} \bv^\top \bLambda_{A,\pi} \bv \|_\op & \le H \Bphi^3 \int \sum_{h=1}^H \| \bx_{h+1}^{\traj} - A \bphi(\bx_h^{\traj},\bu_h^{\traj}) \|_2 \cdot f_{A,\pi}(\traj) \rmd \traj \\
& \le \frac{H^2 \Bphi^3 }{\sigw^2}\cdot \poly(\dimx)
\end{align*}
where the last inequality follows from \Cref{lem:gaussian_bound}. It follows from the Mean Value Theorem that
\begin{align*}
| \bv^\top \bLambda_{A,\pi} \bv - \bv^\top \bLambda_{A',\pi} \bv | \le \frac{H^2 \Bphi^3}{\sigw^2} \cdot \poly(\dimx) \cdot \| A - A' \|_\fro.
\end{align*}
As this holds for all $\bv \in \cS^{d-1}$, it follows that
\begin{align*}
\| \bLambda_{A,\pi}- \bLambda_{A',\pi} \|_\op \le \frac{H^2 \Bphi^3}{\sigw^2} \cdot \poly(\dimx) \cdot \| A - A' \|_\fro.
\end{align*}
\end{proof}


\section{Additional Experimental Details}\label{sec:experiment_details}

In this section, we provide additional details on our experimental results presented in \Cref{sec:experiments}. 
All experiments were run on a machine with 56 Intel(R) Xeon(R) CPU E5-2690 v4 @ 2.60GHz CPUs, and 64GB RAM. All code was implemented in \texttt{PyTorch}.

\subsection{Details on Problem Settings and Controller Parameterizations}
We first expand on the precise definitions of the systems considered. As noted in \Cref{sec:experiments}, for the drone and car examples we set $H = 50$, and for the system of \Cref{sec:motivating_example} we set $H = 10$. In addition, for all examples the noise is distributed as $\bw_h \sim \cN(0,0.1\cdot I)$. In all cases we set $\gamma^2 = 10 H$ (where $\gamma^2$ is a bound on $\Exp_{\piexp}[\sum_{h=1}^H \bu_h^\top \bu_h]$), and we therefore let $\Piexp$ denote the set of all policies satisfying $\Exp_{\piexp}[\sum_{h=1}^H \bu_h^\top \bu_h] \le \gamma^2$. 

\subsubsection{System of \Cref{sec:motivating_example} (\Cref{fig:motivating})}
The dynamics for this system are given by
\begin{align*}
\bx_{h+1} = 0.8 \bx_h + \bu_h - \sum_{i=1}^{10} 3 \bphi_i(\bx_h) + \bw_h
\end{align*}
for $\bphi_i(\bx) = \max \{ 1 - 100(\bx - c_i)^2, 0 \}$, and $\cost(\bx,\bu) = (\bx - c_1)^2 + 100^{-1} \cdot \bu^2$. We set 
\begin{align*}
c_1 = 10, c_2 = -14, c_3 = -11, c_4 = -8, c_5 = -5, c_6 = -2, c_7 = 1, c_8 = 4, c_9 = 7.
\end{align*}
This then corresponds to a system in the form \eqref{eq:system} with 
\begin{align*}
\Ast = [0.8,1,-3,\ldots,-3], \quad \bphi(\bx,\bu) = [\bx,\bu,\bphi_1(\bx),\ldots,\bphi_{10}(\bx)].
\end{align*}

For this system, we parameterize our controller class $\Picon$ as, for any $\pi^{\btheta} \in \Picon$ with parameter $\btheta$,
\begin{align*}
\pi^{\btheta}(\bx) = \btheta_1 \bx + \sum_{i=1}^{10} \btheta_{i+1} \bphi_i(\bx) + \btheta_{12}.
\end{align*}
Note that the form of this controller lets us simply ``match'' the parameters of the system, and cancel undesirable parameters. Given this, for this system we let $\pist(A)$ be the controller which sets $\btheta_{1:11}$ to cancel the dynamics of the system $A$, and set $\btheta_{12} = c_1 = 10$. 

See \Cref{sec:car_details} for details on the computation of $\cH(A)$ on this system.

\subsubsection{Drone System (\Cref{fig:drone})}\label{sec:drone_details}
The dynamics of this system are given by 
\begin{align}\label{eq:drone}
\begin{split}
\bx_{h+1} = \begin{bmatrix} 1 & 0 & 0 & 0.1 & 0 & 0 \\
0 & 1 & 0 & 0 & 0.1 & 0 \\
0 & 0 & 1 & 0 & 0 & 0.1 \\
0 & 0 & 0 & 1 & 0 & 0 \\
0 & 0 & 0 & 0 & 1 & 0 \\
0 & 0 & 0 & 0 & 0 & 1 \end{bmatrix} \bx_h  + \begin{bmatrix} 0 & 0 & 0 \\
0 & 0 & 0 \\
0 & 0 & 0 \\
0.1 & 0 & 0 \\
0 & 0.1 & 0 \\
0 & 0 & 0.1 \end{bmatrix}  \bu_h + \begin{bmatrix} 0 \\ 0 \\ 0 \\ 0 \\ 0 \\ -0.98 \end{bmatrix} + \bw_h.
\end{split}
\end{align}
Here we interpret $[\bx]_{1:3}$ as the $x,y,$ and $z$ positions, respectively, and $[\bx]_{4:6}$ as the $x,y,z$ velocities. This system is therefore equivalent to three double integrator systems, with an affine term (which we interpret as ``gravity'') affecting only the $z$ coordinate. 
We set the cost to
\begin{align*}
\cost(\bx,\bu) = \frac{0.1}{5} \cdot \sum_{i=1}^{\dimx} [\bx]_i^2 + \frac{1}{5} \cdot [\bu]_1^2 + [\bu]_2^2 + [\bu]_3^2
\end{align*}
This then corresponds to a system in the form \eqref{eq:system} with 
\begin{align*}
    \Ast = \begin{bmatrix} 1 & 0 & 0 & 0.1 & 0 & 0 & 0 & 0 & 0 & 0 \\
0 & 1 & 0 & 0 & 0.1 & 0 & 0 & 0 & 0 & 0\\
0 & 0 & 1 & 0 & 0 & 0.1 & 0 & 0 & 0 & 0 \\
0 & 0 & 0 & 1 & 0 & 0 & 0.1 & 0 & 0 & 0 \\
0 & 0 & 0 & 0 & 1 & 0 & 0 & 0.1 & 0 & 0 \\
0 & 0 & 0 & 0 & 0 & 1 & 0 & 0 & 0.1 & -0.98 \end{bmatrix}, \quad \bphi(\bx,\bu) = [\bx,\bu,1].
\end{align*}

For this system, we parameterize our controller class $\Picon$ as, for any $\pi^{\btheta} \in \Picon$ with parameter $\btheta$,
\begin{align*}
\pi_h^{\btheta}(\bx) = \btheta_h^{\mathrm{fb}} \bx + \btheta_h^{\mathrm{offset}}
\end{align*}
where $\btheta_h^{\mathrm{fb}} \in \R^{3 \times 6}$ is the state-feedback portion of the controller, and $\btheta_h^{\mathrm{offset}} \in \R^3$ is an offset term. It can be shown that the optimal controller for a system of the form \eqref{eq:drone} can be parameterized in this way~\cite{yu2020power}. Furthermore, the optimal parameters can be computed in closed-form. As such, for this system we set $\pist(A)$ to be with the optimal parameters, computed using this closed-form solution.

In addition to computing the optimal controller in closed-form, we can also compute the cost of a controller, $\cJ(\pi;A)$, in closed-form.
To compute $\cH(A)$ in this example, we then simply apply the \texttt{torch.autograd.functional.hessian} function to $\cJ(\pist(A);A)$.

\subsubsection{Car System (\Cref{fig:car})}\label{sec:car_details}
The dynamics of this system are given by 
\begin{align}\label{eq:car}
\begin{split}
\bx_{h+1} = \begin{bmatrix} 1 & 0 & 0.1 & 0 & 0 & 0 \\
0 & 1 & 0 & 0.1 & 0 & 0 \\
0 & 0 & 1 & 0 & 0 & 0 \\
0 & 0 & 0 & 1 & 0 & 0 \\
0 & 0 & 0 & 0 & 1 & 0.1 \\
0 & 0 & 0 & 0 & 0 & 1 \end{bmatrix} \bx_h  + \begin{bmatrix} 0 & 0  \\
0 & 0  \\
0.1 \cdot \cos ([\bx_h]_5) & 0  \\
0.1 \cdot \sin( [\bx_h]_5) & 0 \\
0  & 0 \\
0 & 0.1 \end{bmatrix}  \bu_h  + \bw_h
\end{split}
\end{align}
where $[\bx_h]_5$ denotes the 5th element of $\bx_h$. Here we interpret $[\bx_h]_1$ as the $x$ position, $[\bx_h]_2$ as the $y$ position, $[\bx_h]_3$ as the $x$ velocity, $[\bx_h]_4$ as the $y$ velocity, $[\bx_h]_5$ as the angle of orientation (that is, the direction the car is facing), and $[\bx_h]_6$ as the angular velocity. The first control dimension, then, corresponds to the ``gas'', the power given to the car to move forward or backward, and the second control dimension corresponds to altering the direction of the steering wheel. Similar to the drone system, we set the cost to 
\begin{align*}
\cost(\bx,\bu) = \begin{bmatrix} \bx \\ \bu \end{bmatrix}^\top Q \begin{bmatrix} \bx \\ \bu \end{bmatrix} \quad \text{with} \quad Q =  \frac{ 0.1 \cdot I + \bv_1 \bv_1^\top + \bv_2 \bv_2^\top}{\| 0.1 \cdot I + \bv_1 \bv_1^\top + \bv_2 \bv_2^\top \|_\op}
\end{align*}
for some $\bv_1,\bv_2$. To write this in the form of \eqref{eq:system}, in order to make the problem more challenging we choose an overparameterized $\bphi(\bx,\bu)$:
\begin{align*}
    \bphi(\bx,\bu) = \big [\bx,\bu,\cos([\bx]_5),\sin([\bx]_5),[\bu]_1 \cdot \cos([\bx]_5),[\bu]_1 \cdot \sin([\bx]_5),[\bu]_2 \cdot \cos([\bx]_5),[\bu]_2 \cdot \sin([\bx]_5) \big ]
\end{align*}
and set 
\begin{align*}
\Ast = \begin{bmatrix} 1 & 0 & 0.1 & 0 & 0 & 0 & 0 & 0 & 0 & 0 & 0 & 0 & 0 & 0 \\
0 & 1 & 0 & 0.1 & 0 & 0 & 0 & 0 & 0 & 0 & 0 & 0 & 0 & 0 \\
0 & 0 & 1 & 0 & 0 & 0 & 0 & 0 & 0 & 0 & 0.1 & 0 & 0 & 0 \\
0 & 0 & 0 & 1 & 0 & 0 & 0 & 0 & 0 & 0 & 0 & 0.1 & 0 & 0 \\
0 & 0 & 0 & 0 & 1 & 0.1 & 0 & 0 & 0 & 0 & 0 & 0 & 0 & 0 \\
0 & 0 & 0 & 0 & 0 & 1 & 0 & 0.1 & 0 & 0 & 0 & 0 & 0 & 0 \end{bmatrix}.
\end{align*}

For the car system, the controller class $\Picon$ is a hierarchical controller parameterized by some $\btheta \in \R^4$. This controller first uses PD control to compute a ``goal input'', the direction we would like to modify the state in, as:
\begin{align*}
\bugoal(\bx) = - \btheta_1 [\bx]_{1:2} - \btheta_2 [ \bx]_{3:4}  . 
\end{align*}
Given the underactuated structure of the system in \eqref{eq:car}, we cannot directly push the state in the direction of $\bugoal(\bx)$. Instead, we set $\bu$ to the following:
\begin{align*}
\begin{bmatrix} \bugoal(\bx)^\top \begin{bmatrix} \cos([\bx]_5) \\ \sin([\bx]_5) \end{bmatrix} \\
-\btheta_3 ([\bx]_5 - \beta_{\mathrm{goal}}(\bx)) - \btheta_4 [\bx]_6 \end{bmatrix} \quad \text{for} \quad \beta_{\mathrm{goal}}(\bx) = \tan^{-1}( [\bugoal(\bx)]_2 / [\bugoal(\bx)]_1).
\end{align*}
Given the complex form of this controller and the dynamics, there does not exist a closed-form way to set $\btheta$ optimally. Instead, for this system, we rely on a simple random search procedure to compute $\pist(A)$. To find an optimal controller for system $A$, we randomly sample parameters $\btheta$, compute the cost they incur on system $A$, and then set $\pist(A)$ to the randomly generated controller with lowest cost. Note that this procedure is not differentiable, but we require $\pist(A)$ is differentiable. To remedy this, in situations where a differentiable $\pist(A)$ is needed (in particular, in the computation of $\cH(A)$), rather than returning a single controller, we return the softmin distribution over all controllers sampled, weighting each controller by its estimated cost. As the softmin distribution can be differentiated, this parameterization of $\pist(A)$ is differentiable. 

For this system, there does not exist a closed-form expression for $\cJ(\pi;A)$ and, as such, to compute $\cJ(\pi;A)$, we simply perform many roll-outs of policy $\pi$ on system $A$ and average the cost.
Given this and the search-based implementation of $\pist(A)$ outlined above, we found that computing the hessian $\cH(A)$ using the \texttt{torch.autograd.functional.hessian} as in \Cref{sec:drone_details} was very memory-intensive. Instead, we computed the Jacobian $G(A) := \nabla_{A'} \cJ(\pist(A');A)|_{A' = A}$, and then, in place of $\cH(A)$, we use $G(A) G(A)^\top$. To compute $G(A)$, we use the \texttt{torch.autograd.functional.jacobian} function. While using $G(A) G(A)^\top$ in place of $\cH(A)$ is not justified by our theoretical analysis, if we are in settings where $\pist(A)$ is not precisely the minimum of $\cJ(\pi;A)$ (which will likely be the case here since we are relying on a sampling-based implementation of $\pist(A)$, which will incur some small error), then we argue that this is a reasonable metric to use.
In particular, in this setting, the approximation of $\cJ(\pist(\Ahat);\Ast)$ given in \Cref{prop:quadratic_loss_approx} should have an additional first-order term of the form $G(A)^\top \vec(\Ast - \Ahat)$. As we can upper bound
\begin{align*}
G(A)^\top \vec(\Ast - \Ahat) \le \sqrt{ \| \vec(\Ast - \Ahat) \|_{G(\Ast) G(\Ast)^\top}^2},
\end{align*}
optimizing for the metric $G(A) G(A)^\top$ instead of $\cH(A)$ can be seen as minimizing the first-order Taylor-approximation of the excess loss. Intuitively, this metric quantifies the sensitivity of the loss to particular parameters in $\Ast$, and in practice we found that optimizing this metric produced significant improvements over existing methods. The implementation of the example from \Cref{sec:motivating_example} relied on this same approximation.

\subsection{Implementation Details}

For all methods considered, our implementation follows the basic structure of \Cref{alg:main}: at every epoch, we explore so as to minimize some exploration objective, form an estimate of $\Ast$ on the collected data, and then compute $\pist(\Ahat_t)$ on our estimate. Our main experimental results (\Cref{fig:motivating,fig:drone,fig:car}) show the loss of $\pist(\Ahat_t)$ as the time horizon $t$ increases.
For each method, to collect an initial set of data, we begin each trial by exploring randomly for some fixed number of episodes (10 for the drone example, 100 for the car example). The first point in each plot then corresponds to the performance after this initial random exploration.
Each aspect of our implementation is modular, and any given component can be easily replaced. Below we highlight our implementation of the exploration routine, and choice of exploration objective, for the various approaches we consider.

\subsubsection{Implementation of \expdesign}\label{sec:dynamicoed_imp_details}
Implementing the exploration procedure, \expdesign, requires access to a regret minimization oracle. While in principle the \LC algorithm of \cite{kakade2020information} could be applied to this problem to give such an oracle, the \LC algorithm requires access to a computation oracle which is not clear how to implement in practice. 
To remedy this, we implement a Thompson Sampling-inspired modification to the \LC algorithm of \cite{kakade2020information}.

The primary computational challenge of implementing the \LC algorithm is the computation of the \emph{optimistic} policy:
\begin{align*}
\argmin_{\pi \in \Piexp} \min_{A \in \textsc{Ball}^t} \cJ^{\mathrm{exp}}(\pi;A) 
\end{align*}
where $\cJ^{\mathrm{exp}}(\pi;A)$ denotes the exploration cost that is minimized in \LC (i.e. the expected cost on the cost function $\cost_h^n(\bx,\bu) \leftarrow \frac{1}{M} \cdot \bphi(\bx,\bu)^\top (\Xi_{n}) \bphi(\bx,\bu)$ set in \expdesign), and $\textsc{Ball}^t$ the confidence set for $\Ast$ at iteration $t$. 

To avoid solving this optimization, we adopt a Thompson Sampling-inspired variation of this procedure.
In particular, at iteration $t$, we sample $\Atil_t \sim \cN(\Ahat_t, \bLambda_t^{-1})$. Standard Thompson Sampling would then compute $\argmin_{\pi \in \Piexp} \cJ^{\mathrm{exp}}(\pi;\Atil_t)$, but even this can be challenging, so we instead rely on a sampling MPC-inspired approach. Given that we are at state $\bx_h$ and have played inputs $\bu_1,\ldots,\bu_{h-1}$, we aim to approximately solve the following optimization:
\begin{align}\label{eq:mpc_exp_design}
\begin{split}
& \min_{\bu_h, \bu_{h+1},\ldots,\bu_H \in \R^{\dimu}} \sum_{h'=h}^H \cost_h^n(\bx_{h'},\bu_{h'}) \\
& \quad \text{s.t.} \quad \bx_{h+1} = \Atil_t \bphi(\bx_h,\bu_h), \sum_{h=1}^H \bu_h^\top \bu_h \le \gamma^2. 
\end{split}
\end{align}
To solve this approximately, we sample many possible $\bu$ randomly, compute the value of the objective of \eqref{eq:mpc_exp_design} on the trajectories induced by these $\bu$, and finally choose the input that minimizes this objective. Rather than playing the entire sequence of chosen inputs, however, we simply play the first input in the sequence, observe the new state on the actual system, and re-solve \eqref{eq:mpc_exp_design} on this new state.
Note that the implementation of \LC used for the experiments given in \cite{kakade2020information} relies on a similar Thompson Sampling-based approximation to the \LC algorithm.

\subsubsection{Implementation of Uniform Exploration}
The goal of the procedure we have referred to as \texttt{Uniform Exploration} is to collect data which will result in the estimation error, $\| \Ast - \Ahat \|_\op$, being minimized, the goal of the method given in \cite{mania2020active}. It can be shown that this is equivalent to maximizing $\lammin(\bLambda_T)$, so this method reduces to choosing inputs that maximize $\lammin(\bLambda_T)$. To implement this procedure, we rely on the same sampling-based MPC approach as we outlined above, with the primary difference being that instead of minimizing the objective of \eqref{eq:mpc_exp_design}, we choose the inputs that maximize
\begin{align*}
\lammin \left ( \bLambda_t + \sum_{h=1}^H \bphi(\bx_h,\bu_h) \bphi(\bx_h,\bu_h)^\top \right ),
\end{align*}
where $\bLambda_t$ denote the covariates we have obtained so far at iteration $t$. 
While very similar in spirit to the algorithm of \cite{mania2020active}, the implementation details are somewhat different than the algorithm proposed in that work. We found that in practice our implementation performed better than directly implementing (a sampling-based variant of) the algorithm from \cite{mania2020active}, and all reported results for \texttt{Uniform Exploration} are therefore on this version.

\subsubsection{Exploring via Cost Minimization}
A natural point of comparison to our methods would be to forsake the system identification phase entirely, and simply run standard policy optimization algorithms such as TRPO or PPO \citep{schulman2015trust,schulman2017proximal}, to obtain a controller $\pihat_T$. The primary difficulty with these approaches in the settings we consider is that these algorithms are \emph{on-policy}, meaning that they primarily roll out trajectories using their current estimate of the optimal policy, $\pihat_t$, and using the collected data to do policy improvement on $\pihat_t$. In contrast, our setting is \emph{off-policy}, in the sense that the learner must explore by playing policies in $\Piexp$, but return some policy in $\Picon$. Since in the settings we consider $\Piexp \neq \Picon$, on-policy approaches are simply exploring very differently, and therefore cannot be compared with directly. 

This is particularly an issue in our setting where \emph{stability} may come into play. Indeed, it may be the case that some controller in $\Picon$ will destabilize the system, and cause the norm of the state to increase exponentially in $h$. Inducing such trajectories significantly improves one's ability to perform system identification as the signal-to-noise ratio also then increases exponentially. However, to induce this trajectory with a state-feedback controller, the power of the input played by this controller will also increase exponentially in $h$. Since we choose $\Piexp$ to include only policies with bounded power, this is not a fair comparison (and, furthermore, is likely not an algorithm one would want to run in practice).

While direct comparison with such approaches is therefore not possible, it is possible to compare against algorithms that, instead of collecting data that minimizes or maximizes objectives such as $\tr(\cH \bLamchk^{-1})$ or $\lammin(\bLambda)$, instead simply aims to play policies minimizing $\cJ(\pi;A)$. In principle, such algorithms are similar to approaches such as TRPO in how they perform their exploration---both collect data by aiming to minimize the actual cost we are attempting to find a controller to minimize. 

To implement this approach, we rely on a sampling-based MPC algorithm similar to that described in \Cref{sec:dynamicoed_imp_details}, but where the goal is now to solve
\begin{align*}
\min_{\pi \in \Piexp} \cJ(\pi;A).
\end{align*}
Note that the key difference between this approach and approaches such as TRPO is that we still only play $\pi \in \Piexp$. Using this objective to induce exploration, we then simply estimate $\Ast$ on this collected data, and return $\pist(\Ahat)$. The results of this approach on the drone system are given in \Cref{fig:drone_ts} (with this cost minimization approach denoted as \texttt{Cost Minimization Exploration}). As this illustrates, this approach is significantly worse than \Cref{alg:main}, and is also outperformed by \texttt{Uniform Exploration} or \texttt{Random Exploration} when the number of episodes is large enough.

\begin{figure}
\centering
\includegraphics[width=0.4\linewidth]{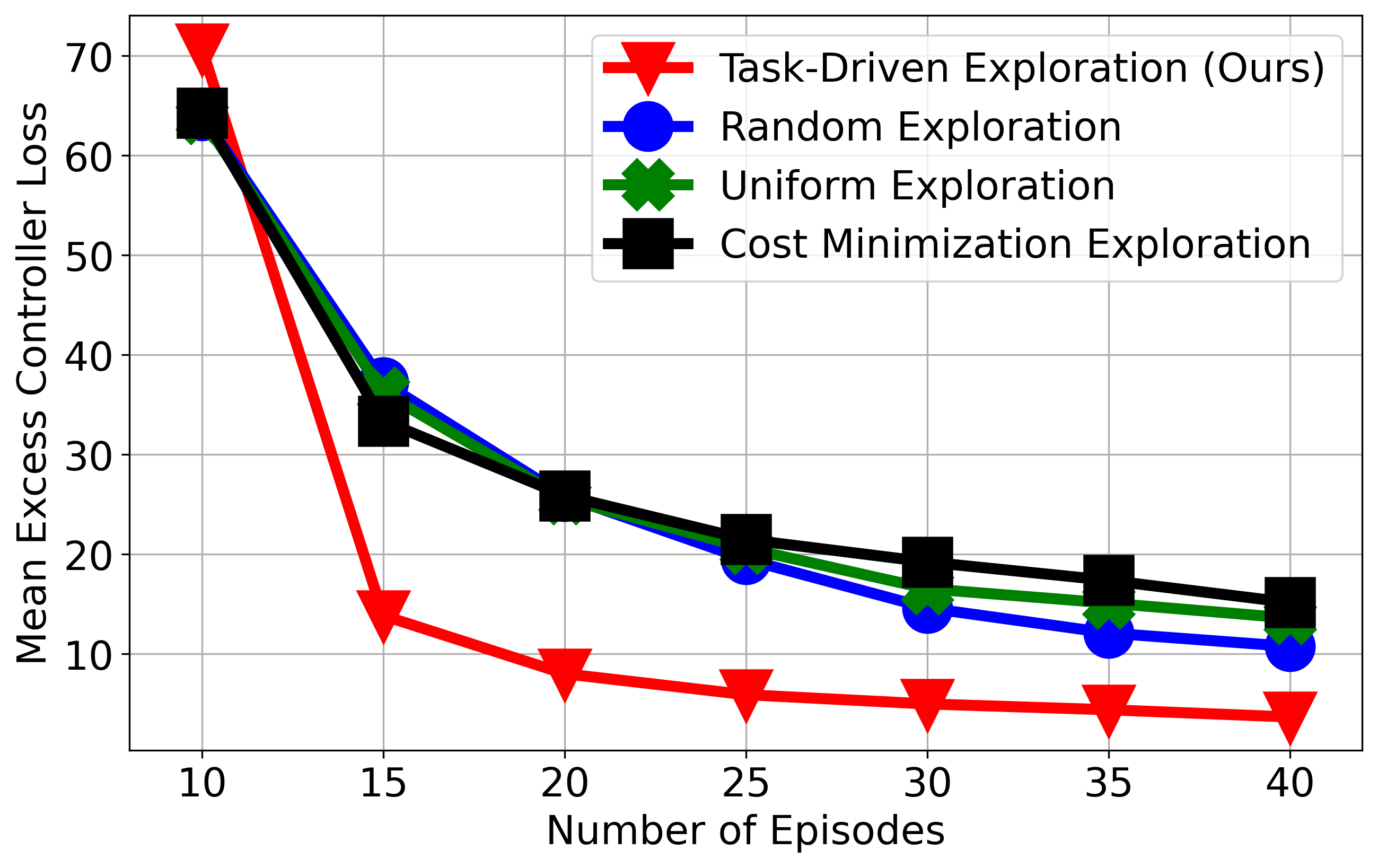}
\caption{Performance on drone with \LC Exploration}
\label{fig:drone_ts}
\end{figure}

\subsection{Additional Results}
Finally, in this section we present versions of \Cref{fig:motivating,fig:drone,fig:car} with error bars in \Cref{fig:special_errorbars,fig:drone_errorbars,fig:car_errorbars}. In all figures, errors bars denote one standard error.

\begin{figure}
\centering
\includegraphics[width=0.4\linewidth]{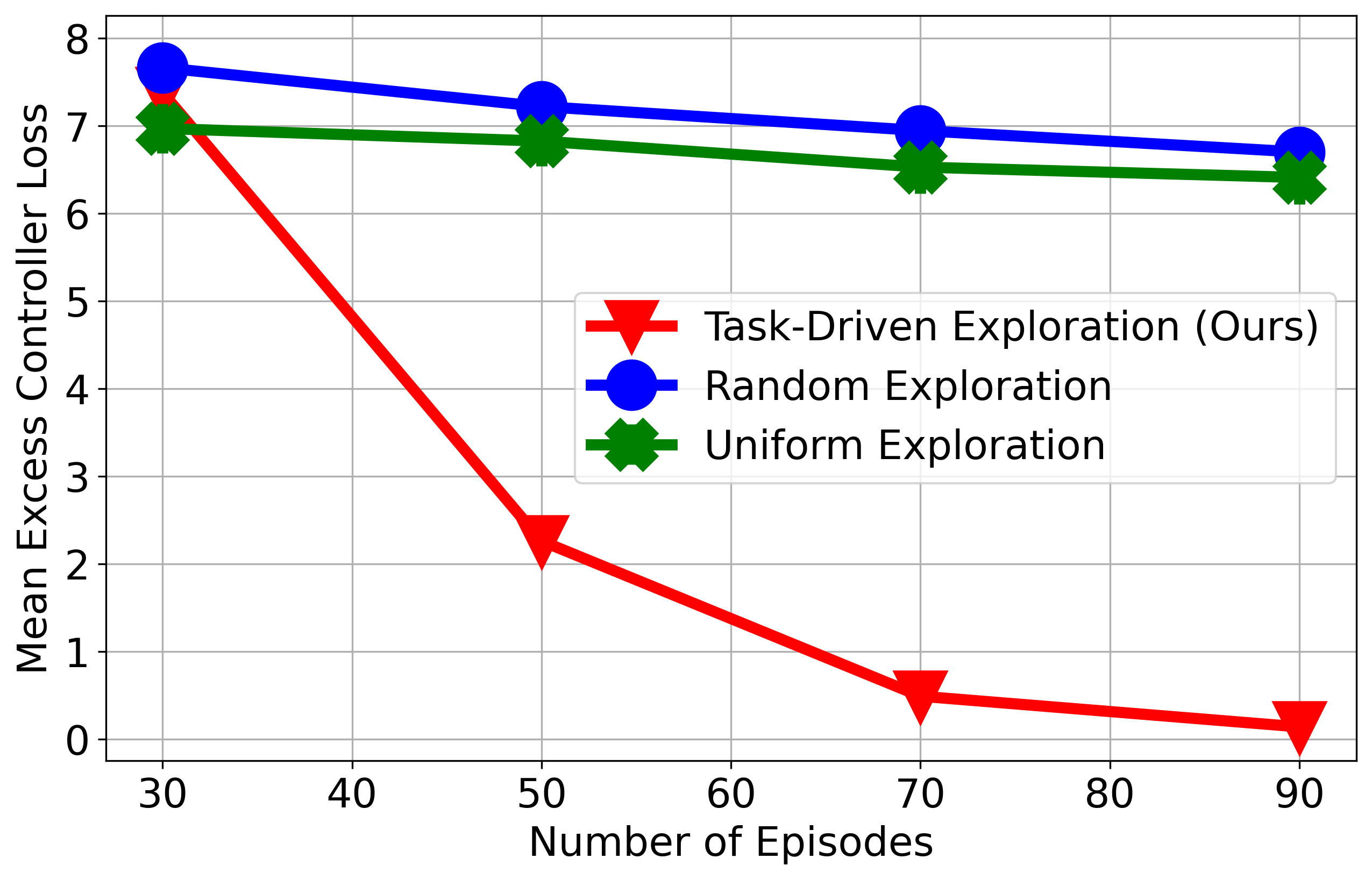}
\caption{Mean excess controller loss on instance of \Cref{sec:motivating_example} with error bars}
\label{fig:special_errorbars}
\end{figure}

\begin{figure}
\centering
\includegraphics[width=0.4\linewidth]{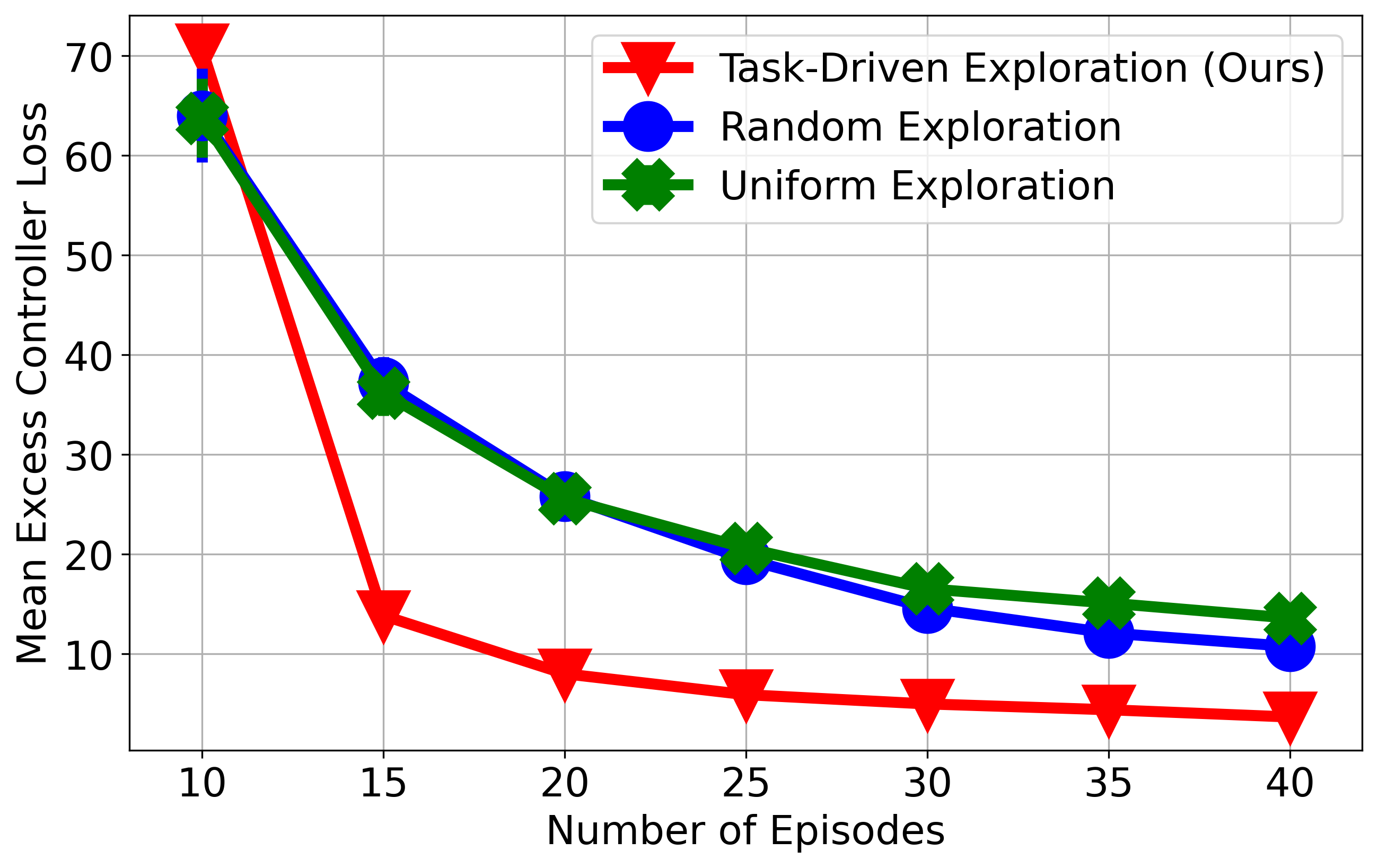}
\caption{Mean excess controller loss on drone with error bars}
\label{fig:drone_errorbars}
\end{figure}

\begin{figure}
\centering
\includegraphics[width=0.4\linewidth]{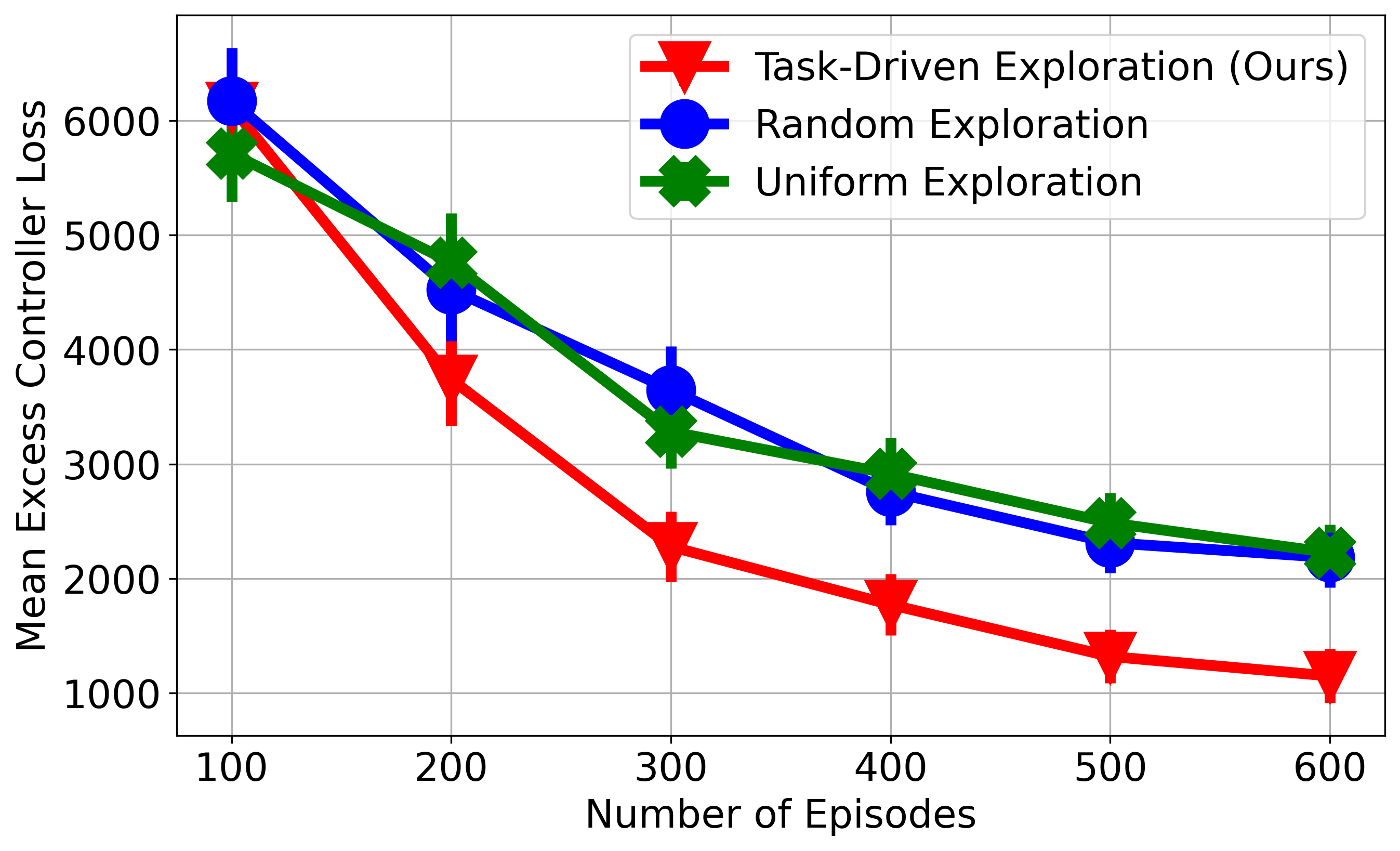}
\caption{Mean excess controller loss on car with error bars}
\label{fig:car_errorbars}
\end{figure}

\end{document}